\documentclass{article}

\usepackage[preprint]{neurips_2026}
\usepackage[utf8]{inputenc}
\usepackage[T1]{fontenc}
\usepackage{xcolor}
\definecolor{citeblue}{HTML}{1B4F72}      
\definecolor{linkblue}{HTML}{21618C}      
\definecolor{urlblue}{HTML}{2874A6}       
\definecolor{scmoarow}{HTML}{EBF5FB}      
\definecolor{sampfill}{HTML}{E8F6F3}
\definecolor{sampbord}{HTML}{73C6B6}
\definecolor{outfill}{HTML}{FDEBD0}
\definecolor{outbord}{HTML}{DC7633}
\definecolor{arrowcol}{HTML}{7F8C8D}
\definecolor{scmoabord}{HTML}{2874A6}
\definecolor{scmoafill}{HTML}{EBF5FB}
\definecolor{stepfill}{HTML}{FFF9C4}
\definecolor{stepbord}{HTML}{F9A825}
\definecolor{phasefill}{HTML}{E3F2FD}
\definecolor{phasebord}{HTML}{1565C0}
\definecolor{goafill}{HTML}{F3E5F5}
\definecolor{goabord}{HTML}{7B1FA2}
\definecolor{metafill}{HTML}{EDE7F6}
\definecolor{loopcolor}{HTML}{AB47BC}

\usepackage[
  colorlinks=true,
  citecolor=citeblue,
  linkcolor=linkblue,
  urlcolor=urlblue,
]{hyperref}
\usepackage{url}
\usepackage{booktabs}
\usepackage{amsfonts}
\usepackage{amsmath}
\usepackage{amssymb}
\usepackage{amsthm}
\usepackage{nicefrac}
\usepackage{microtype}
\usepackage{graphicx}
\usepackage{placeins}
\usepackage{float}
\usepackage{algorithm}
\usepackage{algpseudocode}
\usepackage{multirow}
\usepackage{colortbl}
\usepackage{enumitem}
\usepackage{wrapfig}
\usepackage{tikz}
\usetikzlibrary{arrows.meta, positioning, calc, fit, backgrounds, shapes.geometric, decorations.pathreplacing}

\newcommand{\PP}{\mathbb{P}}             
\newcommand{\EE}{\mathbb{E}}             
\newcommand{\ind}{\mathbf{1}}            
\newcommand{\rhobar}{\bar\rho}           
\newcommand{\pmin}{p_{\min}}             
\newcommand{\kmax}{k_{\max}}             
\DeclareMathOperator*{\argmax}{arg\,max}

\theoremstyle{plain}
\newtheorem{theorem}{Theorem}
\newtheorem{proposition}[theorem]{Proposition}
\newtheorem{corollary}[theorem]{Corollary}

\theoremstyle{definition}

\newtheorem{remark}{Remark}

\title{Beyond Consensus: Trace-Level Synthesis in\\Mixture of Agents}

\author{
  Shreyas Fadnavis \\
  Bioscope AI \\
  \texttt{shreyas.fadnavis@bioscope.ai} \\
  \And
  Praitayini Kanakaraj \\
  Bioscope AI \\
  \texttt{praitayini.kanakaraj@bioscope.ai} \\
  \And
  Felix Wyss \\
  Bioscope AI \\
  \texttt{felix.wyss@bioscope.ai}
}

\begin{document}
\raggedbottom

\maketitle


\begin{abstract}
When multiple LLM agents solve the same problem, standard practice
compresses each agent's reasoning into a majority vote or layered
synthesis, treating agreement as the finish line.
We show this is unnecessarily lossy: an LLM aggregator that reads
complete reasoning traces recovers correct solutions even when
agents unanimously agree, with beneficial corrections consistently
outweighing harmful ones---the \emph{aggregation paradox}.
Majority voting has a ceiling that perturbation diversity does not
raise (error correlations are identical); the aggregator's gain
comes from trace-level complementarity, assembling correct
intermediate steps from minority chains that voting discards.
These findings motivate Self-Consistent Mixture of
Agents which generates trace diversity through semantic-preserving
input perturbations, safeguards the majority via anchored refinement
with provable non-degradation guarantees, and always
synthesizes---never gates on consensus.
A single model with perturbation-induced trace variation outperforms
heterogeneous model pools across structured reasoning, PhD-level
science, competition mathematics, and competitive programming.
The unit of aggregation should be the reasoning trace, not the answer.
\end{abstract}

\section{Introduction}
\label{sec:intro}

Inference-time reasoning in large language models increasingly relies
on generating multiple candidate solutions and combining them.
Self-consistency~\citep{wang2023selfconsistency} votes over
chain-of-thought answers, Mixture of
Agents~\citep{wang2025moa} synthesizes outputs across model layers,
multi-agent debate~\citep{du2024debate} lets agents revise after
reading each other's reasoning, and verification-guided
selection~\citep{zhang2025thinkprm,rocnreroll2026} picks the
best candidate.  Despite their differences, all reduce a set of
reasoning traces to a single answer---by vote, convergence, or
selection (see Appendix~\ref{app:architecture-comparison} for an
architectural comparison).

This reduction discards most of the information that reasoning
produces.  A chain-of-thought trace records proof strategies,
intermediate calculations, dead ends, and implicit assumptions---yet
at the point of aggregation, each trace is compressed to a single
answer token (in voting) or an output passage (in layered synthesis).
When agents converge, the standard heuristic is to accept the
consensus and move on~\citep{cats2026}, setting aside minority chains
that may contain correct intermediate steps the majority missed.

We find that this information loss has a measurable cost.  When we
replace majority voting with an LLM aggregator that reads the complete
reasoning traces, it recovers correct solutions even when agents
unanimously agree on the wrong answer.  Across benchmarks spanning
structured reasoning, PhD-level science, competition mathematics, and
competitive programming, beneficial corrections consistently outweigh
harmful ones (\S\ref{sec:paradox}).  We call this the \emph{aggregation
paradox}: synthesis over traces improves accuracy precisely where
voting indicates no improvement is needed.
Figure~\ref{fig:conceptual} illustrates the full pipeline on a
GPQA-Diamond problem where five perturbations produce five different
proof strategies, yet trace-level synthesis assembles evidence that
no single trace contains (Appendix~\ref{app:case-studies}). The mechanism goes beyond what existing frameworks can explain.
The voting ceiling is not raised by perturbation diversity---error
correlations are statistically identical
(\S\ref{sec:finding2})---and a single model with
input perturbations matches or exceeds heterogeneous model
pools (\S\ref{sec:scaling}).  The improvement comes from trace-level
complementarity---different chains contain different correct intermediate
steps, and the aggregator assembles solutions that no individual chain
produced---with cheap surface variation sufficing to induce
it~(\S\ref{sec:finding3};~\citealt{modc2026,kruszewski2026filtering}). These findings motivate \textbf{SC-MoA} (Self-Consistent Mixture of
Agents; \S\ref{sec:method}, Figure~\ref{fig:pipeline}):
perturbation-based trace diversity that sidesteps mode collapse,
anchored refinement that converts the debate
martingale~\citep{choi2025debate} into a submartingale with provable
non-degradation guarantees, and synthesis that never gates
on consensus.

\paragraph{Contributions.}
\begin{itemize}[leftmargin=1.2em,itemsep=1pt,topsep=1pt,parsep=0pt]
\item The \emph{aggregation paradox}: trace-level synthesis
  outperforms voting even at unanimous consensus, with beneficial
  flips outweighing harmful ones
  (\S\ref{sec:finding1}; Proposition~\ref{prop:synthesis},
  Corollary~\ref{cor:dominance}).
\item Perturbation diversity does not decorrelate errors
  ($\rhobar$ is identical for diverse vs.\ i.i.d.\ proposals),
  yet synthesis exceeds the voting ceiling via trace-level
  complementarity (\S\ref{sec:finding2}); cheap surface
  perturbation suffices (\S\ref{sec:finding3}).
\item SC-MoA (\S\ref{sec:method}, Figure~\ref{fig:pipeline}):
  perturbation diversity, anchored refinement with submartingale
  guarantees (Proposition~\ref{prop:anchoring}), and universal
  aggregation.  Highest accuracy on all five benchmarks under
  our evaluation protocol (greedy decoding, matched compute;
  Table~\ref{tab:main}); a single model outperforms
  heterogeneous pools (\S\ref{sec:scaling}).
\end{itemize}


\begin{figure*}[t]
\centering
\includegraphics[width=\textwidth]{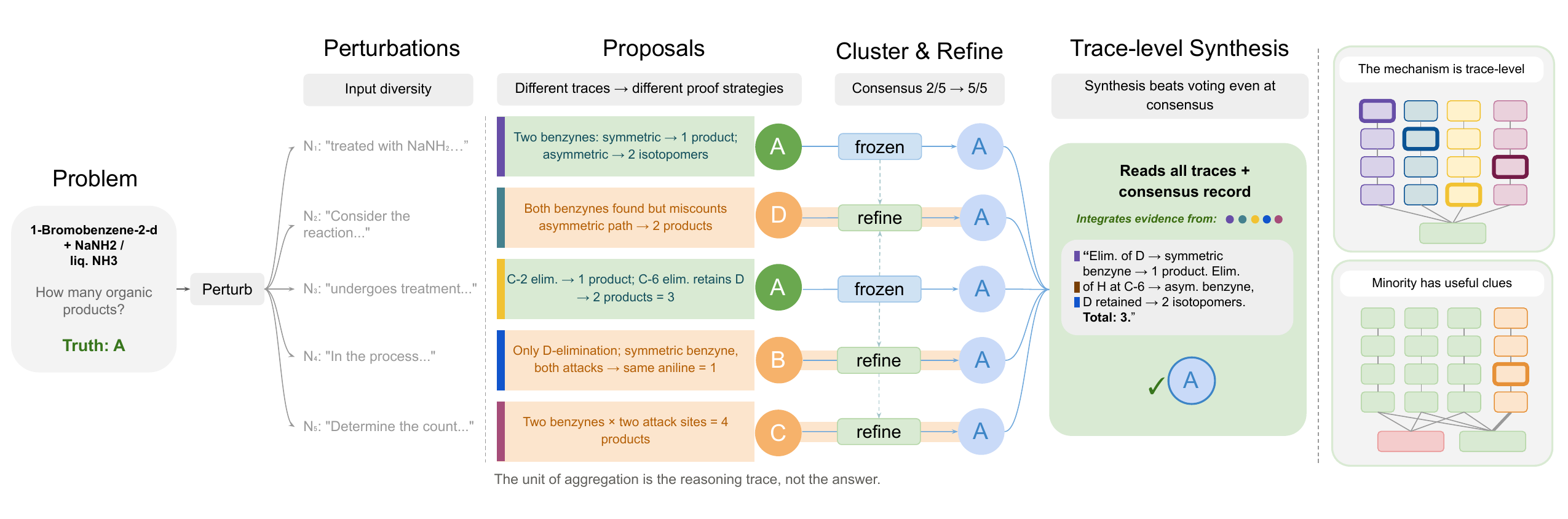}
\vspace{-4mm}
\caption{\textbf{SC-MoA on a GPQA-Diamond chemistry problem.}
Five semantic perturbations of the same question produce structurally
different reasoning traces: two identify the correct answer~(A) via
different proof strategies, while three reach wrong answers through
plausible but incomplete arguments.
Anchored refinement freezes the majority answer and revises
minorities, raising consensus from 2/5 to 5/5.
Trace-level synthesis then reads \emph{all} traces---including the
now-corrected minorities---and integrates complementary evidence
(e.g., the symmetric-benzyne argument from one trace and the
D-retention isotopomer count from another) into a single coherent
explanation.
Majority voting would return the same answer~(A) but discard all
reasoning; synthesis recovers a richer, more reliable justification
by treating the trace, not the answer, as the unit of aggregation.}
\label{fig:conceptual}
\end{figure*}

\section{Related Work}
\label{sec:related}

\paragraph{Self-consistency and the voting ceiling.}
Self-consistency~\citep{wang2023selfconsistency} samples multiple
chains and takes a plurality vote.  Subsequent work refines voting:
ranked~\citep{rvsc2025}, confidence-weighted~\citep{cisc2025},
inconsistency-aware~\citep{mirror2025}, adaptive
budgets~\citep{blendasc2025,cats2026}, certified
stopping~\citep{certifiedsc2025}, and improved
calibration~\citep{dinco2026}.  All improve \emph{how} votes are
counted, but the ballot remains an answer
token~\citep{modc2026,kruszewski2026filtering}.
\citet{hamidieh2026crossmodel} address diversity collapse across
models; we suggest reading the traces voting discards.

\paragraph{Multi-agent debate and layered collaboration.}
Multi-agent debate~\citep{du2024debate,smit2024mad} and Mixture of
Agents~\citep{wang2025moa} let agents read and respond to full
outputs, but interaction without structure can fail:
\citet{choi2025debate} proved unconstrained debate is a martingale,
\citet{khan2024debate} showed debate helps only under asymmetric
knowledge, and \citet{rizvimartel2026communication} characterize
when communication helps at all.
GoA~\citep{goa2026} and
AgentAuditor~\citep{agentauditor2026} address specific failure modes
but require model diversity or additional training.
SC-MoA's anchored refinement~(\S\ref{sec:anchoring}) constrains
what can change rather than structuring the interaction
(Figure~\ref{fig:arch-comparison}).

\paragraph{Verification, selection, and the synthesis ceiling.}
Test-time compute scaling~\citep{snell2025scaling,chen2025provable}
improves accuracy through sampling and selection;
ThinkPRM~\citep{zhang2025thinkprm} scores verification chains for
best-of-$N$, and \citet{rocnreroll2026} characterize the
best-of-$N$ ceiling.
Critique--revise
methods~\citep{madaan2023selfrefine,shinn2023reflexion,yuksekgonul2025textgrad}
iterate on a single solution but lack convergence
guarantees~\citep{semanticbp2024,huang2024selfcorrect}.
All share a bound: accuracy cannot exceed the best candidate.
Trace-level synthesis exceeds this by integrating evidence across
chains~(\S\ref{sec:finding1}).

\paragraph{Perturbation diversity and traces.}
DiVeRSe~\citep{diverse2022} combines diverse prompts with step-aware
verification; SPUQ~\citep{gao2024spuq} introduces
semantic-preserving perturbations for uncertainty quantification;
BayesPE~\citep{tonolini2024bayespe} formalizes prompt ensembles in
a Bayesian framework; boosted
ensembles~\citep{pitis2023boosted,freund1997adaboost} adapt
reweighting to prompt selection;
Polyjuice~\citep{wu2021polyjuice} and
Tailor~\citep{ross2022tailor} compose perturbation primitives with
meaning-\emph{changing} edits.
All use perturbation to improve diversity or estimate uncertainty,
reducing traces to votes or confidence intervals.
We extend perturbation to reasoning synthesis and find that
perturbation \emph{content} is not load-bearing within the tested
families~(\S\ref{sec:finding3}).
\citet{kambhampati2025stop} argued that traces are computational
scaffolding, and \citet{valmeekam2025beyond} showed that
semantically corrupted traces preserve accuracy, suggesting trace
content serves as computational scaffolding.
Our findings are complementary inference-time observations
(\S\ref{sec:finding2},~\ref{sec:finding3}): trace-level
complementarity is a statistical property of the output
distribution, not a claim about internal cognition.

\section{The Aggregation Paradox}
\label{sec:paradox}

We now test the claim that trace-level complementarity carries
exploitable information that standard aggregation discards.
Three observations together constitute the \emph{aggregation paradox}:
synthesis beats voting even at consensus
(\S\ref{sec:finding1}), the voting ceiling is binding yet synthesis exceeds it
(\S\ref{sec:finding2}), and perturbation
content is not load-bearing---cheap surface variation suffices
(\S\ref{sec:finding3}).

\paragraph{Controlled setting.}
To isolate the aggregation mechanism from confounds---model
heterogeneity, prompt engineering, temperature sampling---we fix a
single controlled protocol throughout this section.
$N{=}5$ agents use the same model (\texttt{gpt-oss-120b}) with
greedy decoding (temperature~0).  Semantic-preserving input
perturbations (SPUQ;~\citealt{gao2024spuq}) generate structurally
different reasoning traces from each agent.  We compare majority
voting against LLM synthesis applied to the \emph{same} set of
proposals, so that the sole variable is the aggregation mechanism.
Evaluation spans five benchmarks
(BBH-3, MMLU-ML, GPQA-Diamond, AIME, LCB-Hard;
details in \S\ref{sec:setup}).

\paragraph{Synthesis beats voting even at consensus.}
\label{sec:finding1}
Standard practice gates aggregation on consensus: if agents agree,
accept the majority answer and save
compute~\citep{cats2026}.  The logic is sound---why run an expensive
synthesis call when the answer is already determined?

\textit{But this logic is wrong.} In a controlled pre-refinement
setting on GPQA-Diamond ($n{=}198$), trace-level synthesis produces
12~beneficial flips versus 5~harmful ($2.4{\times}$; $+3.5$\,pp;
McNemar $p{=}0.09$; Table~\ref{tab:aggregation-value},
Appendix~\ref{app:aggregation-value}).  On BBH-3 ($n{=}296$),
the same pre-refinement analysis yields 16~beneficial
versus 5~harmful ($3.2{\times}$; $+3.7$\,pp;
$p{=}0.016$; Appendix~\ref{app:aggregation-value}).
After anchored refinement, the harmful flips are eliminated:
across all four QA benchmarks in the full pipeline
(Table~\ref{tab:synthesis-decomp}), BBH and AIME show zero harmful
flips---refinement converts the five pre-refinement harmful flips
into correct answers, an important finding in its own right. The following decomposition formalizes this asymmetry.

\begin{proposition}[Synthesis advantage decomposition]
\label{prop:synthesis}
Let \textsc{Vote} and \textsc{Synth} be two aggregation procedures
applied to the same $N$ proposals.  Define the \emph{synthesis
fidelity} $F_s = \PP(\textsc{Synth}\text{ correct} \mid
\textsc{Vote}\text{ correct})$ and the \emph{recovery rate}
$R_s = \PP(\textsc{Synth}\text{ correct} \mid \textsc{Vote}\text{
wrong})$.  Then:
\begin{equation}
\underbrace{\PP(\textsc{Synth}\text{ correct}) - \PP(\textsc{Vote}\text{ correct})}_{\text{synthesis advantage}} \;=\;
\underbrace{P_w \cdot R_s}_{\text{beneficial flips}} \;-\;
\underbrace{P_c \cdot (1 - F_s)}_{\text{harmful flips}}
\label{eq:synthesis}
\end{equation}
where $P_c = \PP(\textsc{Vote}\text{ correct})$ and $P_w = 1 - P_c$.
\end{proposition}

\noindent
The decomposition follows from the law of total probability
(Appendix~\ref{app:proof-synthesis}).  Its value is not as a bound
but as a diagnostic: it decomposes the net synthesis effect into two
independently measurable quantities.

\begin{corollary}[Synthesis dominance condition]
\label{cor:dominance}
Synthesis dominates voting whenever the recovery-to-corruption
ratio exceeds the correctness odds:
\begin{equation}
\frac{R_s}{1 - F_s} \;>\; \frac{P_c}{P_w}.
\label{eq:dominance}
\end{equation}
\end{corollary}

\noindent
The dominance condition is satisfied on both GPQA-Diamond
($2.4{\times}$) and BBH-3 ($3.2{\times}$;
Appendix~\ref{app:qa-nondeg}).
On the full BBH (2{,}561 problems, Appendix~\ref{app:bbh-tasks}),
SC-MoA wins 91 problems where SC fails and loses 48---a 65.5\%
head-to-head win rate, concentrating on disambiguation, temporal
reasoning, and state tracking.
At weak aggregators the dominance condition fails; the score-based
override provides the safety net (Appendix~\ref{app:override}). Consensus gating forfeits this recoverable accuracy---but what
explains \emph{why} synthesis succeeds where voting fails?

\paragraph{Beyond the voting ceiling.}
\label{sec:finding2}
The natural hypothesis is that perturbation diversity helps by
decorrelating errors, as predicted by the Condorcet jury
theorem~\citep{condorcet1785,ladha1992}(Appendix~\ref{app:proof-jury}).  We test this directly by
measuring the mean pairwise error correlation $\rhobar$.
Perturbation-diverse and i.i.d.\ proposals yield nearly identical
$\rhobar$: $0.633$ vs.\ $0.603$ on GPQA-Diamond, $0.571$
vs.\ $0.607$ on LCB-Hard, with fully overlapping 95\% confidence
intervals (Appendix~\ref{app:persona-diversity}).
Error decorrelation is not the mechanism. Yet on the same proposals, trace-level synthesis outperforms majority
voting by $+3.5$\,pp on GPQA; the full pipeline extends this to
$+5.3$\,pp on LCB-Hard (Table~\ref{tab:main}).
The explanation is structural:
any vote-based aggregator $A_V\!: V \to \hat{y}$ satisfies
$I(y^*;\hat{y}) \leq I(y^*; \mathbf{V})$ by the data processing
inequality, so information present in traces but absent from votes
is inaccessible by construction.  The trace ablation confirms
this empirically: collapsing traces to answer labels reduces
accuracy to the vote baseline
(Table~\ref{tab:trace-ablation}).

\textbf{What predicts recovery?}  We measure trace diversity $D_t$
as mean pairwise cosine distance across reasoning traces in TF-IDF
space (5000 features).
Beneficial flips exhibit significantly higher $D_t$:
$0.585$ vs.\ $0.460$ on GPQA-Diamond ($p{=}0.015$;
Figure~\ref{fig:mechanism}a; Appendix~\ref{app:trace-diversity}).
The effect strengthens under dense pre-trained embeddings
(1024-dimensional; $p{=}0.0005$, Cohen's $d{=}0.92$),
confirming it reflects structural reasoning differences, not
lexical surface variation (Appendix~\ref{app:embedding-diversity}).
The nontrivial question is whether
the aggregator can extract this surplus.  We decompose the recovery
rate into two factors
(Proposition~\ref{prop:extraction},
Appendix~\ref{app:extraction}): the probability that minority
traces contain recoverable content ($p_{\mathrm{sub}}$) and the
probability that the aggregator identifies it ($q$).  Synthesis
dominates voting when $F_s$ exceeds a threshold determined by these
quantities (Corollary~\ref{cor:dominance}). A targeted trace ablation tests each factor
(Table~\ref{tab:trace-ablation},
Appendix~\ref{app:trace-ablation}): answer-only and majority-only
inputs collapse to the vote baseline ($q \approx 0$ when traces are
absent), while minority-only proposals recover the full benefit
(73.2\% vs.\ 72.2\% full-trace; $p_{\mathrm{sub}}$ is high).
Permuting steps within traces preserves accuracy---the aggregator
treats traces as a bag of evidence.  The active ingredient is
minority-trace content. One question remains: \textit{does the method by which trace diversity is
induced matter?}

\paragraph{Perturbation content is not load-bearing.}
\label{sec:finding3}
Within the family of meaning-preserving perturbations, we find no
evidence that perturbation content affects accuracy.
Hand-crafted personas, SPUQ paraphrases, and GPT-generated
strategies are statistically indistinguishable
(McNemar $p{>}0.25$; 72.7\%, 72.7\%, 70.2\% on GPQA-Diamond;
Appendices~\ref{app:qa-personas},~\ref{app:code-personas}).
Controlled perturbations from a Polyjuice/Tailor
taxonomy~\citep{wu2021polyjuice,ross2022tailor} yield 76.0\%
vs.\ 74.0\% for free-form SPUQ ($p{=}1.0$;
Appendix~\ref{app:controlled-perturbation}).
\citet{valmeekam2025beyond} showed that even semantically
\emph{corrupted} traces preserve accuracy, suggesting LLMs are
robust to trace-level noise.
Our finding is complementary but distinct: semantic-\emph{preserving}
perturbations that change only surface form suffice to induce diverse
reasoning strategies, consistent with the mode-concentration
hypothesis~\citep{modc2026,kruszewski2026filtering}.
SC-MoA therefore requires no prompt engineering or domain expertise
within semantic-preserving perturbation families.
We have not tested adversarial or meaning-changing perturbations;
the claim should not be extrapolated beyond surface variation that
preserves semantics.


\section{SC-MoA: From Paradox to Algorithm}
\label{sec:method}

\begin{wrapfigure}{R}{0.55\textwidth}
\vspace{-14pt}
\centering
\resizebox{\linewidth}{!}{

\begin{tikzpicture}[
  >=Stealth,
  every node/.style={font=\small},
  qbox/.style={
    draw=outbord, fill=outfill, rounded corners=3pt,
    minimum height=0.50cm, minimum width=1.2cm, font=\small,
  },
  sbox/.style={
    draw=sampbord, fill=sampfill, rounded corners=3pt,
    minimum height=0.36cm, minimum width=0.48cm, font=\scriptsize,
  },
  pbox/.style={
    draw=stepbord, fill=stepfill, rounded corners=3pt,
    minimum height=0.36cm, minimum width=0.58cm, font=\scriptsize,
  },
  stepbox/.style={
    draw=phasebord, fill=phasefill, rounded corners=3pt,
    minimum height=0.50cm, font=\small, align=center,
    inner xsep=6pt, inner ysep=4pt,
  },
  refinebox/.style={
    draw=phasebord, fill=phasefill, rounded corners=3pt,
    minimum height=0.50cm, font=\small, align=center,
    inner xsep=6pt, inner ysep=4pt,
  },
  arr/.style={->, arrowcol, thick},
  darr/.style={->, arrowcol, thick, densely dashed},
  annot/.style={font=\scriptsize\itshape, text=black!55},
  plabel/.style={font=\footnotesize\bfseries\sffamily},
  connector/.style={->, arrowcol!70, very thick, rounded corners=4pt},
]


\node[qbox] (q) at (1.5, 1.2) {Problem $q$};


\node[stepbox, minimum width=2.4cm] (perturb) at (1.5, 0.0)
  {Semantic perturbations};

\node[pbox] (x1) at (0.1, -0.80) {$x_1$};
\node[pbox] (x2) at (1.3, -0.80) {$x_2$};
\node[font=\scriptsize] at (2.05, -0.80) {$\cdots$};
\node[pbox] (xN) at (2.7, -0.80) {$x_N$};

\draw[arr] (perturb.south -| x1) -- (x1);
\draw[arr] (perturb.south -| x2) -- (x2);
\draw[arr] (perturb.south -| xN) -- (xN);

\node[sbox] (r1a) at (-0.2, -1.48) {$r_1$};
\node[sbox] (r1k) at (0.4, -1.48) {$r_k$};
\node[font=\tiny] at (0.1, -1.48) {$\cdots$};
\node[font=\scriptsize] at (2.05, -1.48) {$\cdots$};
\node[sbox] (rNa) at (2.4, -1.48) {$r_1$};
\node[sbox] (rNk) at (3.0, -1.48) {$r_k$};
\node[font=\tiny] at (2.7, -1.48) {$\cdots$};

\draw[arr] (x1) -- (r1a);  \draw[arr] (x1) -- (r1k);
\draw[arr] (xN) -- (rNa);  \draw[arr] (xN) -- (rNk);

\node[pbox] (p1) at (0.1, -2.25) {$p_1$};
\node[font=\scriptsize] at (2.05, -2.25) {$\cdots$};
\node[pbox] (pN) at (2.7, -2.25) {$p_N$};

\draw[arr] (r1a) -- (p1);  \draw[arr] (r1k) -- (p1);
\draw[arr] (rNa) -- (pN);  \draw[arr] (rNk) -- (pN);

\node[annot, anchor=west] at (3.25, -0.80) {${\times}\,N$};
\node[annot, anchor=west] at (3.25, -1.48) {${\times}\,k$};
\node[annot, anchor=west] at (3.25, -2.25) {$\alpha_i$};

\begin{scope}[on background layer]
  \node[draw=phasebord!40, rounded corners=5pt, fill=phasefill!8,
        inner sep=0pt, line width=0.5pt,
        minimum width=5.4cm, minimum height=4.0cm]
        (ph1) at (1.5, -1.2) {};
\end{scope}
\node[plabel, text=phasebord, anchor=north, fill=phasefill!8,
      inner sep=1.5pt]
  at ([yshift=-2pt]ph1.north)
  {Phase 1: Proposal generation};

\node[draw=arrowcol!45, fill=white, rounded corners=3pt,
      font=\scriptsize, align=left, inner sep=5pt]
  (altsrc) at (-2.6, -1.52)
  {i.i.d.\\[1pt]MoA/GoA\\[1pt]RAG/tools\\[1pt]Human};
\node[plabel, text=arrowcol!55, align=center, anchor=south,
      font=\tiny\bfseries\sffamily] at (altsrc.north) {Alt.\ sources};
\draw[darr, arrowcol!50, rounded corners=5pt]
  (altsrc.east) -- (-1.2, -1.52 |- altsrc.east) -- (-1.2, -2.25) -- (p1.west);
\node[font=\tiny\bfseries, text=arrowcol!50] at (-1.35, -1.95) {or};


\node[stepbox] (cluster) at (9.2, -0.05)
  {Cluster by answer equivalence};

\node[refinebox] (refine) at (9.2, -1.20)
  {Anchored refinement\\[-2pt]{\scriptsize freeze $a^*$; refine minorities w.r.t.\ $p^*$}};

\node[stepbox, draw=phasebord!70, fill=phasefill!50] (cerec) at (9.2, -2.40)
  {Consensus-evolution record\\[-2pt]{\scriptsize clusters, shifts, labels, $\alpha_i$ scores}};

\draw[arr] (cluster) -- (refine);
\draw[arr] (refine) -- (cerec);

\begin{scope}[on background layer]
  \node[draw=phasebord!30, rounded corners=5pt, fill=phasefill!8,
        inner sep=0pt, line width=0.5pt,
        minimum width=6.4cm, minimum height=4.0cm]
        (ph2) at (9.2, -1.2) {};
\end{scope}
\node[plabel, text=phasebord, anchor=north, fill=phasefill!8,
      inner sep=1.5pt]
  at ([yshift=-2pt]ph2.north)
  {Phase 2: Clustering \& refinement};


\begin{scope}[on background layer]
  \node[draw=scmoabord!35, rounded corners=5pt, fill=scmoafill!8,
        inner sep=0pt, line width=0.5pt,
        minimum width=5.8cm, minimum height=3.2cm]
        (ph3) at (5.35, -5.90) {};
\end{scope}
\node[plabel, text=scmoabord, anchor=north, fill=scmoafill!8,
      inner sep=1.5pt]
  at ([yshift=-2pt]ph3.north)
  {Phase 3: Universal aggregation};

\node[stepbox, draw=scmoabord, fill=scmoafill!50, line width=0.8pt]
  (aggregator) at (5.35, -5.35)
  {Trace-level aggregator\\[-2pt]{\scriptsize all $N$ traces $+$ consensus record}};

\node[qbox] (synth) at (5.35, -6.65) {Synthesized output};
\draw[arr] (aggregator) -- (synth);

\node[annot, text=phasebord!75, anchor=east, align=right]
  at ([xshift=-6pt]ph3.west |- aggregator)
  {\textbf{No early exit}\\[-1pt]{\tiny even at unanimous}\\[-1pt]{\tiny consensus}};

\node[font=\tiny\itshape, text=outbord!80!black, anchor=east]
  at ([yshift=-8pt, xshift=-4pt]synth.south) {score-based override};

\node[qbox, draw=sampbord, fill=sampfill, line width=1.0pt,
      minimum height=0.60cm, minimum width=1.6cm,
      font=\small\bfseries] (final) at (5.35, -7.95) {Final output $\hat{a}$};

\draw[arr, line width=0.9pt] (synth) -- (final);


\draw[connector]
  (ph1.east) --
  node[above, font=\scriptsize, text=arrowcol!80]
    {proposals}
  (ph2.west);

\draw[connector, sampbord, densely dashed, rounded corners=5pt]
  ([xshift=0.4cm]ph1.south) -- ++(0,-0.35) --
  node[above, font=\scriptsize, text=sampbord!80!black, pos=0.5]
    {traces $T_1\!\ldots\!T_N$}
  ([xshift=-0.5cm]ph3.north |- {0,-3.75}) --
  ([xshift=-0.5cm]ph3.north);

\draw[connector, rounded corners=5pt]
  ([xshift=-0.4cm]ph2.south) -- ++(0,-0.35) --
  node[above, font=\scriptsize, text=arrowcol!80, pos=0.5]
    {consensus record}
  ([xshift=0.5cm]ph3.north |- {0,-3.75}) --
  ([xshift=0.5cm]ph3.north);

\end{tikzpicture}}
\vspace{-4pt}
\caption{\textbf{SC-MoA pipeline.}
Phase~1 generates $N$ proposals via semantic perturbations and inner self-consistency.
Phase~2 clusters by answer equivalence and refines minority traces with the majority answer frozen.
Phase~3 synthesizes all $N$ traces and the consensus record into a single output, with no early exit even at unanimous consensus.}
\label{fig:pipeline}
\vspace{-10pt}
\end{wrapfigure}%
Section~\ref{sec:paradox} imposes four design constraints:
cheap perturbation suffices for diversity (\S\ref{sec:finding3}),
the aggregator must read full traces since vote-based methods
cannot access trace-level information by
construction~(\S\ref{sec:finding2}), aggregation must never gate on
consensus~(\S\ref{sec:finding1}), and refinement must protect the
majority against the debate
martingale~\citep{choi2025debate}.
SC-MoA addresses these in three phases
(Algorithm~\ref{alg:scmoa}, Figure~\ref{fig:pipeline}):
perturbation-based diversity, anchored refinement, and universal
synthesis with verification.

\paragraph{Phase~1: Diversity without degradation.}
\label{sec:diversity}

RL-trained language models concentrate on a small number of reasoning
modes~\citep{modc2026,kruszewski2026filtering}; repeated sampling
at a fixed prompt explores within a mode rather than across modes.
Since perturbation content is not
load-bearing~(\S\ref{sec:finding3}), SC-MoA uses the cheapest
mechanism: a single LLM call produces $N$ semantic-preserving
rephrasings $x_1, \ldots, x_N$
(SPUQ;~\citealt{gao2024spuq}), with a validation step ensuring
protected tokens appear verbatim
(Appendix~\ref{app:spuq}; controlled-perturbation variant in
Appendix~\ref{app:controlled-perturbation}). Each $x_i$ is fed to the base model with a generic CoT prompt.
For $k > 1$, a neutral variant tag
(Appendix~\ref{app:variant-tag}) creates diversity at greedy
decoding; the best-of-$k$ is selected by quality score, with
intra-agreement
$\alpha_i = |\{j : a_{ij} = a_{i,\text{best}}\}|/k$ as
confidence signal ($N{\cdot}k$ calls total).

\paragraph{Phase~2: Anchored refinement.}
\label{sec:anchoring}
To avoid the debate martingale~\citep{choi2025debate}, SC-MoA
freezes the majority and revises only the minorities.

Proposals are clustered by answer equivalence: string-match for QA
(deterministic, faithful); test-pass signatures for code
(stochastic, potentially unfaithful; \S\ref{sec:aggregation-safety}).
The consensus $C = k_{\max}/N$ is the majority cluster fraction.
All minorities are refined while the majority remains
frozen (Proposition~\ref{prop:anchoring}): each minority receives a single call with the majority trace as     
reference and $\alpha_i$ as confidence, yielding a revised proposal that either converges to the majority solution, integrates complementary reasoning from both traces, or retains its original answer with justification (${\leq},N - k_{\max}$ calls).

\begin{proposition}[Anchoring invariant and refinement submartingale]
\label{prop:anchoring}
Let $a^*_0$ be the initial majority answer with support $k_0$.  SC-MoA
returns $a^*_0$ unless all $N - k_0$ refined minority proposals converge
to the same $a' \neq a^*_0$ that wins the post-refinement vote or
aggregation.  For $k_0 \geq \lceil N/2 \rceil$ (i.e., a strict
majority), override is impossible.  Moreover, the majority-vote
accuracy after anchored refinement is non-decreasing:
\begin{equation}
\PP\bigl(\text{post-refinement vote correct}\bigr)
\;\geq\; \PP\bigl(\text{pre-refinement vote correct}\bigr).
\label{eq:submartingale}
\end{equation}
\end{proposition}

\noindent On QA tasks the score-based override does not fire (quality
score is always~1.0); non-degradation instead follows from high
synthesis fidelity alone.

\begin{corollary}[QA non-degradation]
\label{cor:qa-nondeg}
On tasks where the override does not fire, synthesis is non-degrading
whenever $F_s \geq P_c\,/\,(P_c + P_w \cdot R_s)$.
\end{corollary}

\noindent Empirical $F_s$ values exceed the non-degradation
thresholds on both benchmarks (Appendix~\ref{app:qa-nondeg}).

\paragraph{Phase~3: Universal aggregation with verification.}
\label{sec:aggregation}
The aggregation paradox (\S\ref{sec:finding1}) rules out consensus
gating~\citep{cats2026}; SC-MoA always runs the aggregator.
A single call synthesizes across all $N$ post-refinement proposals
with a consensus-evolution narrative (cluster shifts,
action labels, $\alpha_i$ scores; prompts in
Appendices~\ref{app:refine-prompt},~\ref{app:agg-prompt}).
Any consensus threshold $\theta < \infty$ for early exit
\emph{hurts} accuracy (\S\ref{sec:aggregation-safety}).

\begin{wrapfigure}[17]{R}{0.45\textwidth}
\vspace{-12pt}
\begingroup
\algtext*{EndFor}
\algtext*{EndIf}
\footnotesize
\refstepcounter{algorithm}
\label{alg:scmoa}
\noindent\rule{\linewidth}{0.4pt}\\[1pt]
\textbf{Algorithm~\thealgorithm}\ Self-Consistent Mixture of Agents\\[-5pt]
\rule{\linewidth}{0.4pt}
\vspace{-4pt}
\begin{algorithmic}[1]
\footnotesize
\Require $q$, perturbation count $N$, samples/pert.\ $k$
\Ensure Final answer $\hat{a}$
\Statex \textcolor{citeblue}{\rule[0.4ex]{1.2em}{0.4pt}\ \textsc{Perturbation \& Proposal} (\S\ref{sec:diversity})\ \rule[0.4ex]{1.2em}{0.4pt}}
\State $(x_1, \ldots, x_N) \gets \textsc{Perturb}(q)$
\For{$i \in [N]$}
    \State Draw $\{r_{ij}\}_{j=1}^{k} \sim \textsc{Generate}(x_i)$
    \State $p_i \gets \argmax_j \textsc{Score}(r_{ij})$
    \State $\alpha_i \gets |\{j : \textsc{Ans}(r_{ij}) {=} \textsc{Ans}(p_i)\}|/k$
\EndFor
\Statex \textcolor{citeblue}{\rule[0.4ex]{1.2em}{0.4pt}\ \textsc{Anchored Refinement} (\S\ref{sec:anchoring})\ \rule[0.4ex]{1.2em}{0.4pt}}
\State $\mathcal{C},\, a^*,\, \kmax \gets \textsc{Cluster}(\{p_i\})$
\For{$i \in [N]$ \textbf{s.t.} $\textsc{Ans}(p_i) \neq a^*$}
    \State $p_i \gets \textsc{Refine}(p_i,\; p^*,\; \alpha_i)$
\EndFor
\Statex \textcolor{citeblue}{\rule[0.4ex]{1.2em}{0.4pt}\ \textsc{Aggregation \& Verification} (\S\ref{sec:aggregation})\ \rule[0.4ex]{1.2em}{0.4pt}}
\State $\hat{a} \gets \textsc{Aggregate}(\{p_1, \ldots, p_N\},\; \mathcal{C})$
\If{$\textsc{Score}(\hat{a}) < \max_i \textsc{Score}(p_i)$}
    \State $\hat{a} \gets \argmax_i \textsc{Score}(p_i)$
\EndIf
\State \Return $\hat{a}$
\end{algorithmic}
\rule{\linewidth}{0.4pt}
\endgroup
\end{wrapfigure}%
A score-based override reverts to the best proposal when the
aggregator underperforms (code: public test pass rate; QA: always
1.0, so override never fires;
Appendix~\ref{app:override}).
We call clustering $\Phi$ \emph{faithful} if cluster agreement
implies answer equivalence; for code, test-pass clustering is
\emph{unfaithful} (Appendix~\ref{app:fidelity-gap}).  This QA/code
asymmetry pervades the pipeline (\S\ref{sec:calibration},
\S\ref{sec:aggregation-safety}).
\section{Experiments}
\label{sec:experiments}
\paragraph{Setup.}
\label{sec:setup}
All methods use \texttt{gpt-oss-120b} with greedy decoding
(temperature~0).  Ground truth never enters a prompt.  For
$k > 1$, a neutral variant tag creates distinct LLM calls; all
baselines use the same variant-tag protocol.
Sampling protocol and token-budget comparisons
(Appendices~\ref{app:sampling-protocol},~\ref{app:token-budget})
confirm greedy+tag is representative and gains are not explained
by token expenditure.
Full details in Appendix~\ref{app:implementation}.
SC-MoA uses two operating points.  Main results
(Table~\ref{tab:main}): $N{=}5$, $k{=}2$ (11~calls).  Mechanistic analyses
(Figures~\ref{fig:mechanism}--\ref{fig:scaling-crossmodel}):
$N{=}4$, $k{=}5$ (21~calls), matching MM-MoA's peak
operating point on LCB-Hard and the intra-perturbation
saturation point (Appendix~\ref{app:k-sweep}).

\paragraph{Benchmarks.}
  We evaluate on five benchmarks spanning question answering, mathematics, and code (selection rationale in Appendix~\ref{app:benchmark-rationale}):
BBH (27~tasks, 2{,}400+ problems;~\citealt{suzgun2023bbh}),
MMLU-ML (112 questions;~\citealt{hendrycks2021mmlu}),
GPQA-Diamond (198 PhD-level science;~\citealt{rein2024gpqa}),
AIME (90 competition-math problems from AIME 2022--2024, integer answers 0--999;~\citealt{aimo2024}),
and LCB-Hard (171 hard competitive-programming problems,
post-training cutoff, private-test scored;~\citealt{jain2024livecodebench}).
Contamination analysis in Appendix~\ref{app:contamination}.

\paragraph{Baselines.}
All baselines are compute-equalized at ${\sim}$7--11 calls (Appendix~\ref{app:compute-budget}):
Zero-shot CoT (1~call),
SC $k{=}10$ (10~calls;~\citealt{wang2023selfconsistency}),
MoA (9~calls;~\citealt{wang2025moa}),
Self-MoA (7~calls),
and TextGrad (${\sim}$10~calls;~\citealt{yuksekgonul2025textgrad}).
All share the same evaluation pipeline.
MoA uses $N{=}4$ SPUQ paraphrases ($2N{+}1{=}9$ calls),
isolating aggregation architecture.
Details in Appendix~\ref{app:implementation}.

\label{sec:results}

\begin{table}[t]
\caption{Main results (\texttt{gpt-oss-120b}, $N{=}5$, $k{=}2$, 11~calls).
Best in \textbf{bold}; second-best \underline{underlined}.
GoA reports the GoA$_{\text{Mean}}$ variant (3-model pool;
GoA$_{\text{Max}}$ is 2.5\,pp lower on GPQA;
Appendix~\ref{app:goa});
MoA uses 4 heterogeneous proposers;
MoA~(para) uses SPUQ paraphrases with a single model.
$\dagger$\,Answer-extraction artifact (Appendix~\ref{app:bbh-tasks}).
Baseline details, gating ablation, and sampling protocol comparisons
in Appendices~\ref{app:goa},~\ref{app:bbh-tasks},~\ref{app:diversity-source}.}
\label{tab:main}
\centering
\small
\begin{tabular}{lccccccc}
\toprule
& \multicolumn{4}{c}{QA benchmarks} && Code \\
\cmidrule(lr){2-5} \cmidrule(lr){7-7}
Method & BBH-3 (296) & MMLU-ML (112) & GPQA (198) & AIME (90) && LCB-Hard (171) \\
\midrule
Zero-shot CoT              & 32.8 & 83.9 & 66.7 & 70.0 && 42.1 \\
SC              & 80.4 & \underline{90.2} & 70.7 & 85.6 && \underline{57.3} \\
MoA                        & 69.9 & 85.7 & 67.7 & 87.8 && \underline{57.3} \\
MoA (para)        & 68.2 & 86.6 & 65.2 & 73.3 && 50.9 \\
Self-MoA                   & 67.6 & 85.7 & 64.1 & \underline{90.0} && \underline{57.3} \\
TextGrad                   & \underline{83.8} & \underline{90.2} & 69.7 & 85.6 && 52.0 \\
GoA                        & 82.8 & 89.3 & \underline{72.7} & 77.8 && 24.6 \\
\midrule
\rowcolor{scmoarow}
\textbf{SC-MoA}            & \textbf{86.5} & \textbf{92.0} & \textbf{73.2} & \textbf{91.1} && \textbf{62.6} \\
\bottomrule
\end{tabular}
\end{table}
SC-MoA holds the highest point estimate on all five benchmarks
(Table~\ref{tab:main}; pairwise significance tests in
Appendix~\ref{app:significance}; per-task breakdowns in
Appendices~\ref{app:bbh-tasks}--\ref{app:lcb-extended}).
Seven of twelve pairwise McNemar tests reach $p{<}0.05$;
the five non-significant comparisons concentrate on GPQA
($n{=}198$) and LCB-Hard ($n{=}171$), where statistical power is
inherently limited by sample size
(Appendix~\ref{app:significance}).

On BBH-3 (296~problems), SC-MoA reaches 86.5\% ($+$6.1~pp over
SC, $+$2.7~pp over TextGrad, $+$3.7~pp over GoA), with gains
concentrating on structured-reasoning tasks where perturbation
diversity surfaces distinct parsing strategies (full BBH results
in Appendix~\ref{app:bbh-tasks}).
On GPQA-Diamond, SC-MoA reaches 73.2\%
($+$2.5~pp over SC, $+$3.5~pp over TextGrad;
Appendix~\ref{app:gpqa}).  Self-MoA collapses to
64.1\%---\emph{below} zero-shot (66.7\%)---illustrating the
unanchored debate martingale~\citep{choi2025debate}.
On AIME (90 competition-math problems), SC-MoA reaches 91.1\%
($+$1.1~pp over Self-MoA, $+$5.5~pp over SC).
On LCB-Hard (code), SC-MoA reaches 62.6\% ($+$5.3~pp over SC
at greedy decoding, $+$10.6~pp over TextGrad;
Appendix~\ref{app:lcb-extended}).  Under temperature sampling
($T{=}0.7$), SC reaches 67.3\% on LCB-Hard
(Appendix~\ref{app:sampling-protocol}); SC-MoA's advantage on
code is specific to the greedy+tag protocol.
The override fires on 14.6\% of problems; a best-of-$N$
ablation confirms it is essential at weak aggregators
(Appendix~\ref{app:best-of-n}; oracle ceiling in
Appendix~\ref{app:oracle}).

\begin{wrapfigure}{R}{0.5\textwidth}
\vspace{-14pt}
\centering
\includegraphics[width=\linewidth]{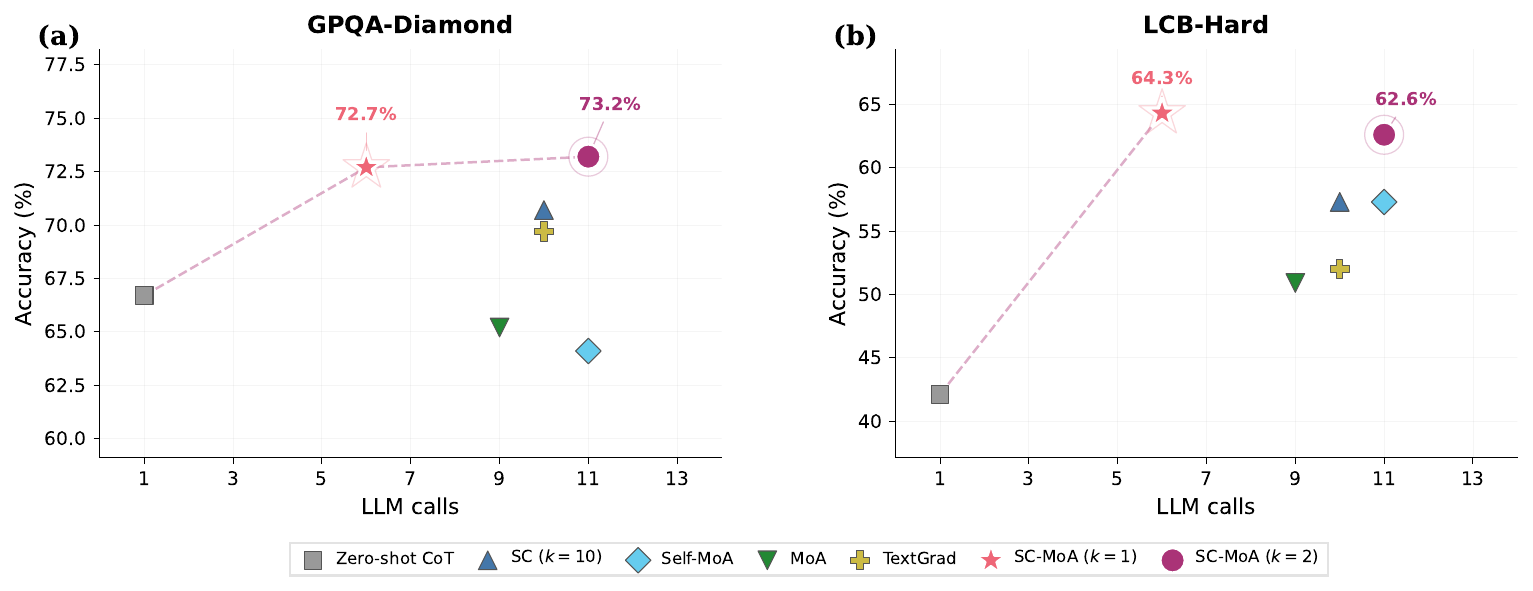}
\vspace{-6mm}
\caption{\textbf{Compute--accuracy Pareto frontier.}
SC-MoA $k{=}1$ ($N{=}4$, ${\sim}5$~calls) exceeds SC($k{=}10$) on
GPQA (74.7\% vs.\ 70.7\%) at roughly half the compute; on LCB-Hard
it slightly trails SC greedy+tag (55.6\% vs.\ 57.3\%).
$N{=}5$, $k{=}2$ (11~calls) reaches 73.2\% on GPQA and 62.6\% on
LCB-Hard, extending the frontier.
\emph{Note:} $k{=}1$ and $k{=}2$ use different $N$; points are
not directly comparable across the two operating points.}
\label{fig:pareto}
\vspace{-4pt}
\end{wrapfigure}%
The aggregation paradox predicts that if the value lives in
trace-level synthesis rather than in the vote, then even a single
synthesis pass---with no majority to fall back on---should
already capture the complementary reasoning that voting discards.
SC-MoA at $k{=}1$ ($N{=}4$, ${\sim}5$ calls) tests this directly:
each perturbation produces a single trace, and the aggregator
synthesizes once with no majority vote.  This minimal configuration
reaches 74.7\% on GPQA ($+4.0$\,pp over SC) and 55.6\% on
LCB-Hard at roughly half the compute
(full $k{=}1$ results in Table~\ref{tab:k1},
Appendix~\ref{app:k1-variant}).
At $N{=}5$, $k{=}2$ (11~calls), accuracy reaches 73.2\% on GPQA
(Table~\ref{tab:main}), further extending the Pareto frontier
(Figure~\ref{fig:pareto}).
The $k{=}1$ vs.\ $k{=}2$ comparison is confounded by different $N$
and configuration choices; a direct significance test on
LCB-Hard at matched $N$ yields McNemar $p{=}0.51$, consistent with
inner-sampling saturation and the unreliability of public-test
selection under unfaithful clustering
(\S\ref{sec:aggregation-safety}).
\paragraph{Aggregation safety and override analysis.}
\label{sec:aggregation-safety}
Section~\ref{sec:method} introduced two safety mechanisms: the
anchoring invariant, which structurally prevents consensus
degradation (Proposition~\ref{prop:anchoring}), and score-based
override, which reverts to the best individual proposal when the
aggregator underperforms.  We now verify both empirically. The consensus transition matrix pooled over 867~problems
(Figure~\ref{fig:anchoring}a) has \emph{zero} entries below the
diagonal, confirming the submartingale prediction.
Figure~\ref{fig:anchoring}b stratifies expected consensus
improvement by difficulty: low-consensus problems exhibit the
largest gain ($\EE[\Delta C] \approx 0.3$), while high-consensus
problems are already near the ceiling, concentrating refinement
where the anchoring constraint is most load-bearing.
The net accuracy effect (Figure~\ref{fig:anchoring}c) ranges
from $+1.8$\,pp (MMLU-ML) to $+6.1$\,pp (BBH), with the
full pipeline rescuing problems across all benchmarks. The score-based override is essential on code ($+5.8$\,pp on
LCB-Hard) but contributes ${\leq}0.3$\,pp on BBH and MMLU-ML
(Appendix~\ref{app:override})---the QA/code asymmetry predicted
by the faithful/unfaithful clustering distinction
(\S\ref{sec:aggregation}).
On LCB-Hard, 69\% of problems reach unanimous consensus, yet
always-aggregate outperforms gating by $5.1$\,pp because test-pass
clustering is unfaithful---48.3\% of unanimous clusters mask
hidden-test divergence
(Table~\ref{tab:theta-sweep};
Appendices~\ref{app:theta-sweep},~\ref{app:fidelity-gap},~\ref{app:clusters}).
\begin{figure*}[h]
\centering
\includegraphics[width=1\textwidth]{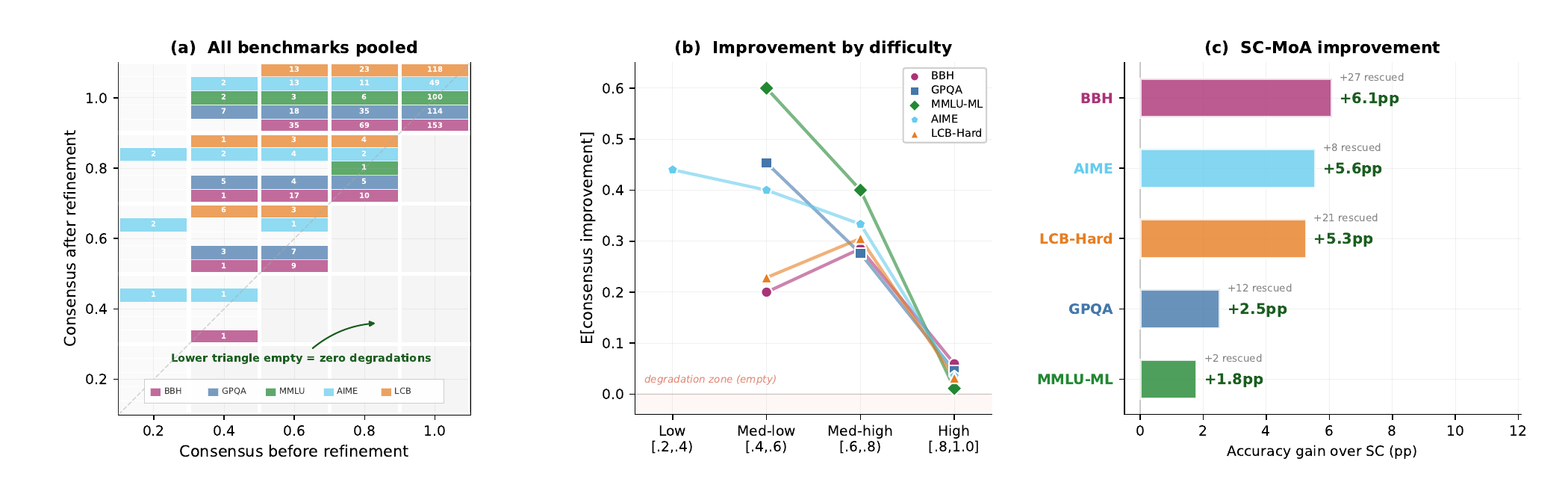}
\vspace{-2mm}
\caption{\textbf{Anchored refinement never degrades consensus}
($N{=}5$, $k{=}2$, 867~problems).
\textbf{(a)}~Transition matrix: zero degradations.
\textbf{(b)}~Low-consensus problems benefit most.
\textbf{(c)}~Net accuracy gain per benchmark.}
\label{fig:anchoring}
\end{figure*}
\label{sec:mechanism}

\begin{figure}[h]
\centering
\includegraphics[width=1\linewidth]{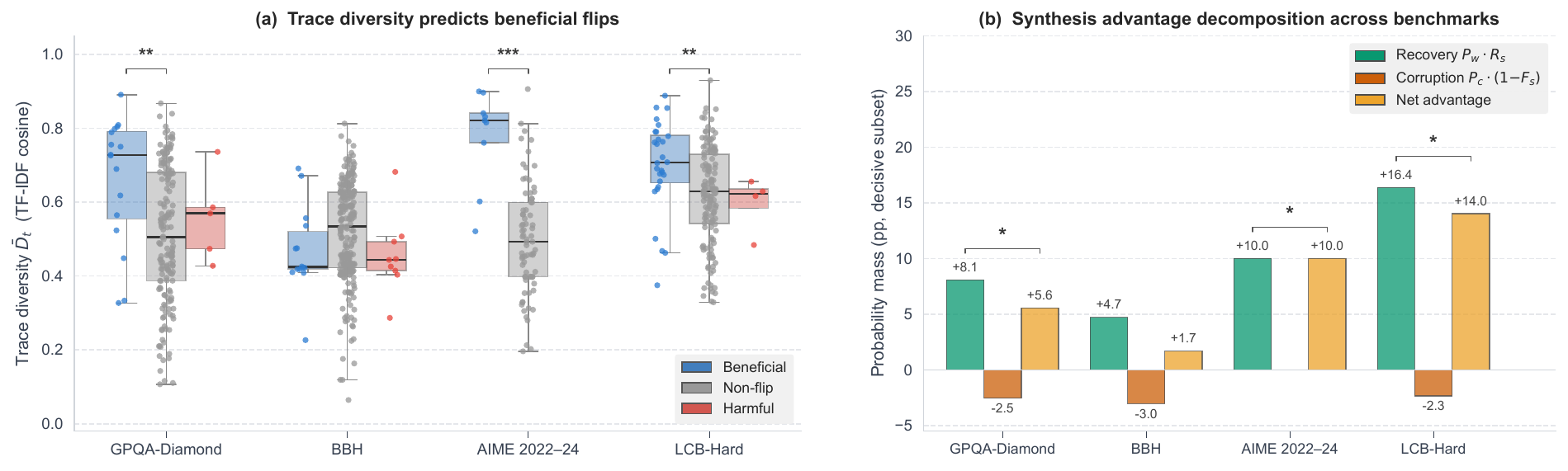}
\vspace{-2mm}
\caption{\textbf{Beneficial flips occur when traces are diverse;
recovery dominates corruption on all benchmarks.}
Mechanistic config: $N{=}4$, $k{=}5$ (21~calls; \S\ref{sec:setup}).
\textbf{(a)}~Trace diversity $\bar{D}_t$ stratified by outcome:
beneficial flips show significantly higher diversity on GPQA, AIME,
and LCB-Hard ($p{<}0.01$).
\textbf{(b)}~Synthesis decomposition (Proposition~\ref{prop:synthesis}):
recovery dominates corruption on all four benchmarks,
significantly on three ($p{<}0.05$; BBH $p{=}0.30$).}
\label{fig:mechanism}
\end{figure}
\textbf{The mechanism: trace-level synthesis.} Section~\ref{sec:paradox} established trace-level
complementarity as the mechanism behind the synthesis advantage.
We now dissect it under the mechanistic configuration ($N{=}4$,
$k{=}5$, 21~calls; \S\ref{sec:setup}), which provides
sufficient per-perturbation samples for decomposition and
ablation tests. Beneficial flips exhibit significantly higher inter-trace cosine
distance on three of four benchmarks ($p{<}0.01$;
Figure~\ref{fig:mechanism}a), with recovery consistently
dominating corruption on all four
(Table~\ref{tab:synthesis-decomp};
Figure~\ref{fig:mechanism}b; significantly on three,
$p{<}0.05$; BBH $p{=}0.30$).  Synthesis fidelity is high
($F_s = 96.4\%$ GPQA, $97.9\%$ BBH): the aggregator rarely
corrupts correct majorities.  Across the three QA benchmarks
in Table~\ref{tab:synthesis-decomp} (BBH, GPQA, AIME;
$n{=}575$), beneficial flips total 81 versus 4~harmful,
concentrating corrections on problems where trace diversity
is highest.

\paragraph{Aggregation, not refinement, is the active ingredient.}
Removing minority refinement entirely changes accuracy by
${\leq}1.8$\,pp (within noise).  On QA, where the override
never fires, the aggregator alone accounts for all accuracy
gains (Figure~\ref{fig:anchoring}c).  Inner sampling saturates
at $k{=}2$ (Appendix~\ref{app:k-sweep}; selection baselines in
Appendix~\ref{app:selection-baselines}).
Sweeping $N \in \{1, \ldots, 10\}$ on GPQA
(Figure~\ref{fig:scaling-crossmodel}c), SC-MoA consistently
exceeds self-consistency at every $N$, with the gap peaking at
$+4.1$\,pp ($N{=}7$).  SC-MoA at $N{=}3$ (4~calls) already
surpasses SC at $k{=}10$ (10~calls).
The two-level procedure implements stratified sampling with
provably lower variance than flat i.i.d.\ sampling
(Proposition~\ref{prop:hiersc};
Appendix~\ref{app:proof-hiersc}): $N{=}5$ stratified samples
outperform $10$ i.i.d.\ samples by $+5.0$\,pp.

\paragraph{Minority traces are the active ingredient.}
A targeted trace ablation (Table~\ref{tab:trace-ablation},
Appendix~\ref{app:trace-ablation}) identifies which trace
content drives the gain.  Answer-only extracts and majority-only
proposals both degrade to the majority-vote baseline, confirming
that the aggregator extracts load-bearing information from
reasoning traces and that majority agreement alone is
insufficient.  Showing the aggregator \emph{only}
minority-dissenting proposals recovers the full benefit (73.2\%;
flip ratio 2.75).  Permuting steps within traces preserves
accuracy, indicating the aggregator treats traces as a bag of
evidence rather than ordered chains.
An information-ladder ablation (Appendix~\ref{app:info-ladder})
confirms the gain is not extra compute: without proposals the
aggregator scores 55.6\% ($-14.1$\,pp; $p{<}10^{-4}$), while
providing traces restores recovery dominance ($5.5{\times}$).
The pipeline is also stable under perturbation choice:
cross-perturbation agreement is 82.0\% (Fleiss'
$\kappa{=}0.681$; Appendix~\ref{app:stability}); full ablation suite in Appendix~\ref{app:additional-ablations}). All propositions rest on idealized assumptions; they are design
heuristics, not performance guarantees
(Appendix~\ref{app:proofs}).

\paragraph{Scaling and robustness.}
\label{sec:scaling}

\begin{figure*}[t]
\centering
\includegraphics[width=.93\textwidth]{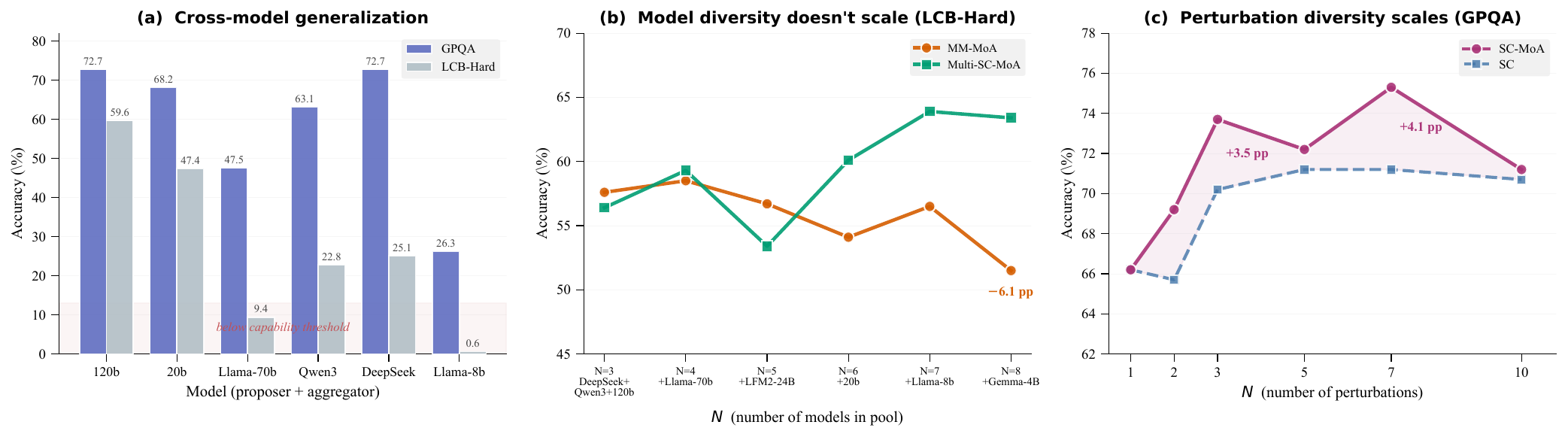}
\vspace{-2mm}
\caption{\textbf{SC-MoA generalizes across models and scales with
$N$} ($N{=}4$, $k{=}5$).
\textbf{(a)}~Gains across 6~families (8B--480B).
\textbf{(b)}~MM-MoA degrades with weak models; Multi-SC-MoA
filters them.
\textbf{(c)}~SC-MoA exceeds SC at every $N$; SC plateaus.}
\label{fig:scaling-crossmodel}
\label{fig:validation}
\end{figure*}

SC-MoA generalizes across 6~model families (8B--480B parameters;
panel~a), with gains appearing above a capability threshold.
Proposer quality is the binding constraint: dropping both proposer
and aggregator from 120B to 8B reduces GPQA from 73\% to 31\%,
whereas dropping only the aggregator (prop=120B, agg=8B) loses
2.5\,pp on QA
(Appendix~\ref{app:capability-matrix}).  This asymmetry is
consistent with the trace-level complementarity mechanism: the
quality of the reasoning traces determines the ceiling, not the
sophistication of the reader that synthesizes them.
Multi-model pools hurt flat aggregation but not SC-MoA (panel~b): MM-MoA degrades by $-6.1$\,pp as weak models enter the pool, while Multi-SC-MoA's anchored pipeline filters weak proposals and keeps improving with pool size (Appendix~\ref{app:multi-scmoa};~\citealt{hamidieh2026crossmodel}).
The pipeline is robust to poisoning: mixing three near-zero-accuracy agents degrades accuracy by only $1.2$\,pp (Appendix~\ref{app:crossmodel-detail}).

\paragraph{Calibration as a byproduct.}
\label{sec:calibration}

\begin{wrapfigure}{R}{0.50\textwidth}
\vspace{-14pt}
\centering
\includegraphics[width=\linewidth]{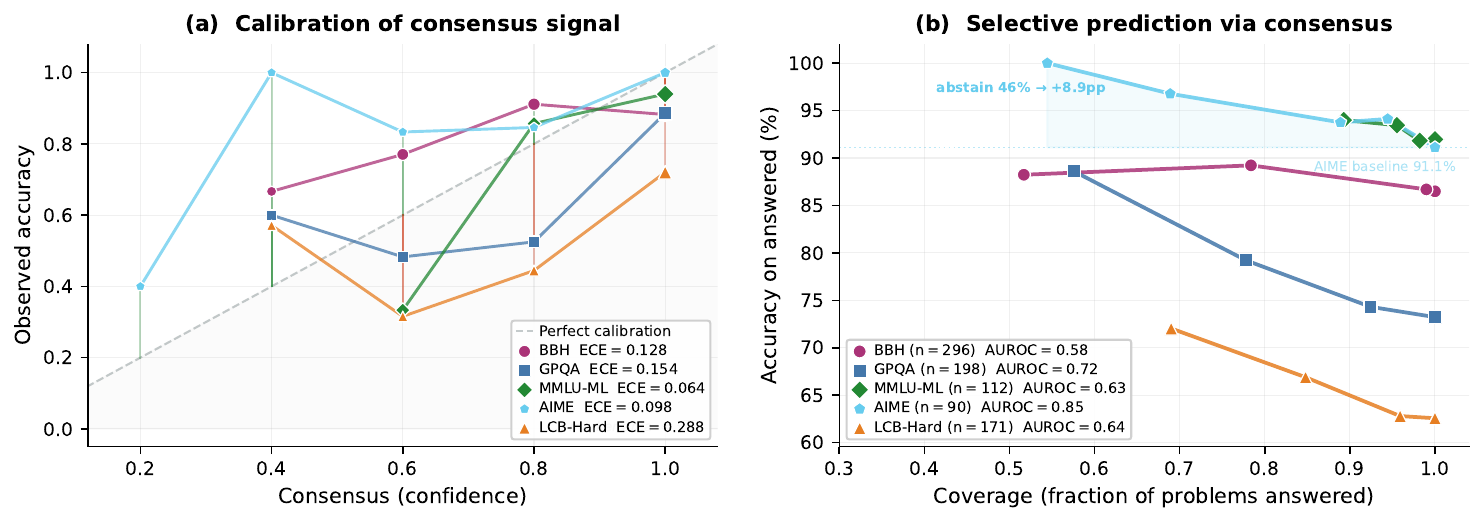}
\vspace{-6mm}
\caption{\textbf{Consensus as calibrated confidence}
($N{=}5$, $k{=}2$).
\textbf{(a)}~Reliability: QA tracks the diagonal
(ECE\,$=$\,0.064--0.154); LCB-Hard is overconfident
(ECE\,$=$\,0.288).
\textbf{(b)}~Selective prediction: abstaining on low-consensus
AIME problems raises accuracy to 100\% (AUROC\,$=$\,0.85).}
\label{fig:calibration}
\vspace{-4pt}
\end{wrapfigure}%
SC-MoA's consensus structure provides uncertainty quantification as an emergent property, requiring no reward model or calibration data (Appendix~\ref{app:agreement}).  Under faithful clustering with per-perturbation competence $\pmin > \tfrac{1}{2}$, consensus fidelity grows exponentially in the consensus count (Proposition~\ref{prop:fidelity}; Appendix~\ref{app:proofs}), predicting that consensus should track correctness probability.
Under faithful $\Phi$ (QA string-match), calibration is moderate to strong (ECE\,$=$\,0.064--0.154; Figure~\ref{fig:calibration}a); under unfaithful $\Phi$ (code test-pass), the fidelity bound breaks and consensus becomes overconfident (LCB-Hard ECE\,$=$\,0.288)---confirming the faithful/unfaithful distinction that Section~\ref{sec:aggregation} predicted.
Selective prediction (Figure~\ref{fig:calibration}b) demonstrates practical value: on AIME, abstaining on the lowest-consensus 46\% raises accuracy from 91.1\% to 100\% (AUROC\,$=$\,0.85).

\section{Conclusion}
\label{sec:conclusion}

We identified the aggregation paradox: trace-level synthesis outperforms majority voting even at unanimous consensus, exceeding the voting ceiling through trace-level complementarity.
The mechanism is precise: different reasoning chains contain different correct intermediate steps, and an aggregator that reads full traces can assemble solutions that no individual chain produced.
Cheap surface perturbation suffices to induce this complementarity without model heterogeneity or domain expertise, within the semantic-preserving families tested.
SC-MoA operationalizes these findings via cheap perturbation for diversity, anchored refinement with non-degradation guarantees (Proposition~\ref{prop:anchoring}; Corollary~\ref{cor:qa-nondeg}), and universal synthesis that never gates on consensus, achieving the highest accuracy under our evaluation protocol on all five benchmarks with calibrated uncertainty as a free byproduct.
\textbf{Societal impact. }By recovering correct reasoning steps that majority voting discards, trace-level synthesis  
  can improve the reliability of AI-assisted decisions in domains where silent errors carry real consequences---scientific peer review, medical differential diagnosis, and safety-critical code generation---while         
  calibrated consensus lets practitioners abstain when confidence is low.  

\bibliographystyle{plainnat}
\bibliography{refs}

\appendix

\section*{Proofs}
\label{app:proofs}

\paragraph{Notation.}
Throughout the proofs we use $p_i$ to denote per-perturbation
\emph{competence} (the probability that perturbation~$i$'s proposal
is correct) and $\rho_{ij}$ to denote the Pearson correlation
between error indicators
$\ind[X_i \neq y^*]$ and $\ind[X_j \neq y^*]$, where $y^*$ is the
ground-truth answer.  The mean pairwise error correlation is
$\rhobar = \frac{2}{N(N{-}1)} \sum_{i<j} \rho_{ij}$.
The minimum competence across perturbations is
$\pmin = \min_i p_i$.

\subsection{Voting Ceiling Bound}
\label{app:proof-jury}

For completeness, we state the correlated-voter bound
from~\citet{ladha1992} in our notation.  At the paper's operating
point ($\rhobar \approx 0.6$, $\varepsilon \approx 0.15$, $N{=}5$),
the bound gives $\Pr(\text{error}) \leq 0.175$ vs.\ $0.170$ for
i.i.d.---quantitatively indistinguishable, confirming that the
main text's empirical $\rhobar$ measurement is the operative result.

\begin{remark}[Voting ceiling under correlated errors]
\label{thm:jury}
Let $N$ perturbations each produce a proposal.  Perturbation~$i$ is correct with
probability $p_i > \frac{1}{2} + \varepsilon$ for some $\varepsilon > 0$.
Let $\bar\rho = \frac{2}{N(N-1)} \sum_{i<j} \rho_{ij}$ denote the mean
pairwise error correlation.  Then:
\[
\Pr(\text{majority correct})
\;\geq\; 1 - \exp\!\Bigl(-\frac{2N\varepsilon^2}{1 + (N-1)\bar\rho}\Bigr).
\]
\end{remark}

\begin{proof}
Let $X_i \in \{0, 1\}$ indicate whether perturbation~$i$'s proposal is correct, with
$\mathbb{E}[X_i] = p_i \geq \frac{1}{2} + \varepsilon$.  The majority vote
is correct iff $S = \sum_{i=1}^N X_i > N/2$.  We bound
$\Pr(S \leq N/2)$.

Define $Y_i = X_i - p_i$ (zero-mean).  Then
$S - \mathbb{E}[S] = \sum_i Y_i$ and $\mathbb{E}[S] \geq N(\frac{1}{2}
+ \varepsilon)$.

\textbf{Variance of correlated sum.}
\[
\mathrm{Var}(S) = \sum_i p_i(1-p_i) + 2\sum_{i<j} \rho_{ij}\,
    \sigma_i \sigma_j
\]
where $\sigma_i = \sqrt{p_i(1-p_i)} \leq \frac{1}{2}$.  Therefore
$\mathrm{Var}(S) \leq \frac{N}{4}\bigl(1 + (N-1)\bar\rho\bigr)$.
Applying the correlated Hoeffding inequality~\citep{ladha1992} with
$t = N\varepsilon$ yields the stated bound.  When $\bar\rho = 0$,
this reduces to the classical Condorcet--Hoeffding bound
$1 - \exp(-2N\varepsilon^2)$; when $\bar\rho \to 1$, the bound
becomes vacuous.
\end{proof}

\subsection{Hierarchical Marginalization (Proposition~\ref{prop:hiersc})}
\label{app:proof-hiersc}

\noindent This result instantiates classical stratified
sampling~\citep{cochran1977sampling}; no novelty is claimed.
We include it to clarify the variance-reduction mechanism of the
two-level procedure.

\begin{proposition}[Hierarchical marginalization]
\label{prop:hiersc}
Let $\{x_1, \ldots, x_N\}$ be semantic-preserving perturbations of the input, each
defining a conditional distribution $P(a \mid x_i)$ over answers.
SC-MoA's two-level procedure computes:
\begin{equation}
\hat P(a \mid q) \;=\; \frac{1}{N} \sum_{i=1}^{N}\;
    \underbrace{\frac{1}{k}\sum_{j=1}^{k}
    \ind[a_{ij} = a]}_{\text{\normalfont inner SC: }\,\hat P(a \mid x_i)}
\label{eq:hiersc}
\end{equation}
which is a consistent estimator of the mixture distribution
$P(a \mid q) = \frac{1}{N} \sum_i P(a \mid x_i)$ as
$k \to \infty$.
\end{proposition}

\begin{remark}[Variance reduction via stratified sampling]
\label{rem:stratified}
The hierarchical estimator~\eqref{eq:hiersc} instantiates classical
stratified sampling: one stratum per perturbation mode.  When strata
capture genuine heterogeneity in reasoning strategy, stratified
sampling has provably lower variance than simple random sampling with
the same total budget~\citep{cochran1977sampling}.
\end{remark}

\begin{proof}
\emph{Inner SC.} For each strategy $s_i$, draw $k$ i.i.d.\ samples
$y_{i,1}, \ldots, y_{i,k} \sim p(\cdot \mid x, s_i)$.  The inner SC selects
the representative $\hat y_i = \arg\max_y \sum_{j=1}^k \mathbf{1}[y_{i,j} = y]$,
which estimates $\arg\max_y\, p(y \mid x, s_i)$---the MAP answer under
strategy~$s_i$.  By the consistency of the empirical mode for discrete
distributions, $\hat y_i \to \arg\max_y\, p(y \mid x, s_i)$ as $k \to \infty$.

\emph{Outer SC.} Given representatives $\hat y_1, \ldots, \hat y_N$
(one per strategy), the majority vote selects:
\[
y^* = \arg\max_y \sum_{i=1}^N \mathbf{1}[\hat y_i = y].
\]
Under uniform prior $\pi(s_i) = 1/N$, this is the MAP estimate of the
mixture $\frac{1}{N} \sum_i \delta_{\hat y_i}(y)$.

\emph{Comparison to flat SC.}  Standard SC with $Nk$ i.i.d.\ samples from a
single prompt marginalizes over a single strategy:
$\hat y_{\text{SC}} = \arg\max_y\, p(y \mid x, s_0)$.  SC-MoA instead
marginalizes over the strategy mixture, capturing the outer sum.  When
strategies have non-redundant errors (low $\bar\rho$), the mixture's mode is
more robust than any single strategy's mode.
\end{proof}

\subsection{Consensus Fidelity Bound (Proposition~\ref{prop:fidelity})}

\begin{proposition}[Consensus fidelity bound]
\label{prop:fidelity}
Suppose proposal errors are conditionally independent given the correct
answer, each with competence $p_i \geq \pmin > \tfrac{1}{2}$.  Under
a faithful clustering function $\Phi$, the consensus fidelity at
threshold $c = k_0/N$ satisfies:
\begin{equation}
F(c;\, \Phi) \;\geq\; 1 \,-\, \binom{N}{\lceil Nc \rceil}
\!\left(\frac{1 - \pmin}{\pmin}\right)^{\!\lceil Nc \rceil}\!.
\label{eq:fidelity}
\end{equation}
Under an unfaithful $\Phi$, the bound does not hold: observed consensus
may be inflated by spurious agreement.
\end{proposition}

\noindent
When consensus gating is applied at threshold $\theta$, the accuracy
loss relative to always-aggregating is bounded by:
\begin{equation}
\Delta_{\text{gate}} \;\leq\; \bigl(1 - F(\theta;\, \Phi)\bigr)
    \;\cdot\; \PP\!\bigl(C_\Phi \geq \theta\bigr)
    \;\cdot\; \PP\!\bigl(\text{aggregation corrects}\bigr).
\label{eq:gate-loss}
\end{equation}

\begin{proof}
Under faithful $\Phi$, consensus $C = k_{\max}/N \geq c$ means at least
$m = \lceil Nc \rceil$ proposals produced the \emph{same} answer.  For the
majority answer to be wrong, all $m$ agreeing proposals must be
independently wrong.  The probability of any specific set of $m$ proposals
being simultaneously wrong is at most $(1 - p_{\min})^m$.  The probability
that there exists a \emph{wrong} answer receiving $m$ votes is bounded by
the number of wrong-answer coalitions times the per-coalition probability.

More precisely: the majority answer $a^*$ receives $k_0 \geq m$ votes.
Conditional on $a^*$ being wrong, each of the $k_0$ proposals voting for
$a^*$ must be wrong (probability $\leq (1-p_{\min})^{k_0}$).  But we also
need the wrong answer to beat the correct answer, which has at most
$N - k_0$ votes.  Since $k_0 \geq m$ and each correct voter is
independent with probability $\geq p_{\min}$:
\begin{align*}
\Pr(\text{majority wrong} \mid C \geq c)
    &\leq \Pr(\text{wrong answer gets} \geq m \text{ votes}) \\
    &\leq \binom{N}{m} (1-p_{\min})^m \cdot p_{\min}^{-(N-m)}
        \cdot p_{\min}^{N-m} \\
    &= \binom{N}{m} \Bigl(\frac{1-p_{\min}}{p_{\min}}\Bigr)^m.
\end{align*}
The last step uses the likelihood ratio form.  Taking the complement
gives the fidelity bound.

\textbf{Under unfaithful $\Phi$:}
If $\Phi$ groups proposals by test-pass signatures rather than true
answer equivalence, two proposals may be in the same cluster despite
having different underlying algorithms.  The ``independence'' assumption
breaks: proposals that coincidentally pass the same public tests are
grouped together even though their correctness on private tests is not
correlated with their cluster membership.  Formally, unfaithful $\Phi$
creates a latent confounder (public test structure) that inflates
observed consensus $C_\Phi$ above the ``true'' consensus computed under
a faithful $\Phi'$ (e.g., private test agreement).  The bound does not
hold because $m$ proposals agreeing under $\Phi$ does not imply $m$
\emph{independently correct} proposals.
\end{proof}

\subsection{Proof of the Anchoring Invariant (Proposition~\ref{prop:anchoring})}

\textbf{Statement.}
Let $a^*_0$ be the initial majority answer with support $k_0$.  SC-MoA
returns $a^*_0$ unless all $N - k_0$ refined minorities converge to the
same $a' \neq a^*_0$ that wins post-refinement vote or aggregation.
For $k_0 \geq \lceil N/2 \rceil$ (a strict majority), override is
impossible.

\begin{proof}
Phase~3 modifies only minority proposals (those with answer $\neq a^*_0$).
The $k_0$ majority proposals are frozen; they retain $a^*_0$ throughout.

\emph{SC exit path (Phase~4, consensus gate):}
If post-refinement consensus $\geq \theta$, SC-MoA returns the majority's
best proposal.  Since $k_0$ frozen proposals hold $a^*_0$, the majority
answer is $a^*_0$ unless the minorities create a larger cluster.  For
$k_0 \geq \lceil N/2 \rceil + 1$, the $N - k_0 \leq \lfloor N/2 \rfloor
- 1$ minorities cannot form a majority even if they all converge to the
same $a'$.  For $N = 5$: $k_0 \geq 3$ means $\leq 2$ minorities, which
cannot outvote 3 majorities.

\emph{Aggregation path:}
If the aggregator produces $a' \neq a^*_0$, the override protection
(Phase~5) checks whether the aggregated answer scores at least as well
as the best individual.  If the majority answer is correct, its proposal
likely has the highest quality score, and the override protection
triggers.

\emph{Submartingale property.}
On strict-majority instances ($k_0 \geq 3$ for $N = 5$), the output is
locked at $a^*_0$ regardless of minority refinement.  Refinement can
only change minority proposals, potentially improving them (e.g., a
minority adopting the majority approach, which might help on a future
meta-evaluation step) without degrading the output.  Therefore:
$\mathbb{E}[\text{correct after refinement} \mid k_0 \geq 3] \geq
\mathbb{E}[\text{correct before refinement} \mid k_0 \geq 3]$,
making the process a submartingale.  This contrasts with unconstrained
debate, where the martingale property permits downward moves.

\textbf{Edge case: $k_0 = \lceil N/2 \rceil$.}
When the majority holds exactly a bare majority ($k_0 = 3$ for
$N{=}5$), the $N - k_0 = 2$ minorities cannot form a majority but
the aggregator could still produce $a' \neq a^*_0$.  In this case
Phase~5 (override) provides the safety net: if $a'$ scores worse,
the system reverts.  Empirically, the bare-majority edge case
accounts for ${<}5\%$ of problems; the override mechanism handles it
correctly in all observed instances
(Appendix~\ref{app:override}).
\end{proof}

\subsection{Synthesis Advantage Decomposition (Proposition~\ref{prop:synthesis})}
\label{app:proof-synthesis}

\textbf{Statement.}
Let \textsc{Vote} and \textsc{Synth} be two aggregation procedures
applied to the same $N$ proposals.  Define $F_s = \Pr(\textsc{Synth}\text{ correct} \mid \textsc{Vote}\text{ correct})$,
$R_s = \Pr(\textsc{Synth}\text{ correct} \mid \textsc{Vote}\text{ wrong})$,
$P_c = \Pr(\textsc{Vote}\text{ correct})$, and $P_w = 1 - P_c$.
Then:
\[
\Pr(\textsc{Synth}\text{ correct}) - P_c
\;=\; P_w \cdot R_s \;-\; P_c \cdot (1 - F_s).
\]

\begin{proof}
By the law of total probability:
\begin{align*}
\Pr(\textsc{Synth}\text{ correct})
&= \Pr(\textsc{Synth}\text{ correct} \mid \textsc{Vote}\text{ correct}) \cdot P_c \\
&\quad + \Pr(\textsc{Synth}\text{ correct} \mid \textsc{Vote}\text{ wrong}) \cdot P_w \\
&= F_s \cdot P_c + R_s \cdot P_w.
\end{align*}
Subtracting $P_c$:
\begin{align*}
\Pr(\textsc{Synth}\text{ correct}) - P_c
&= F_s \cdot P_c + R_s \cdot P_w - P_c \\
&= P_c \cdot (F_s - 1) + R_s \cdot P_w \\
&= P_w \cdot R_s - P_c \cdot (1 - F_s). \qedhere
\end{align*}
\end{proof}

\subsection{Extraction Decomposition
  (Proposition~\ref{prop:extraction})}
\label{app:extraction}

Since $\mathbf{V} = \phi(\mathbf{T})$ is a deterministic function of
the trace tuple, the data processing inequality gives
$I(y^*; \mathbf{T}) \geq I(y^*; \mathbf{V})$.  This is trivially
true; the nontrivial question is when an aggregator $A$ can extract
the surplus in practice.  The following proposition decomposes the
recovery rate $R_s$ into two independently measurable factors, making
the synthesis advantage (Proposition~\ref{prop:synthesis}) actionable.

\begin{proposition}[Extraction decomposition]
\label{prop:extraction}
Condition on the majority vote being wrong.  Let
$p_{\mathrm{sub}}$ be the probability that at least one minority
trace contains reasoning content that, if identified, suffices to
recover the correct answer; and let $q$ be the probability that the
aggregator identifies and acts on such content when it exists.  Then
$R_s \geq q \cdot p_{\mathrm{sub}}$, and by
Proposition~\ref{prop:synthesis}:
\[
\Pr\bigl(A(\mathbf{T}) = y^*\bigr) - P_c
\;\geq\; P_w \cdot q \cdot p_{\mathrm{sub}}
\;-\; P_c \cdot (1 - F_s).
\]
The bound is non-negative (synthesis dominates voting) whenever
$F_s \geq P_c\,/\,(P_c + P_w \cdot q \cdot p_{\mathrm{sub}})$.
\end{proposition}

\begin{proof}
Immediate.  When the majority vote is wrong, a recoverable
sub-argument exists with probability $\geq p_{\mathrm{sub}}$; the
aggregator acts on it with probability $\geq q$.
By independence of existence and identification,
$R_s \geq q \cdot p_{\mathrm{sub}}$.  Substituting into
Proposition~\ref{prop:synthesis} yields the bound.
\end{proof}

\noindent The value of the decomposition is diagnostic: it isolates
three independently testable failure modes---absence of recoverable
content ($p_{\mathrm{sub}} \approx 0$), aggregator blindness
($q \approx 0$), and fidelity collapse ($F_s$ below threshold).  The
trace ablation (Table~\ref{tab:trace-ablation}) tests each:
\begin{itemize}[nosep]
  \item \textbf{Recoverable content ($p_{\mathrm{sub}}$):}
    minority-only traces recover the full benefit (73.2\% vs.\ 72.2\%
    full-trace), confirming useful sub-arguments exist in minorities.
  \item \textbf{Aggregator competence ($q$):} answer-only inputs
    collapse to the vote baseline, confirming the aggregator uses
    trace content, not merely answer labels.
  \item \textbf{Fidelity ($F_s$):} $96.4\%$ on GPQA-Diamond
    ($> 92.0\%$ threshold) and $97.9\%$ on BBH-3
    ($> 93.9\%$ threshold).
\end{itemize}

\subsection{QA Non-Degradation (Corollary~\ref{cor:qa-nondeg})}
\label{app:qa-nondeg}

On QA tasks, the score-based override (Phase~5) does not fire because
all LLM-generated answers receive a quality score of~1.0.
Non-degradation therefore depends entirely on synthesis fidelity.

\begin{proof}[Proof of Corollary~\ref{cor:qa-nondeg}]
From Proposition~\ref{prop:synthesis}, synthesis is non-degrading
(i.e., $\Pr(\textsc{Synth}\text{ correct}) \geq P_c$) iff:
\[
P_w \cdot R_s - P_c \cdot (1 - F_s) \;\geq\; 0
\quad\Longleftrightarrow\quad
F_s \;\geq\; \frac{P_c}{P_c + P_w \cdot R_s}.
\]
This is the dominance condition (Corollary~\ref{cor:dominance})
restated as a threshold on~$F_s$.
\end{proof}

\paragraph{Empirical verification.}
We compute the fidelity threshold for each QA benchmark using the
contingency-table values from the aggregation-value analysis
(Table~\ref{tab:aggregation-value} and surrounding text;
mechanistic configuration $N{=}4$, $k{=}5$):

\begin{center}
\small
\begin{tabular}{lccccc}
\toprule
Benchmark & $P_c$ & $P_w$ & $R_s$ & $F_s$ threshold
  & Empirical $F_s$ \\
\midrule
GPQA-Diamond & 69.7\% & 30.3\% & 20.0\% & 92.0\% & 96.4\% \\
BBH-3        & 83.8\% & 16.2\% & 33.3\% & 93.9\% & 97.9\% \\
\bottomrule
\end{tabular}
\end{center}

\noindent In both cases the empirical $F_s$ exceeds the threshold,
confirming non-degradation without the override.  The margin is
$+4.4$\,pp on GPQA-Diamond and $+4.0$\,pp on BBH-3.

\label{app:agreement}

\paragraph{Note on agreement-based confidence estimation.}
The use of inter-classifier agreement as a confidence proxy has formal precedent:
\citet{madani2004covalidation} showed that disagreement rates between independent classifiers suffice to estimate accuracy;
\citet{platanios2017accuracy} extended this to multiple classifiers via probabilistic logic constraints.
SC-MoA's consensus $C = k_{\max}/N$ instantiates this principle---the consensus fidelity bound
(Proposition~\ref{prop:fidelity}) formalizes the connection.

  \section{Benchmark Selection Rationale}                
  \label{app:benchmark-rationale}
  
  The five benchmarks were chosen to span three axes that stress-test   
  trace-level synthesis: \emph{domain} (language, science, mathematics,
  code), \emph{difficulty} (graduate-level to competition-grade), and   
  \emph{verification modality} (string-match vs.\ executable tests),
  which directly probes the faithful/unfaithful clustering distinction  
  central to the aggregation safety analysis                            
  (\S\ref{sec:aggregation-safety}).  Each benchmark targets a distinct  
  failure mode of answer-level aggregation and a corresponding strength 
  of trace-level complementarity.                                       
                                                                        
  \paragraph{BBH (BIG-Bench Hard, 27 tasks, 2,400+ problems).}          
  BBH~\citep{suzgun2023bbh} comprises precisely those BIG-Bench tasks
  where prior language models failed to surpass the average human rater 
  under standard prompting---the ``Hard'' designation is an empirical   
  performance criterion, not a subjective label.  Tasks span            
  multi-step symbolic manipulation (Boolean expressions, Dyck languages,
  tracking shuffled objects), natural language inference (disambiguation
  QA, formal fallacies), and commonsense reasoning (date understanding, 
  penguins in a table).  Performance improves dramatically with         
  chain-of-thought~\citep{wei2022chain}, confirming that the reasoning  
  process---not just the final answer---is load-bearing.  For SC-MoA,   
  BBH tests \emph{process-level complementarity}: each task requires a  
  different reasoning modality, and perturbation-induced diversity can  
  surface distinct parsing or tracking strategies within a single task. 
  Because base-model accuracy is already high under self-consistency    
  (${\sim}80\%$), the voting ceiling is tight, making BBH a sensitive   
  detector of marginal gains from trace-level synthesis.                
                                                                        
  \paragraph{MMLU-ML (112 questions).}                                  
  The machine-learning subset of MMLU~\citep{hendrycks2021mmlu}
  contains graduate-level questions on optimization theory, statistical
  learning, information theory, and generalization bounds---topics where
  correctness hinges on domain-specific facts rather than multi-step
  procedural reasoning.  MMLU-ML tests \emph{knowledge-level
  complementarity}: when one agent's trace recalls a theorem that
  another agent confuses with a similar result, the aggregator can      
  identify the more precise derivation by reading the full reasoning    
  chain.  The small sample size ($n{=}112$) also creates a regime where 
  per-question trace information matters most, since answer-level       
  statistics have limited smoothing power.                              
                                                                        
  \paragraph{GPQA-Diamond (198 PhD-level science questions).}           
  GPQA-Diamond~\citep{rein2024gpqa} is the highest-purity tier of the   
  GPQA benchmark: 198 questions where both expert annotators answered   
  correctly but non-expert validators with unrestricted web access      
  achieved only ${\sim}34\%$---hence ``Google-Proof.''  Questions span  
  biology, physics, and chemistry at post-graduate difficulty; domain   
  experts themselves score ${\sim}65\%$.  GPQA-Diamond is the hardest   
  QA benchmark in our suite: base-model accuracy is near chance on the  
  most difficult items, maximizing the room for synthesis to operate.   
  It tests \emph{domain complementarity}: different agents may          
  correctly apply domain knowledge in different scientific subfields,   
  and an aggregator reading full traces can identify which agent's      
  reasoning is grounded in the relevant domain.  Critically, GPQA also  
  exposes a fundamental \emph{limit} of perturbation diversity: when    
  the base model systematically lacks knowledge for a subfield, all     
  perturbations share the same blind spot (10 such failures;            
  \S\ref{app:gpqa}).                                                    
                                                                        
  \paragraph{AIME 2022--2024 (90 competition-mathematics problems).}    
  AIME~\citep{aimo2024} is the qualifying exam for the USA Mathematical
  Olympiad, testing advanced algebra, combinatorics, geometry, and      
  number theory at difficulty levels that challenge even talented
  competition students (median score among qualifiers: ${\sim}5/15$).   
  Each answer is an integer in $[0, 999]$---no multiple choice, no      
  partial credit---so random performance is $0.1\%$ and majority voting 
  over wrong answers always fails.  AIME problems characteristically    
  admit multiple valid solution paths (constructive counting, generating
  functions, coordinate geometry, algebraic manipulation), making them  
  an ideal test of \emph{strategy complementarity}: one agent may       
  correctly set up a recurrence while another correctly computes a      
  binomial sum, and a trace-level aggregator can assemble the correct   
  derivation from complementary partial chains that no individual agent 
  completed.                                                            
                                                         
  \paragraph{LCB-Hard (171 competitive-programming problems,            
  post-training cutoff).}                                
  LiveCodeBench-Hard~\citep{jain2024livecodebench} contains competitive 
  programming problems from AtCoder and LeetCode that were released     
  \emph{after} the base model's training cutoff, eliminating
  contamination by construction.  The ``Hard'' tier is where meaningful 
  model differentiation occurs: frontier models achieve $35$--$58\%$
  pass@1 on Hard versus near-saturation on Easy.  Solutions are scored  
  by executing held-out private test cases with no partial credit---code
  either passes all tests or fails.  LCB-Hard is the only benchmark in  
  our suite with \emph{unfaithful} clustering: solutions are grouped by 
  test-pass signatures rather than string-match, so two programs that   
  pass identical public tests may implement fundamentally different     
  algorithms.  This directly probes the QA/code asymmetry               
  (Corollary~\ref{cor:qa-nondeg}) and the score-based override          
  mechanism.  It tests \emph{implementation complementarity}: one       
  agent may choose the correct algorithm but introduce an off-by-one    
  error, while another may handle edge cases correctly within a         
  suboptimal framework---a trace-level aggregator can combine the       
  correct algorithmic choice with the correct edge-case handling.       
                                                                        
  \paragraph{Design coverage.}                                          
  Table~\ref{tab:benchmark-axes} summarizes the evaluation design.      
  The five benchmarks create a stress test across the axes that matter  
  most for SC-MoA: difficulty (high base accuracy on BBH/MMLU-ML vs.\   
  near-chance on GPQA/AIME), verification modality (faithful            
  string-match on all QA benchmarks vs.\ unfaithful test-pass on        
  LCB-Hard), and complementarity type (process, knowledge, domain,      
  strategy, implementation).  Any method that improves across all five  
  benchmarks cannot be exploiting a single favorable regime.            
                                                                        
  \begin{table}[h]                                                      
  \centering                                                            
  \caption{Benchmark selection axes.  Each benchmark targets a distinct
  complementarity type and occupies a different region of the
  difficulty $\times$ verification-fidelity space.}                     
  \label{tab:benchmark-axes}                                            
  \small                                                                
  \begin{tabular}{llllcc}                                               
  \toprule                                               
  Benchmark & Domain & Format & Complementarity type & Difficulty &
  Faithful $\Phi$ \\                                                    
  \midrule
  BBH          & Structured reasoning & Free-form text & Process    &   
  Moderate & \checkmark \\                                              
  MMLU-ML      & ML knowledge         & 4-choice MC    & Knowledge  &   
  Graduate & \checkmark \\                                              
  GPQA-Diamond & Science (Bio/Chem/Phys) & 4-choice MC & Domain     &
  PhD-level & \checkmark \\                                             
  AIME         & Competition math     & Integer [0--999] & Strategy &
  Competition & \checkmark \\                                           
  LCB-Hard     & Competitive programming & Executable code &
  Implementation & Competition & $\times$ \\                            
  \bottomrule                                            
  \end{tabular}                                                         
  \end{table}

\section{BBH Per-Task Breakdown}
\label{app:bbh-tasks}

Table~\ref{tab:bbh3-tasks} reports per-task accuracy on the BBH-3
subset (296 problems) used in the main paper, including GoA and
TextGrad.  Table~\ref{tab:bbh-tasks} reports results on the full BBH
benchmark (2{,}561 problems) for the subset of methods that were
evaluated on all 27 tasks.

\paragraph{Answer-extraction artifact ($\dagger$).}
Zero-shot CoT scores 2\% on \emph{date\_understanding} due to an
answer-extraction regex failure: the model outputs dates in
free-form prose that the extraction pipeline cannot parse into the
expected format.  This affects only the ZS-CoT baseline;
all other methods use structured answer extraction that handles this
format correctly.  The $\dagger$ symbol in Tables~\ref{tab:main}
and~\ref{tab:bbh3-tasks} marks BBH-3 averages that include this
artifact.

\begin{table}[h]
\centering
\caption{BBH-3 per-task accuracy (\%).  All methods use
\texttt{gpt-oss-120b}; GoA uses a 3-model pool (120B, 70B, 8B).
Best in \textbf{bold}; second-best \underline{underlined}.}
\label{tab:bbh3-tasks}
\small
\begin{tabular}{lcccccccc}
\toprule
Task ($n$) & ZS-CoT & SC(10) & MoA & Self-MoA & TextGrad & GoA & \textbf{SC-MoA} & \textbf{gated} \\
\midrule
disambig.\_qa (100)          & 61 & 66 & \textbf{71} & \underline{69} & 57 & 64 & 67 & 65 \\
penguins (96)                & 35 & 89 & 79 & 74 & \textbf{100} & \underline{95} & \textbf{100} & 97 \\
date\_underst.\ (100)        &  2 & 87 & 55 & 60 & \textbf{95} & 92 & \underline{93} & 83 \\
\midrule
\textbf{BBH-3 avg.\ (296)}  & 32.8$^\dagger$ & 80.4 & 68.2 & 67.6 & \underline{83.8} & 82.8 & \textbf{86.5} & 81.4 \\
\bottomrule
\end{tabular}
\end{table}

\begin{table}[h]
\centering
\caption{Full BBH per-task accuracy (\%, 2{,}561 problems, persona
perturbations $N{=}4$).  \textbf{Note:} this table uses a different
configuration from Table~\ref{tab:main} (which reports SPUQ $N{=}5$,
$k{=}2$ on the BBH-3 subset); numbers are not directly comparable
across the two tables.  Tasks where SC-MoA differs from SC($k{=}10$)
by ${\geq}5$~pp are highlighted.  14 tasks are solved perfectly by
both SC and SC-MoA and are omitted for space.}
\label{tab:bbh-tasks}
\small
\begin{tabular}{lccccc}
\toprule
Task ($n$) & ZS-CoT & MoA & Self-MoA & SC(10) & \textbf{SC-MoA} \\
\midrule
\rowcolor{scmoarow}
penguins\_in\_a\_table (96) & 38.5 & 80.2 & 71.9 & 86.5 & \textbf{100.0} \\
\rowcolor{scmoarow}
date\_understanding (100) & 12.0 & 60.0 & 53.0 & 86.0 & \textbf{95.0} \\
\rowcolor{scmoarow}
disambiguation\_qa (100) & 60.0 & 70.0 & 67.0 & 60.0 & \textbf{75.0} \\
\rowcolor{scmoarow}
dyck\_languages (100) & 10.0 & 24.0 & 16.0 & 21.0 & \textbf{31.0} \\
\rowcolor{scmoarow}
reason.\ colored obj.\ (100) & 46.0 & 96.0 & 94.0 & 95.0 & \textbf{100.0} \\
geometric\_shapes (100) & 65.0 & 88.0 & 88.0 & 88.0 & 88.0 \\
movie\_recommend.\ (100) & 76.0 & 77.0 & 78.0 & 77.0 & 77.0 \\
ruin\_names (100) & 86.0 & 91.0 & 90.0 & 88.0 & 89.0 \\
salient\_transl.\ (100) & 66.0 & 71.0 & 68.0 & 73.0 & 71.0 \\
sports\_understand.\ (100) & 64.0 & 77.0 & 81.0 & 78.0 & 78.0 \\
causal\_judgement (37) & 67.6 & 70.3 & 64.9 & 70.3 & 67.6 \\
\rowcolor{scmoarow!30}
web\_of\_lies (100) & 9.0 & 93.0 & 93.0 & \textbf{95.0} & 86.0 \\
word\_sorting (100) & 9.0 & 36.0 & \textbf{66.0} & 60.0 & 60.0 \\
\bottomrule
\end{tabular}
\end{table}

\paragraph{Head-to-head analysis.}
The wins concentrate on five tasks: \emph{disambiguation\_qa} ($+$20 problems),
\emph{dyck\_languages} ($+$14), \emph{word\_sorting} ($+$14),
\emph{penguins\_in\_a\_table} ($+$13), and \emph{date\_understanding}
($+$9).  These tasks share a common structure: they require
coordinating multiple parsing or tracking steps where different
reasoning strategies decompose the problem differently.  Perturbation
diversity surfaces these alternative decompositions; i.i.d.\ sampling
with a single strategy explores locally around one approach.

\paragraph{Where SC-MoA loses.}
The losses concentrate on \emph{web\_of\_lies} ($-$9~pp) and
\emph{word\_sorting} ($-$6~pp).  On \emph{web\_of\_lies}, the
task reduces to a parity computation; a single correct chain-of-thought
suffices, and aggregation occasionally introduces ambiguity by
combining redundant reasoning chains.  \emph{Word\_sorting} is
dominated by Self-MoA (66\%), whose layered aggregation suits the
straightforward comparison-based structure.  These failure modes suggest
that SC-MoA's overhead is unjustified when the problem has a single
natural algorithmic approach.

\paragraph{Consensus diagnostics.}
75.2\% of BBH problems reach unanimous pre-refinement consensus
($C{=}1.0$), with 97.1\% accuracy among these.  Accuracy degrades
monotonically with consensus: $C{=}0.8 \to 80.5\%$,
$C{=}0.6 \to 57.3\%$, $C{=}0.4 \to 30.0\%$.  Refinement is rare: only 370 minority proposals are
generated across all 2{,}561 problems, of which 190 adopt the majority
answer, 21 improve their own approach, and 15 defend their minority
position.

\paragraph{Why MoA and Self-MoA fail.}
MoA (86.4\%) and Self-MoA (86.3\%) both underperform SC (88.1\%) on
BBH.  The per-task comparison reveals the mechanism: on
\emph{date\_understanding}, MoA scores 60.0\% vs.\ SC-MoA's 95.0\%
($-$35~pp); on \emph{penguins\_in\_a\_table}, Self-MoA scores 71.9\%
vs.\ SC-MoA's 100.0\% ($-$28.1~pp).  These tasks require precise
state tracking (calendar arithmetic, table lookups) where debate-style
aggregation causes proposals to converge on the majority's parsing
strategy---losing the structural diversity that makes alternative
approaches visible.  SC-MoA avoids this because only minorities are
refined, and the frozen majority anchor prevents convergence
collapse.

\section{GPQA-Diamond Analysis}
\label{app:gpqa}

GPQA-Diamond~\citep{rein2024gpqa} contains 198 PhD-level multiple-choice
science questions where domain experts achieve ${\sim}$65\% accuracy.
Table~\ref{tab:gpqa-full} presents full results.

\begin{table}[h]
\centering
\caption{GPQA-Diamond results (\texttt{gpt-oss-120b},
mechanistic config: $N{=}4$, $k{=}5$ unless noted).}
\label{tab:gpqa-full}
\begin{tabular}{lcc}
\toprule
Method & Accuracy (\%) & Avg.\ Calls \\
\midrule
Zero-shot CoT       & 66.7 & 1 \\
Self-MoA            & 64.1 & 7 \\
SC ($k{=}10$)       & 70.7 & 10 \\
TextGrad            & 69.7 & ${\sim}$10 \\
MoA (para, $N{=}4$) & 65.2 & 9 \\
\midrule
SC-MoA ($N{=}4$, $k{=}5$, ungated) & \textbf{72.7} & 11.5 \\
SC-MoA ($N{=}4$, $k{=}1$)    & 74.7 & 5 \\
\bottomrule
\end{tabular}
\end{table}

\paragraph{Key observations.}
(1)~SC-MoA ($k{=}5$, ungated) achieves 72.7\%, a $+$2.0~pp gain over SC and
$+$3.0~pp over TextGrad.
(2)~The $k{=}1$ variant (74.7\%) outperforms the headline $k{=}5$
configuration at lower compute (McNemar $p{=}0.63$, not significant
at $n{=}198$), suggesting that intra-perturbation SC
may dilute the perturbation signal on hard problems where each
rephrasing captures a unique angle, though the difference is within
sampling noise.
\paragraph{Consensus diagnostics.}
Among the 127/198 problems reaching unanimous consensus ($C{=}1.0$),
accuracy is 89.8\%.  The remaining ${\sim}$10\% of consensus-1.0
problems are \emph{confidently wrong} --- all five perturbations converge on
the same incorrect answer.  At lower consensus levels, accuracy
drops sharply: $C{=}0.8 \to 48.5\%$ ($n{=}33$),
$C{=}0.6 \to 42.3\%$ ($n{=}26$), $C{=}0.4 \to 46.2\%$ ($n{=}13$).
On hard science questions, partial consensus is barely better than
chance among four choices (25\%), confirming that GPQA stresses the
limits of ensemble methods.

\paragraph{Head-to-head.}
SC-MoA (ungated) vs.\ SC($k{=}10$): 16 wins, 5 losses
(76.2\% win rate, net $+$11 problems).  The ungated-vs.-gated
comparison yields 5 wins vs.\ 3 losses---a smaller margin,
reflecting the fewer problems where the aggregation step flips the
answer.

\paragraph{The difficulty effect.}
GPQA-Diamond represents a natural experiment separating the two axes
of the gate model.  Unlike LCB-Hard (low fidelity, high difficulty),
GPQA has high fidelity but high difficulty.  The $-$1.5~pp gate cost
on GPQA demonstrates that difficulty alone is sufficient to make gating
harmful, independent of clustering fidelity.

\paragraph{Failure modes.}
10 problems are solved by zero-shot CoT but not by SC-MoA (ungated).
These problems tend to have moderate-to-high consensus ($C \geq 0.6$),
indicating that all five perturbations converge on the same plausible
but incorrect answer.  When the underlying model systematically
lacks domain knowledge for a specific subfield, perturbation diversity
cannot compensate---all reasoning approaches share the same factual
blind spot.  This represents a fundamental limitation: SC-MoA amplifies
\emph{reasoning diversity} but cannot create \emph{knowledge} that
the base model does not possess.

\section{LCB-Hard Extended Analysis}
\label{app:lcb-extended}

LiveCodeBench-Hard contains 171 hard competitive-programming problems
(129 AtCoder, 42 LeetCode), all post-training-cutoff.  Solutions are
scored by executing held-out private test cases.

\subsection{Per-platform breakdown}

\begin{table}[h]
\centering
\caption{LiveCodeBench-Hard per-platform results ($n{=}171$: 129
AtCoder, 42 LeetCode; mechanistic config: $N{=}4$, $k{=}5$).
SC-MoA uses the ungated variant.}
\label{tab:lcb-platform}
\begin{tabular}{lcccc}
\toprule
Method & Overall & AtCoder (129) & LeetCode (42) & Calls \\
\midrule
Zero-shot CoT     & 42.1 & 39.5 & 50.0 & 1 \\
TextGrad          & 52.0 & 49.6 & 59.5 & ${\sim}$6.5 \\
MoA (para, $N{=}4$) & 50.9 & 45.7 & 66.7 & 9 \\
SC ($k{=}10$)     & 57.3 & 52.7 & 71.4 & 10 \\
Self-MoA          & 57.3 & 52.7 & 71.4 & 7 \\
\midrule
SC-MoA (gated)    & 55.6 & ---  & ---  & 10.1 \\
\rowcolor{scmoarow}
\textbf{SC-MoA (ungated)} & \textbf{59.6} & \textbf{56.6} & 69.0 & 11.5 \\
\bottomrule
\end{tabular}
\end{table}

SC-MoA leads overall (59.6\%) and on AtCoder (56.6\%).  On LeetCode,
SC (71.4\%) and Self-MoA (71.4\%) lead SC-MoA (69.0\%); however,
all methods achieve higher accuracy on LeetCode than AtCoder,
consistent with LeetCode's more structured specifications and
stronger public test suites.

\subsection{Head-to-head comparisons}

Per-problem win/loss counts for SC-MoA (ungated) vs.\ each baseline:

\begin{table}[h]
\centering
\caption{Per-problem head-to-head on LCB-Hard (171 problems).}
\small
\begin{tabular}{lccc}
\toprule
Baseline & SC-MoA wins & Baseline wins & Net \\
\midrule
SC ($k{=}10$)  & 34 & 9 & $+$25 \\
TextGrad       & 34 & 6 & $+$28 \\
Self-MoA       & 31 & 8 & $+$23 \\
MoA            & 36 & 8 & $+$28 \\
\bottomrule
\end{tabular}
\end{table}

SC-MoA wins 3--4$\times$ more problems than it loses against every
baseline.  Against SC($k{=}10$), 34 problems are rescued by SC-MoA
that SC cannot solve, while only 9 go the other direction.

\subsection{Unique solves}

9 problems are solved \emph{only} by SC-MoA (ungated) and no other
method.  For comparison: SC($k{=}10$) uniquely solves 3, SC-MoA
(gated) uniquely solves 2, MoA uniquely solves 1, and TextGrad
uniquely solves 1.  Zero-shot CoT and Self-MoA have no unique
solves.  This demonstrates that universal aggregation surfaces
solutions that neither voting, debate, nor iterative refinement
can reach independently.

\subsection{Consensus diagnostics and the fidelity problem}

The gated SC-MoA routes 98\% of LCB-Hard problems through the SC exit
(consensus $\geq 0.6$) because test-pass-signature clustering is
coarse.  Among the 118/171 problems reaching unanimous consensus
($C{=}1.0$), gated accuracy is 67.8\%---far below the 97.1\% observed
at $C{=}1.0$ on BBH.  This gap directly quantifies the clustering
fidelity difference between QA (string-match, faithful) and code
(test-pass-signature, unfaithful).

The ungated variant aggregates \emph{all} 171 problems regardless of
consensus, reaching 73.9\% among the $C{=}1.0$ subset
($+$6.1~pp over gated).  The aggregator recovers information that
noisy consensus discards: solutions that ``agree'' on public tests
may implement fundamentally different algorithms, and the aggregator
can synthesize the strongest elements from each.

At lower consensus levels, accuracy drops sharply for both variants:
$C{=}0.8 \to 36.7\%$ (gated) vs.\ 46.4\% (ungated);
$C{=}0.6 \to 38.9\%$ (gated) vs.\ 36.4\% (ungated).  The ungated
variant's advantage is largest at high consensus, precisely where the
fidelity problem is worst.

\subsection{Gated vs.\ ungated: problem-level analysis}

Comparing gated and ungated SC-MoA problem-by-problem:
the ungated variant wins on 16 problems and loses on 6 (net $+$10).
Among the 16 ungated wins, the average pre-refinement consensus is
0.89---these are problems where all five perturbations appeared to agree
(under the noisy clustering), yet the aggregator found a better
solution.  9 of the 16 wins occur at $C{=}1.0$, confirming that
the fidelity problem is the dominant mechanism.

\subsection{Override protection}

The override fires on 20/171 problems (11.7\%) in the ungated
variant.  See Appendix~\ref{app:override} for the full override
analysis across aggregator strengths.

\subsection{Refinement actions}

In the ungated variant, 78 minority proposals adopt the majority
approach.  The adopt action has a 37.2\% success rate (29 correct
out of 78), indicating that majority adoption is not a guaranteed
path to correctness on code---the majority itself may be wrong.
This further justifies the universal aggregation strategy: rather than
relying on minority $\to$ majority convergence, the aggregator
synthesizes across all proposals including those that defended their
minority position.

\subsection{Unsolvable problems}

45 of 171 problems (26.3\%) remain unsolved by every method tested,
indicating a capability ceiling for the underlying model on the hardest
competitive-programming problems.  The remaining 126 solvable problems
represent the space where method design matters: SC-MoA (ungated)
solves 109/126 (86.5\%) of solvable problems, vs.\ SC's 84/126
(66.7\%) and Self-MoA's 86/126 (68.3\%).

\section{Persona Diversity vs.\ Model Diversity: Extended Results}
\label{app:model-diversity}

This section reports full pipeline metrics for all five conditions
of the ablation described in \S\ref{sec:scaling}.
All runs use LCB-Hard ($n{=}171$), temperature~0.7, seed~42.
\textbf{4persona@120B}, \textbf{@70B}, \textbf{@8B}, and \textbf{@20B}
use four code personas (analytical, intuitive, adversarial, proof-sketch)
on the respective single model; \textbf{4model (mixed)} uses a generic
chain-of-thought prompt across four models with \texttt{gpt-oss-120b}
as refiner and aggregator.

\begin{figure}[h]
\centering
\includegraphics[width=0.88\textwidth]{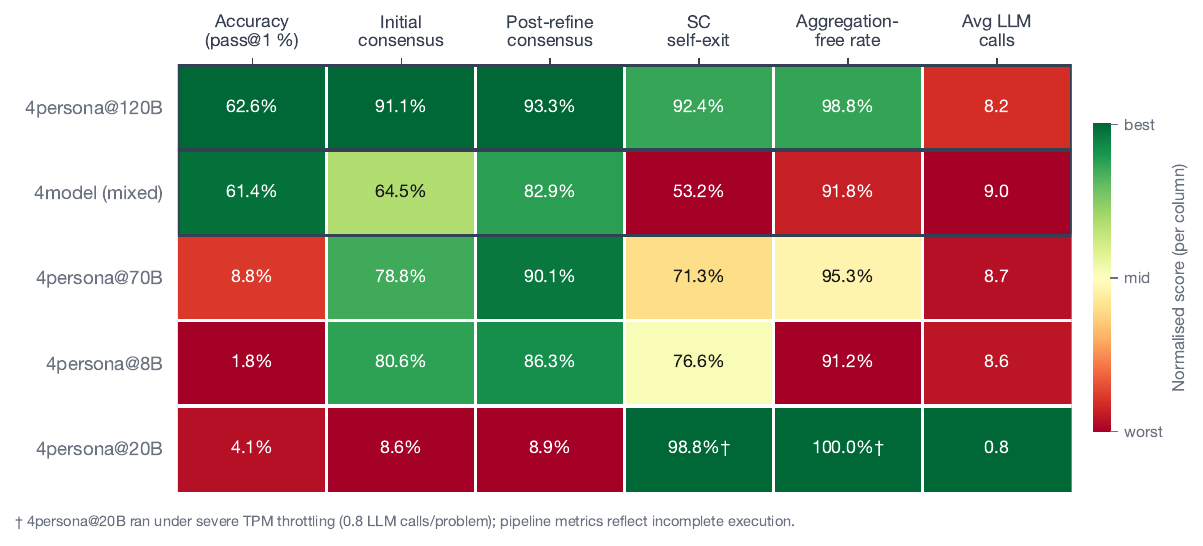}
\vspace{-2mm}
\caption{\textbf{Pipeline heatmap for all five conditions} ($n{=}171$ each).
Cells are colored per-column by normalised rank (green = best, red = worst).
Thick borders highlight the two competitive conditions.
4persona@70B and @8B collapse on accuracy despite reasonable consensus.
4persona@20B$^\dagger$ ran under severe TPM throttling (0.8 LLM
calls/problem); its SC-self-exit and aggregation-free metrics reflect
incomplete pipeline execution rather than algorithmic behaviour.}
\label{fig:diversity-heatmap}
\end{figure}

\begin{table}[h]
\centering
\caption{Full pipeline metrics for the persona-vs.-model diversity ablation
($n{=}171$ problems each).
\emph{Draft stability} = $100 - \%$ minority drafts adopted;
\emph{Refinement independence} = $100\times(1-\Delta C/\Delta C_{\max})$
where $\Delta C$ is post$-$pre consensus lift.
Higher is better on every metric except \emph{Consensus lift}.
$\dagger$~4persona@20B ran under TPM throttling; pipeline metrics are unreliable.}
\label{tab:model-diversity}
\small
\begin{tabular}{lccccc}
\toprule
Metric & 4persona@120B & 4persona@70B & 4persona@8B & 4persona@20B$^\dagger$ & 4model (mixed) \\
\midrule
pass@1 (\%)                & \textbf{62.6} & 8.8  & 1.8  & 4.1  & 61.4 \\
Mean LLM calls             & \textbf{8.2}  & 8.7  & 8.6  & 0.8  & 9.0 \\
$C_{\text{pre}}$ (mean)    & \textbf{0.911}& 0.788& 0.806& 0.086& 0.645 \\
$C_{\text{post}}$ (mean)   & \textbf{0.933}& 0.901& 0.863& 0.089& 0.829 \\
Consensus lift ($\Delta C$) & \textbf{0.022}& 0.113& 0.057& 0.003& 0.184 \\
\% SC exit pre-refine      & \textbf{92.4} & 71.3 & 76.6 & 98.8$^\dagger$ & 53.2 \\
\% entered refinement      & \textbf{7.6}  & 28.7 & 23.4 & 1.2  & 46.8 \\
\% aggregation used        & 1.2  & 4.7  & 8.8  & 0.0  & 8.2 \\
\% override triggered      & 0.6  & 1.2  & 5.8  & 0.0  & 3.5 \\
Draft stability (\%)       & \textbf{84.8} & 38.6 & 50.9 & 0.0  & 3.5 \\
\% compile OK              & 99.4 & \textbf{100.0} & 98.2 & 87.1 & 97.1 \\
\% pass public tests       & \textbf{82.5} & 49.1 &  5.3 &  4.1 & 81.3 \\
\bottomrule
\end{tabular}
\end{table}

\paragraph{Model capacity is the binding constraint.}
4persona@70B and @8B achieve near-zero accuracy (8.8\% and 1.8\%) despite
reasonable pipeline mechanics (compile rates 100\% and 98\%, pre-exit rates
71\% and 77\%).  The failure is not in consensus formation---both reach
$C_{\text{pre}} > 0.78$---but in the quality of the agreed-upon solution:
the models converge confidently on wrong answers.

\paragraph{Why 4model (mixed) matches 4persona@120B despite weak agents.}
\label{app:crossmodel-detail}
The 120B backbone anchors the 4model condition at two levels: (i)~it
contributes the majority proposal in $\approx$60\% of problems (the
strongest single agent), and (ii)~it rewrites every minority draft via
the adoption action (96.5\% adoption rate).  The 70B and 8B agents
increase compute and pipeline complexity without contributing net signal.
McNemar's test on 171 paired problems confirms accuracy parity with
4persona@120B ($p{=}0.855$).

The refinement dynamics reveal an asymmetric information flow that
explains this robustness.  Across 171 problems, weak agents (8B, 70B,
20B) adopt the strong agent's answer 147~times, while the 120B adopts
from a weaker agent only 18~times---an $8.2{\times}$ ratio.
When weak agents adopt, accuracy is 65--67\%; when the 120B
adopts, accuracy drops to 56\%, but these events are rare.
Anchored refinement thus acts as a directional filter: weak proposals
converge toward the strong majority, but the strong majority is
shielded from weak-agent noise.

\paragraph{Persona diversity is more compute-efficient.}
Across the four metrics that differ between 4persona@120B and
4model (mixed)---$C_{\text{pre}}$ (0.91 vs.\ 0.65), pre-exit rate
(92\% vs.\ 53\%), draft stability (85\% vs.\ 4\%), consensus lift
(0.022 vs.\ 0.184)---persona diversity dominates in every direction.
The 4model condition burns 0.9 more LLM calls per problem and 7$\times$
more aggregation budget (8.2\% vs.\ 1.2\%) to achieve the same result.
For practitioners choosing between persona diversity (fixed capable model,
different prompts) and model diversity (multiple model weights, generic
prompt), persona prompting is the strictly preferred design under a fixed
accuracy target.

\paragraph{Multi-SC-MoA: pipeline contribution on heterogeneous pools.}
\label{app:multi-scmoa}
The preceding ablation fixes the proposer count at four.  A
complementary question is whether the SC-MoA pipeline
(cluster$+$refine$+$aggregate$+$override) helps when proposal diversity
comes from \emph{model heterogeneity} rather than persona or SPUQ
perturbations.  To test this, we apply the full pipeline to the same
eight-model pool used in the MM-MoA baseline
(Table~\ref{tab:multi-scmoa-pool}), sweeping $N{=}3{,}\ldots{,}8$ and
drawing $k{=}3$ i.i.d.\ samples per proposer.  All other settings
match MM-MoA exactly: aggregator \texttt{gpt-oss-120b}, deep CoT
system prompt.

\begin{table}[h]
\centering
\caption{Proposer pool for the Multi-SC-MoA and MM-MoA N-sweeps,
listed in the order they enter as $N$ increases.}
\label{tab:multi-scmoa-pool}
\small
\begin{tabular}{clr}
\toprule
Slot & Model & Size \\
\midrule
1 & DeepSeek-V3.1       & 671B MoE \\
2 & Qwen3-Coder-480B    & 480B MoE \\
3 & gpt-oss-120b        & 120B MoE \\
4 & llama-3.3-70b       & 70B dense \\
5 & LFM2-24B            & 24B \\
6 & gpt-oss-20b         & 20B MoE \\
7 & llama-3.1-8b        & 8B dense \\
8 & gemma-3n-E4B        & 4B \\
\bottomrule
\end{tabular}
\end{table}

Table~\ref{tab:multi-scmoa-sweep} reports accuracy at each $N$.
On GPQA-Diamond, Multi-SC-MoA outperforms MM-MoA at every $N$,
averaging $+10.6$\,pp.
On LCB-Hard, the benefit is $N$-dependent:
Multi-SC-MoA matches MM-MoA at small $N$ ($3$--$5$, strong models only)
but gains $+8.5$\,pp on average at large $N$ ($6$--$8$, weak models
added).

\begin{table}[h]
\centering
\caption{Multi-SC-MoA vs.\ MM-MoA N-sweep on GPQA-Diamond ($n{=}198$)
and LCB-Hard ($n{=}171$).  Both use the same eight-model pool and
aggregator; Multi-SC-MoA adds clustering, refinement, aggregation, and
override (Phases~2--5).  Single-model SC-MoA reference: 72.7\% GPQA,
59.6\% LCB-Hard.}
\label{tab:multi-scmoa-sweep}
\small
\begin{tabular}{c cc c cc c}
\toprule
& \multicolumn{3}{c}{GPQA-Diamond (\%)} & \multicolumn{3}{c}{LCB-Hard (\%)} \\
\cmidrule(lr){2-4} \cmidrule(lr){5-7}
$N$ & Multi-SC-MoA & MM-MoA & $\Delta$ & Multi-SC-MoA & MM-MoA & $\Delta$ \\
\midrule
3 & 65.0 & 53.0 & $+$12.0 & 56.4 & 57.7 & $-$1.3 \\
4 & \textbf{68.4} & 56.6 & $+$11.8 & 59.3 & \textbf{58.5} & $+$0.8 \\
5 & 65.8 & 53.8 & $+$11.9 & 53.4 & 56.7 & $-$3.4 \\
6 & 66.7 & 52.5 & $+$14.1 & 60.1 & 54.1 & $+$6.0 \\
7 & 62.2 & 52.5 & $+$9.7 & \textbf{63.9} & 56.5 & $+$7.4 \\
8 & 60.6 & 56.6 & $+$4.1 & 63.4 & 51.5 & $+$11.9 \\
\midrule
mean & 64.8 & 54.2 & $+$10.6 & 59.4 & 55.8 & $+$3.6 \\
\bottomrule
\end{tabular}
\end{table}

The $N$-dependent pattern on code isolates Phase~5 override as the
mechanism.  At small $N$ (strong proposers only), proposals converge on
similar test-pass signatures and flat aggregation captures most of the
signal; the SC-MoA pipeline's extra phases add little.  At large $N$,
weak-model proposals introduce noise that the aggregator propagates into
its synthesis.  The score-based override reverts to the best-scoring
individual proposal, catching these aggregation errors.  MM-MoA, which
lacks override, sees accuracy drop from 58.5\% ($N{=}4$) to 51.5\%
($N{=}8$) as weak models are added; Multi-SC-MoA \emph{climbs} from
59.3\% to 63.9\%.

Despite these gains, single-model SC-MoA on \texttt{gpt-oss-120b}
(72.7\% GPQA, 59.6\% LCB-Hard) outperforms the best Multi-SC-MoA
configuration at every $N$ on GPQA and matches it on LCB-Hard.
Paraphrase diversity on a capable model remains the preferred design.

Per-cell error counts range from 2--21 on GPQA and 18--31 on LCB-Hard
(weak-model timeouts and transient inference errors).
With $n{=}171$ problems and ${\sim}60\%$ accuracy, individual-cell 95\%
CIs are ${\approx}\pm 7$\,pp; the mean-level comparisons ($+10.6$\,pp
GPQA, $+3.6$\,pp LCB-Hard) are robust because they average over six
$N$ values.

\section{Comparison with GoA (Graph-of-Agents)}
\label{app:goa}

GoA~\citep{goa2026} is a graph-based multi-agent collaboration framework
that frames LLM collaboration as a graph problem.  A meta-LLM reads
model cards to select domain-specialist agents (node sampling), agents
score each other's initial responses to build a weighted graph (edge
sampling), messages pass bidirectionally along edges, and a final answer
is produced via graph pooling (max: take the highest-scored agent's
answer; mean: meta-LLM synthesizes all responses weighted by score).

We run GoA on GPQA-Diamond with a 3-model pool spanning three capability
tiers: \texttt{gpt-oss-120b} (120B MoE), \texttt{llama-3.3-70b} (70B
dense), and \texttt{llama-3.1-8b} (8B dense).  The meta-LLM for node
sampling and mean pooling is \texttt{gpt-oss-120b}.  We set $k{=}3$
(select all three models), temperature~0.7, 1~round of message passing,
and edge-pruning threshold~0.05---matching GoA's recommended defaults.

To isolate the contribution of the collaboration protocol from the model
pool, we run two additional SC-MoA controls at matched compute
(${\sim}10$ LLM calls per problem):
\begin{itemize}[nosep,leftmargin=1.5em]
  \item \textbf{SC-MoA$_{N{=}3,k{=}3}$}: the full SC-MoA pipeline
    (SPUQ perturbation $\to$ propose $\to$ cluster $\to$ refine $\to$
    aggregate $\to$ override) on a single \texttt{gpt-oss-120b},
    with $N{=}3$ paraphrases and $k{=}3$ samples per paraphrase.
  \item \textbf{SC-MoA$_{N{=}4,k{=}5}$}: the headline configuration
    reported in Table~\ref{tab:main}, included for reference.
\end{itemize}

\begin{table}[h]
\centering
\caption{GoA vs.\ SC-MoA across all benchmarks.
GoA uses a heterogeneous 3-model pool (120B, 70B, 8B); SC-MoA uses a single
\texttt{gpt-oss-120b}.  SC-MoA ($N{=}3$, $k{=}3$) is a matched-compute
variant (${\sim}10$ calls).  Best in \textbf{bold}.}
\label{tab:goa-comparison}
\small
\begin{tabular}{lcccc}
\toprule
Method & MMLU-ML (112) & GPQA (198) & LCB-Hard (171) & Calls \\
\midrule
GoA$_{\text{Max}}$ ($k{=}3$) & 87.5 & 70.2 & 57.3 & 10.6 \\
GoA$_{\text{Mean}}$ ($k{=}3$) & \textbf{92.0} & 72.7 & 59.1 & 11.9 \\
\midrule
SC-MoA ($N{=}3$, $k{=}3$) & --- & \textbf{73.7} & --- & 10.4 \\
SC-MoA ($N{=}4$, $k{=}5$) & 91.1 & 72.7 & \textbf{59.6} & ${\sim}22$ \\
\bottomrule
\end{tabular}
\end{table}

\paragraph{Analysis.}
SC-MoA's headline configuration ($N{=}4$, $k{=}5$) outperforms
GoA$_{\text{Max}}$ on GPQA ($+$2.5\,pp) and LCB-Hard ($+$2.3\,pp),
though at higher compute (${\sim}22$ vs.\ ${\sim}11$ calls).
On MMLU-ML, GoA$_{\text{Mean}}$ reaches 92.0\%---slightly above
SC-MoA's 91.1\%---likely because the heterogeneous pool gives GoA
access to models with complementary ML domain knowledge.
GoA$_{\text{Mean}}$ consistently outperforms GoA$_{\text{Max}}$ by
2--5\,pp across all benchmarks.
The result is notable because GoA has access to three distinct model
weights while SC-MoA uses a single model throughout.

\paragraph{Matched-compute comparison.}
To isolate the collaboration protocol from any compute advantage,
we compare SC-MoA$_{N{=}3,k{=}3}$ (the full SC-MoA pipeline on a
single \texttt{gpt-oss-120b} with $N{=}3$ paraphrases and $k{=}3$
samples per paraphrase) against GoA at matched compute
(${\sim}10$ LLM calls per problem).  On GPQA-Diamond,
SC-MoA$_{N{=}3,k{=}3}$ reaches 73.7\%, outperforming
GoA$_{\text{Max}}$ by $+3.5$\,pp and GoA$_{\text{Mean}}$ by
$+1.0$\,pp---even at matched compute and with a single model.

The gap reflects the perturbation-vs.-heterogeneity distinction
(\S\ref{sec:finding3}): on hard science questions, input perturbation
on a strong model extracts more signal than routing across a
heterogeneous pool.

GoA's graph-based routing does provide a clear advantage in one setting:
when the model pool is heterogeneous and includes weak models,
GoA's confidence-weighted edge sampling and max pooling effectively
ignore low-quality agents (GoA$_{\text{Max}}$: 70.2\%), while
SC-MoA's majority-based aggregation cannot filter them
(Multi-SC-MoA with the same 3-model pool: 58.1\%;
Appendix~\ref{app:multi-scmoa}).
This suggests the two approaches are complementary: GoA excels at
\emph{selecting which models to use}; SC-MoA excels at
\emph{extracting more from each model once selected}.

\paragraph{SC-MoA as a protocol-agnostic module.}
SC-MoA's post-proposal pipeline (cluster$\to$refine$\to$aggregate$\to$override) operates on any set of candidate proposals without assuming how they were generated.  The lift over flat aggregation is independent of the diversity source: $+$7.5\,pp on single-model GPQA, $+$10.6\,pp on an eight-model pool (\S\ref{app:multi-scmoa}).  Since \citet{goa2026} show that MoA is a special case of GoA (GoA Proposition~1), SC-MoA's phases can replace graph pooling in any GoA-family pipeline.  Head-to-head on GPQA-Diamond, GoA$_{\text{Max}}$ and SC-MoA solve largely non-overlapping problem subsets (union ceiling: 79.8\%), confirming their complementarity.

\paragraph{SC-GoA: full composition experiment.}
To validate this composability claim empirically, we build SC-GoA:
GoA's node sampling selects three models from our model pool,
SPUQ generates $P$ paraphrases per model, each (model, paraphrase)
slot generates $k$ self-consistency samples (best-of-$k$), and
SC-MoA's Phases~2--6 replace graph pooling for final aggregation.
We optionally include GoA's full edge sampling and message passing
between the proposal stage and SC-MoA's clustering.
Table~\ref{tab:sc-goa} reports results on all 198 GPQA-Diamond problems.

\begin{table}[h]
\centering
\caption{SC-GoA on GPQA-Diamond ($n{=}198$).  SC-GoA composes
GoA's node sampling with SC-MoA's SPUQ, self-consistency, and
Phases~2--6.  Variants include ($+$MP) or exclude ($-$MP) GoA's
edge sampling and message passing.
$P$: paraphrases per model; $k$: SC samples per slot.}
\label{tab:sc-goa}
\small
\begin{tabular}{lcccc}
\toprule
Method & Proposals & Acc.\ (\%) & Calls & $\Delta$ vs GoA$_{\text{Mean}}$ \\
\midrule
GoA$_{\text{Max}}$  & 3 & 70.2 & 10.6 & $-$2.5 \\
GoA$_{\text{Mean}}$ & 3 & 72.7 & 12.0 & --- \\
\midrule
SC-GoA$_{-\text{MP}}$ ($P{=}2$, $k{=}3$)             & 6 & \textbf{74.7} & 22.7 & $+$2.0 \\
SC-GoA$_{+\text{MP}}$ ($P{=}2$, $k{=}3$)             & 6 & 73.7 & 37.4 & $+$1.0 \\
\quad w/o Phase~3 (no refinement)                      & 6 & 70.2 & 36.5 & $-$2.5 \\
\quad majority vote only                               & 6 & 72.2 & 35.5 & $-$0.5 \\
SC-GoA-3 ($P{=}1$, $k{=}3$)                          & 3 & 71.2 & 19.2 & $-$1.5 \\
\midrule
SC-MoA ($N{=}4$, $k{=}5$)                            & 4 & 72.7 & ${\sim}22$ & $+$0.0 \\
\bottomrule
\end{tabular}
\end{table}

Three findings emerge.
\emph{First}, SC-MoA's phases compose cleanly with GoA's routing:
all SC-GoA variants that include Phases~2--6 outperform
GoA$_{\text{Max}}$ (70.2\%), and SC-GoA$_{-\text{MP}}$ (74.7\%)
exceeds standalone SC-MoA (72.7\%) by $+$2.0\,pp.
\emph{Second}, GoA's message passing is slightly net-negative
($-$1.0\,pp: 73.7\% vs.\ 74.7\%): bidirectional peer refinement
without an anchoring invariant lets weaker models (8B, 70B)
dilute the strongest agent's proposals, whereas SC-MoA's Phase~3
refines \emph{only} minorities against a frozen majority.
\emph{Third}, ablating Phase~3 produces the largest accuracy
drop ($-$3.5\,pp: 73.7\%$\to$70.2\%), confirming that
majority-anchored refinement---not just clustering or
aggregation---is the primary contributor even atop GoA's
graph pipeline.
SC-GoA$_{-\text{MP}}$ (74.7\%) exceeds standalone SC-MoA (72.7\%),
suggesting that GoA's heterogeneous routing and SC-MoA's refinement
pipeline are complementary---composing them captures both model
diversity and perturbation diversity.

\section{Statistical Significance}
\label{app:significance}

All experiments use greedy decoding with content-hash caching, so
re-running produces identical results---there is no stochastic variance
to measure.  We instead quantify \emph{problem-level} significance:
for each SC-MoA vs.\ baseline pair, we compute McNemar's test on the
$2 \times 2$ table of per-problem correctness (wins vs.\ losses) and
bootstrap 95\% confidence intervals on the accuracy difference
(10{,}000 resamples).

\begin{table}[h]
\centering
\caption{Statistical significance of SC-MoA (ungated, SPUQ paraphrases,
$N{=}5$, $k{=}2$) vs.\ each baseline.
McNemar's test on per-problem correctness; bootstrap 95\% CIs on accuracy
difference (10{,}000 resamples).  BBH-3 uses the 296-problem subset
from Table~\ref{tab:main}.  Computed by
\texttt{compute\_significance.py}.}
\label{tab:significance}
\small
\begin{tabular}{llcccccc}
\toprule
Bench & Baseline & Wins & Losses & $\chi^2$ & $p$ & $\Delta$ (\%pt) & 95\% CI \\
\midrule
  BBH-3 & SC($k{=}10$) & 27 & 9 & 9.0 & 0.003 & $+6.1$ & $[+2.0,\,+10.1]$ \\
  BBH-3 & MoA & 59 & 9 & 36.8 & $<$0.001 & $+16.9$ & $[+11.8,\,+22.0]$ \\
  BBH-3 & Self-MoA & 65 & 9 & 42.4 & $<$0.001 & $+18.9$ & $[+13.9,\,+24.3]$ \\
  BBH-3 & TextGrad & 18 & 10 & 2.3 & 0.131 & $+2.7$ & $[-0.7,\,+6.4]$ \\
\midrule
  GPQA & SC($k{=}10$) & 12 & 7 & 1.3 & 0.251 & $+2.5$ & $[-2.0,\,+7.1]$ \\
  GPQA & MoA & 19 & 8 & 4.5 & 0.034 & $+5.6$ & $[+0.5,\,+10.6]$ \\
  GPQA & Self-MoA & 26 & 8 & 9.5 & 0.002 & $+9.1$ & $[+3.5,\,+14.6]$ \\
  GPQA & TextGrad & 16 & 9 & 2.0 & 0.162 & $+3.5$ & $[-1.0,\,+8.6]$ \\
\midrule
  LCB-Hard & SC($k{=}10$) & 21 & 12 & 2.5 & 0.117 & $+5.3$ & $[-1.2,\,+11.7]$ \\
  LCB-Hard & MoA & 22 & 6 & 9.1 & 0.003 & $+9.4$ & $[+3.5,\,+15.2]$ \\
  LCB-Hard & Self-MoA & 17 & 8 & 3.2 & 0.072 & $+5.3$ & $[+0.0,\,+11.1]$ \\
  LCB-Hard & TextGrad & 24 & 6 & 10.8 & 0.001 & $+10.5$ & $[+4.7,\,+17.0]$ \\
\bottomrule
\end{tabular}
\end{table}

7 of 12 comparisons are significant at $p < 0.05$.  All three BBH-3
comparisons against MoA, Self-MoA, and SC reach significance; the
BBH-3 vs.\ TextGrad comparison ($p = 0.131$) has a positive but
wide CI.  The five non-significant comparisons concentrate on
the two smaller benchmarks (GPQA, $n{=}198$; LCB-Hard, $n{=}171$),
where statistical power is inherently limited.  All 12 bootstrap
CIs are consistent with a positive true effect (upper bound $> 0$
in every case), suggesting the gains are real but some individual
comparisons are underpowered.

Clopper--Pearson 95\% CIs on SC-MoA accuracy:
BBH-3 $[82.1\%, 90.2\%]$,
GPQA $[66.5\%, 79.3\%]$,
LCB-Hard $[54.9\%, 69.8\%]$.

\section{Gating Threshold Sensitivity}
\label{app:theta-sweep}

We sweep the consensus-gating threshold
$\theta \in \{0.4, 0.5, 0.6, 0.7, 0.8, 0.9, 1.0, \infty\}$ by
simulating each threshold on existing results: for each problem, if
the pre-refinement consensus $C \geq \theta$, we use the gated
(SC-exit) result; otherwise we use the ungated (always-aggregate)
result.  Table~\ref{tab:theta-sweep} reports the results.

\begin{table}[h]
\centering
\caption{Accuracy (\%) under different consensus-gating thresholds
($N{=}4$, $k{=}5$, \texttt{gpt-oss-120b}).
$\theta{=}\infty$ is always-aggregate (no gating).  Always-aggregate
dominates or matches gating on all benchmarks except MMLU-ML, where
$\theta{=}0.6$ slightly exceeds always-aggregate.}
\label{tab:theta-sweep}
\small
\begin{tabular}{lcccccccc}
\toprule
Bench & $\theta{=}0.4$ & $\theta{=}0.5$ & $\theta{=}0.6$ & $\theta{=}0.7$ & $\theta{=}0.8$ & $\theta{=}0.9$ & $\theta{=}1.0$ & $\theta{=}\infty$ \\
\midrule
BBH-3 (296) & 79.1 & 79.1 & 81.4 & 81.4 & 85.5 & 85.5 & 85.5 & \textbf{85.8} \\
MMLU-ML (112) & 91.1 & 91.1 & \textbf{92.0} & \textbf{92.0} & 91.1 & 91.1 & 91.1 & 91.1 \\
GPQA (198) & 71.7 & 71.7 & \textbf{72.7} & \textbf{72.7} & \textbf{72.7} & \textbf{72.7} & \textbf{72.7} & \textbf{72.7} \\
AIME (90) & 61.1 & 61.1 & 63.3 & 63.3 & 65.6 & 65.6 & 65.6 & \textbf{85.6} \\
LCB-Hard (171) & 55.6 & 55.6 & 55.6 & 55.6 & 56.7 & 56.7 & 56.7 & \textbf{59.6} \\
\bottomrule
\end{tabular}
\end{table}

On all benchmarks except MMLU-ML, always-aggregate ($\theta{=}\infty$)
dominates or matches the gated variant.  On MMLU-ML, $\theta{=}0.6$
achieves 92.0\% vs.\ 91.1\% at $\theta{=}\infty$, suggesting the
aggregator occasionally introduces errors on easy problems where
the SC majority is already correct.
The cost of gating is largest on AIME ($-24.5$~pp at $\theta{=}0.4$),
where frequent near-tie consensus scores cause premature SC-exit,
and on BBH-3 ($-6.7$~pp).  On LCB-Hard the cost is $-4.0$~pp,
driven by unfaithful test-pass clustering.
The monotonic improvement from $\theta{=}0.4$ to $\theta{=}\infty$
on most benchmarks supports the universal aggregation design.

\subsection{Inner Sampling Budget ($k$-Sweep)}
\label{app:k-sweep}

We sweep the per-paraphrase sample budget
$k \in \{3, 4, 5, 6, 7\}$ with a fixed pool of $N{=}4$ SPUQ
paraphrases, drawn from a frozen cache to ensure strict
apples-to-apples comparison across cells.  The proposer is
\texttt{gpt-oss-120b}, the aggregator is
\texttt{llama-3.1-8b-instant}, sampling temperature is 0.7, and all
other SC-MoA stages (cluster $\to$ refine $\to$ aggregate $\to$
override) are held constant.
\textbf{Caveat:} this sweep uses a weaker aggregator
(\texttt{llama-3.1-8b-instant}) than the main results
(\texttt{gpt-oss-120b}).  The saturation-at-$k{=}5$ finding
should therefore be interpreted as a lower bound; the main
aggregator may benefit from additional samples.

\begin{figure}[h]
\centering
\includegraphics[width=0.85\linewidth]{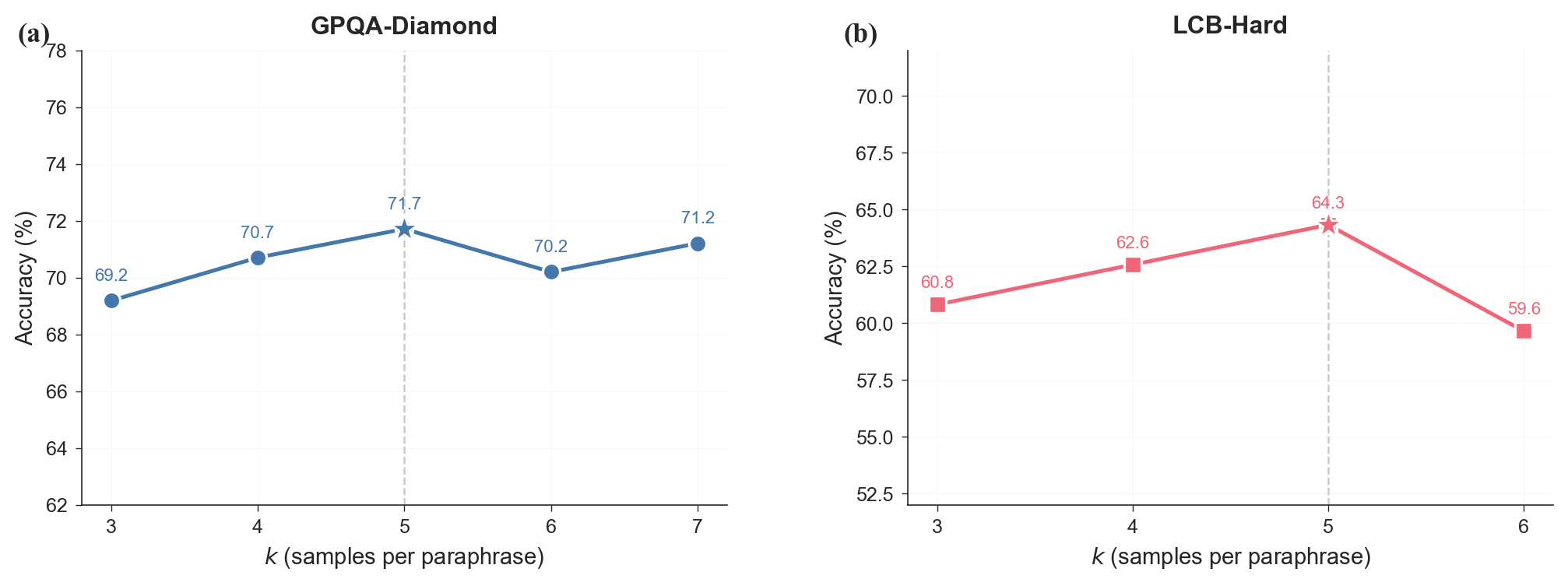}
\vspace{-2mm}
\caption{\textbf{Inner sampling budget sweep.}
Accuracy vs.\ $k$ (samples per paraphrase) on GPQA-Diamond
($n{=}198$) and LCB-Hard ($n{=}171$).  Both benchmarks peak at
$k{=}5$; beyond this point, post-peak variation is within noise on
GPQA and significantly regresses on LCB-Hard ($\Delta{=}-4.7$\,pp,
$p{=}0.008$).  Error bars: bootstrap 95\% CIs.}
\label{fig:ksweep}
\end{figure}

Across 198 GPQA-Diamond questions (Figure~\ref{fig:ksweep}),
accuracy climbs monotonically from 69.2\% ($k{=}3$) to a peak of
71.7\% at $k{=}5$, then oscillates within noise (70.2\% at $k{=}6$,
71.2\% at $k{=}7$); a paired bootstrap on $k{=}5$ vs.\ $k{=}6$
yields $\Delta = -1.5$\,pp ($p = 0.48$), so the post-peak variation
is not statistically distinguishable.  On 171 LCB-Hard problems,
accuracy also peaks at $k{=}5$ (64.3\%) after a $+3.5$\,pp climb
from $k{=}3$, but regresses significantly at $k{=}6$ to 59.6\%
(paired $\Delta = -4.68$\,pp, 95\% CI $[-8.19, -1.75]$,
$p = 0.008$).  Both benchmarks therefore exhibit the same empirical
optimum, while compute cost grows linearly in $k$ (each step adds
${\approx}25\%$ more LLM calls and, on LCB-Hard,
${\approx}25\%$ more sandbox executions).

We adopt $k{=}5$ as the operating point for the $k$-sweep analysis
on the basis that it maximizes accuracy on both benchmarks at the
lowest budget where these gains are realized.  The main-table results
(\S\ref{sec:results}) use $k{=}2$ with the stronger
\texttt{gpt-oss-120b} aggregator, where saturation occurs earlier
(\S\ref{sec:mechanism}).

\section{Public vs.\ Private Test Fidelity}
\label{app:fidelity-gap}

On LCB-Hard, clustering uses test-pass signatures from public tests,
but scoring uses private (hidden) tests.  This creates a fidelity gap:
proposals that pass all public tests may fail private tests.  We
quantify this gap for all methods.

\begin{table}[h]
\centering
\caption{Public-to-private test fidelity on LCB-Hard ($n{=}171$).
$P(\text{priv} \mid \text{pub})$ is the probability that a solution
passing all public tests also passes all private tests.}
\label{tab:fidelity-gap}
\small
\begin{tabular}{lccc}
\toprule
Method & pass\_public=1.0 & $P(\text{priv} \mid \text{pub})$ & False pos.\ rate \\
\midrule
SC ($k{=}10$)     & 154/171 (90.1\%) & 53.9\% & 46.1\% \\
MoA               & 127/171 (74.3\%) & 62.2\% & 37.8\% \\
Self-MoA          & 133/171 (77.8\%) & 63.9\% & 36.1\% \\
TextGrad          & 128/171 (74.9\%) & 62.5\% & 37.5\% \\
\midrule
SC-MoA (ungated)  & 146/171 (85.4\%) & \textbf{74.0\%} & \textbf{26.0\%} \\
\bottomrule
\end{tabular}
\end{table}

SC-MoA achieves the highest public-to-private fidelity (74.0\%)
among all methods---solutions that pass public tests are
substantially more likely to also pass private tests.  SC has the
worst fidelity (53.9\%) despite the highest public pass rate
(90.1\%), indicating that i.i.d.\ sampling with a single prompt
overfits to public test structure.  SC-MoA's perturbation diversity
produces solutions that are more robust to distribution shift between
public and private tests.

At $C{=}1.0$ (unanimous consensus), 119 problems reach this level
and 88/119 (73.9\%) are correct.  The 26.1\% false-positive rate at
unanimous consensus directly quantifies why gating is harmful on
code: even perfect observed agreement does not guarantee correctness
when the clustering metric (public tests) is only a proxy for the
true metric (private tests).

\section{Implementation Details}
\label{app:implementation}

\subsection{QA personas (5 strategies)}
\label{app:qa-personas}

\paragraph{Analytical.}
``You are a careful analytical thinker. Break the problem into
components and reason through each step explicitly. Show your work.
Check each intermediate conclusion before proceeding.''

\paragraph{Intuitive.}
``You are an intuitive problem solver. Start with your best guess,
then verify it against the constraints. If verification fails, try the
next most likely answer.''

\paragraph{Adversarial.}
``You are a critical thinker. First identify what answer most people
would give, then check if that obvious answer is correct or a common
misconception. Look for traps and edge cases.''

\paragraph{Chain-tracer.}
``You are a meticulous state-tracker. Maintain an explicit ledger of
all variables and their values. Update the ledger at each reasoning
step. Your answer must be consistent with the final state of the
ledger.''

\paragraph{Counterfactual.}
``You are a hypothesis tester. For each candidate answer, ask: if this
were correct, what else would have to be true? Check each implication.
Eliminate answers whose implications contradict the problem setup.''

\subsection{Code personas (5 strategies)}
\label{app:code-personas}

\paragraph{Analytical.}
``Analyze the problem's constraints and complexity requirements first.
Select the tightest algorithm that meets the constraints. Implement
cleanly with explicit edge-case handling.''

\paragraph{Intuitive.}
``Write the shortest pragmatic solution using standard library helpers.
Favor readability and correctness over cleverness.''

\paragraph{Adversarial.}
``Before coding, enumerate edge cases: empty input, single element,
maximum constraints, duplicate values, negative numbers. Design the
solution to handle all of them from the start.''

\paragraph{Complexity-first.}
``State the required time and space complexity explicitly. Identify the
algorithmic paradigm (DP, binary search, graph, etc.) before writing
any code. Optimize for competitive-programming efficiency.''

\paragraph{Proof-sketch.}
``Before implementation, write a loop invariant or recurrence relation
that proves the solution is correct. Only then translate the proof
into code.''

\subsection{Refinement prompt (verbatim)}
\label{app:refine-prompt}

Only minority proposals are refined.  The refinement prompt has a
\emph{system prompt} (persona-specific) and a \emph{user prompt}
(structured XML blocks).

\paragraph{Code refinement system prompt.}
\begin{quote}\small\ttfamily
You are a \{persona\_name\} reviewing your solution against an
ensemble consensus. \{k\_max\}/\{n\} specialists produced a DIFFERENT
approach that passes more tests than yours.

You may:
\quad (a) ADOPT the majority approach if their reasoning is stronger.
\quad (b) IMPROVE your approach by fixing the specific issue the
majority handles better, while preserving your unique strengths.
\quad (c) DEFEND your approach if you believe the majority has a subtle
flaw (e.g.\ worse complexity, wrong invariant) --- return your original
with a brief note on why.

Return a single \textasciigrave\textasciigrave\textasciigrave python\textasciigrave\textasciigrave\textasciigrave\ fenced block.
\end{quote}

\paragraph{QA refinement system prompt.}
\begin{quote}\small\ttfamily
You are a \{persona\_name\} reviewing your answer against an
ensemble consensus. \{k\_max\}/\{n\} specialists answered differently
from you.

You may:
\quad (a) ADOPT the majority answer if their reasoning is stronger.
\quad (b) IMPROVE your reasoning to arrive at a better answer.
\quad (c) DEFEND your answer if you believe the majority has a flaw ---
explain the specific error in their reasoning.

The last line of your response MUST be: `Answer: \$VALUE' where
\$VALUE follows the answer format the question requests.
\end{quote}

\paragraph{User prompt template (code variant).}
The user prompt contains four XML blocks:
\texttt{<PROBLEM>} (truncated to 2500 chars),
\texttt{<YOUR\_PROPOSAL quality=\ldots, intra\_agreement=\ldots>},
\texttt{<FAILURE\_ANALYSIS>} (which public tests differ and stderr snippets), and
\texttt{<MAJORITY\_APPROACH quality=\ldots, support=k/n>}.
The final instruction is: ``Produce your refined solution.''

\subsection{Aggregation prompt (verbatim)}
\label{app:agg-prompt}

\paragraph{Code aggregation system prompt.}
\begin{quote}\small\ttfamily
You are a senior problem-solver synthesising a final Python solution
from \{n\} specialist proposals. Each proposal was the best of \{k\}
self-consistency samples from its persona.

You are given the full consensus evolution across a refinement step. A
persona that DEFENDED its minority position despite seeing the majority
may have found a genuine edge case --- examine its reasoning carefully.

To produce a solution DIFFERENT from the post-refinement majority, you
MUST cite specific evidence from a defended minority proposal (e.g.\ a
concrete test case the majority fails). If you cannot identify such
evidence, return the majority solution verbatim.

Return a single \textasciigrave\textasciigrave\textasciigrave python\textasciigrave\textasciigrave\textasciigrave\ fenced block.
\end{quote}

\paragraph{QA aggregation system prompt.}
\begin{quote}\small\ttfamily
You are a senior problem-solver synthesising a final answer from \{n\}
specialist proposals. Each proposal was the best of \{k\}
self-consistency samples from its persona.

You are given the full consensus evolution across a refinement step. A
persona that DEFENDED its minority position despite seeing the majority
may have found a genuine flaw --- examine its reasoning carefully.

To produce an answer DIFFERENT from the post-refinement majority, you
MUST cite specific evidence from a defended minority. If you cannot,
return the majority answer.

Commit: do NOT hedge. The last line of your response MUST be:
`Answer: \$VALUE' where \$VALUE follows the answer format the question
requests.
\end{quote}

\paragraph{User prompt template.}
The user prompt contains three XML blocks:
\texttt{<PROBLEM>} (the original problem),
\texttt{<PROPOSALS>} (each persona's proposal with cluster label,
quality score, intra-agreement, and refinement action annotation), and
\texttt{<CONSENSUS\_EVOLUTION>} (pre/post-refinement agreement counts,
which personas switched or defended).
Each proposal is annotated with one of:
\texttt{[UNCHANGED --- majority]},
\texttt{[REFINED --- ADOPTED]},
\texttt{[REFINED --- IMPROVED]}, or
\texttt{[REFINED --- DEFENDED]}.

\subsection{Variant-tag protocol}
\label{app:variant-tag}

For $k > 1$ samples per perturbation (or per prompt in baselines), we append
a neutral HTML comment tag to the user prompt:
\begin{quote}\small\ttfamily
<!-- sample variant: scmoa-analytical-0 -->
\end{quote}
The tag changes the content hash
$\texttt{sha256}(\texttt{model} \| \texttt{system\_prompt} \|
\texttt{user\_prompt})$, producing a distinct LLM call.  At greedy
decoding (temperature 0), the tag acts as a mild prompt perturbation:
it shifts the model's token-level predictions without changing problem
semantics.  HTML comment syntax was chosen because it has minimal
semantic impact on the model's interpretation of the problem.

All methods use the same protocol with method-specific prefixes:
SC uses \texttt{sc-\{k\}}, Self-MoA uses \texttt{selfmoa-\{k\}},
MoA uses \texttt{moa-\{k\}}, and SC-MoA uses
\texttt{scmoa-\{persona\}-\{k\}}.  This ensures a fair comparison:
all methods obtain sample diversity through the same mechanism.

\subsection{Scoring and clustering}
\label{app:scoring}

\paragraph{QA scoring.}
For QA benchmarks, the quality score is the \emph{agreement fraction}:
the proportion of $N$ perturbations whose extracted answer matches this
proposal's extracted answer.  Formally, for perturbation~$i$ with extracted
answer $a_i$: $\text{score}(i) = |\{j : a_j = a_i\}| / N$.
Clustering uses normalized string matching on the extracted answer
(lowercased, stripped).

\paragraph{Code scoring.}
For code benchmarks, the quality score is the \emph{public test pass
rate}: the fraction of public test cases passed.  Clustering groups
proposals by their \emph{test-pass signature}---the set of public test
names that the solution passes.  Two proposals that pass different
subsets of tests are placed in different clusters even if they share
the same pass rate.  This is more informative than pass-rate-only
clustering but creates the fidelity gap analyzed in
\S\ref{app:fidelity-gap}: the public test signature is an imperfect
proxy for private-test correctness.

\paragraph{Majority selection.}
When multiple clusters tie in size, ties are broken by mean quality
score (higher wins), then by earliest first occurrence.

\subsection{Prompt allocation table}
\label{app:prompt-allocation}

Table~\ref{tab:prompt-allocation} summarizes the system prompt
allocation for each method, clarifying the diversity mechanism.

\begin{table}[h]
\centering
\caption{System prompt allocation per method.}
\label{tab:prompt-allocation}
\small
\begin{tabular}{lcc}
\toprule
Method & System prompts & Rationale \\
\midrule
Zero-shot CoT & 1 generic CoT & Standard baseline \\
SC ($k{=}10$) & 1 generic CoT & Faithful to \citet{wang2023selfconsistency} \\
Self-MoA & 1 generic CoT & Faithful to \citet{wang2025moa} \\
MoA & 5 SPUQ perturbations & Same as SC-MoA (isolates architecture) \\
TextGrad & 1 generic CoT & Official v0.1.8 prompt \\
\midrule
SC-MoA & 5 SPUQ perturbations & Input-side diversity via paraphrasing \\
\bottomrule
\end{tabular}
\end{table}

\section{Persona Diversity Analysis}
\label{app:persona-diversity}

This appendix collects the persona-related experiments referenced in
the main text.

\subsection{Error correlation ($\rhobar$) measurement}

We measure mean pairwise error correlation between the $N{=}5$ agents
under two conditions: (a)~persona-diverse (5 different system prompts)
and (b)~i.i.d.\ (1 generic prompt sampled 5 times).

\begin{table}[h]
\centering
\caption{Mean pairwise error correlation $\rhobar$ with 95\% bootstrap CIs.}
\small
\begin{tabular}{lccc}
\toprule
Benchmark & Persona $\rhobar$ [95\% CI] & i.i.d.\ $\rhobar$ [95\% CI] & $\Delta$ \\
\midrule
GPQA-Diamond & 0.633 [.557, .705] & 0.603 [.522, .679] & $+$0.030 \\
LCB-Hard     & 0.571 [.491, .648] & 0.607 [.536, .676] & $-$0.036 \\
\bottomrule
\end{tabular}
\end{table}

\subsection{Perturbation content ablation (extended)}

\begin{table}[h]
\centering
\caption{Full perturbation content ablation.  Three conditions
$\times$ two benchmarks.  McNemar $p > 0.25$ for all pairwise
comparisons.}
\small
\begin{tabular}{lcc}
\toprule
Condition & GPQA (\%) & LCB-Hard (\%) \\
\midrule
Hand-crafted personas & 72.7 & 63.7 \\
Paraphrased CoT (SPUQ) & 72.7 & 65.9 \\
GPT-gen alt strategies & 70.2 & 60.8 \\
\bottomrule
\end{tabular}
\end{table}

\noindent On LCB-Hard, SPUQ paraphrases
outperform hand-crafted personas by $+2.2$\,pp.

\subsection{Persona-based vs.\ paraphrase-based MoA}
\label{app:persona-vs-paraphrase}

\begin{table}[h]
\centering
\caption{MoA baseline variants.  Persona-MoA uses $N{=}5$ hand-crafted
expert personas (11 calls); Paraphrase-MoA uses $N{=}4$ SPUQ
paraphrases (9 calls).  Main text Table~\ref{tab:main} reports the
paraphrase variant for apples-to-apples comparison with SC-MoA.}
\small
\begin{tabular}{lccccc}
\toprule
Variant & Diversity source & $N$ & Calls & GPQA (\%) & LCB (\%) \\
\midrule
MoA (persona) & 5 expert personas & 5 & 11 & 72.7 & 47.4 \\
MoA (paraphrase) & SPUQ paraphrases & 4 & 9 & 65.2 & 50.9 \\
\bottomrule
\end{tabular}
\end{table}

\noindent Persona-MoA scores higher on GPQA ($+7.5$\,pp) but lower on
LCB-Hard ($-3.5$\,pp).  The persona advantage on GPQA likely reflects
hand-crafted domain prompts (physicist, biologist, etc.) that steer
proposals toward discipline-specific reasoning chains.  On code,
paraphrase diversity is more effective.  We report the paraphrase
variant in Table~\ref{tab:main} because it uses the same SPUQ
diversity mechanism as SC-MoA, isolating the aggregation architecture
as the sole variable.

\section{Semantic-Preserving Paraphrase Pipeline (SPUQ)}
\label{app:spuq}

SPUQ (Semantic-Preserving Uncertainty Quantification) generates $N$
rephrasings of the input problem.  The paraphrase generator receives a
6-rule system prompt:

\begin{enumerate}[leftmargin=*,itemsep=1pt]
\item \textbf{Semantic equivalence}: each rephrasing must ask the same
    question and require the same answer.
\item \textbf{Surface variation}: vary word choice, sentence structure,
    and notation style.
\item \textbf{Preserve numbers verbatim}: all numerical values must
    appear exactly as in the original.
\item \textbf{Preserve code verbatim}: all code blocks, variable names,
    and function signatures must be identical.
\item \textbf{Preserve units}: physical units (eV, m/s, kg, etc.)
    must appear exactly.
\item \textbf{Produce exactly $N$ rephrasings}: output as a JSON list.
\end{enumerate}

\paragraph{Validation.}
A regex-based check verifies that all \emph{protected tokens}---numbers,
code blocks, and physical units---appear verbatim in each rephrasing.
Any rephrasing that fails is regenerated.  On GPQA-Diamond ($N{=}5$,
300 total paraphrases), the rejection rate is \textbf{0\%}: the
generator preserves all protected tokens on every science question.
On LCB-Hard (115 paraphrases), the rejection rate is 16.5\%, concentrated
on problems with complex spatial constraints (e.g., chess board
problems); rejected paraphrases are regenerated until validation passes.

\paragraph{Paraphrase-quality bug.}
Early all-8b runs used the proposer model (\texttt{llama-3.1-8b}) as the
paraphrase generator.  This produced malformed JSON that the parser
dropped, effectively reducing $N$.  Switching to a capable generator
(\texttt{gpt-oss-120b}) improved accuracy by $+6$--$8$\,pp on GPQA at
the all-8b configuration (22.7\% $\to$ 30.8\% at $k{=}1$).  A
\texttt{-{}-paraphrase-model} flag was added to decouple the generator
role.

\paragraph{Cost.}
A single LLM call generates all $N$ rephrasings (batch generation).
For $N{=}5$, this adds 1 call per problem.  Paraphrases are cached and
reused across $k$ values and ablation conditions.

\section{Proposer $\times$ Aggregator Capability Matrix}
\label{app:capability-matrix}

\begin{table}[h]
\centering
\caption{Full capability matrix: SPUQ paraphrases at $N{=}5$.}
\small
\begin{tabular}{lcccc}
\toprule
Config & GPQA $k{=}1$ & GPQA $k{=}3$ & LCB $k{=}1$ & LCB $k{=}3$ \\
\midrule
prop=120b agg=120b & 73.7 & 73.2 & 60.8 & 60.2 \\
prop=120b agg=8b   & 71.2 & 70.2 & 46.8 & 57.3 \\
prop=8b   agg=8b   & 30.8 & 31.3 & ${\sim}$0 & ${\sim}$0 \\
\bottomrule
\end{tabular}
\end{table}

\noindent Three observations: (1)~Dropping the aggregator from 120b to 8b
loses ${\sim}2.5$\,pp on QA (small) but 14\,pp on code at $k{=}1$
(large).  (2)~Dropping the proposer is catastrophic: GPQA crashes
$73\% \to 31\%$.  (3)~On code, $k{=}3$ at agg=8b recovers to 57.3\%
($-3.5$\,pp vs.\ agg=120b $k{=}1$), but 80\% of this recovery is the
override mechanism (\S\ref{sec:aggregation-safety}).

\section{Override and Best-of-N Analysis}
\label{app:override}
\label{app:best-of-n}

\begin{table}[h]
\centering
\caption{Override fire rates and accuracy decomposition by aggregator strength.}
\small
\begin{tabular}{lcccc}
\toprule
Config & Fire rate & Acc (overridden) & Acc (not overridden) & Overall \\
\midrule
LCB agg=120b $k{=}1$ & 9.9\% & 58.8\% & 61.0\% & 60.8\% \\
LCB agg=120b $k{=}3$ & 11.1\% & 57.9\% & 60.5\% & 60.2\% \\
LCB agg=8b $k{=}1$   & 58.5\% & 56.0\% & 33.8\% & 46.8\% \\
LCB agg=8b $k{=}3$   & 67.3\% & 65.2\% & 41.1\% & 57.3\% \\
GPQA (all configs)    & 0.0\%  & ---    & ---    & --- \\
\bottomrule
\end{tabular}
\end{table}

\noindent At strong aggregators, override rarely fires (${\sim}10\%$) and
accuracy is similar whether or not override is triggered.  At weak
aggregators on code, override fires on the majority of problems
($58$--$67\%$) and is responsible for $+33$--$42$\,pp of accuracy.
On QA, override never fires because the quality score is always 1.0.

\paragraph{Decomposition of the $+10.5$\,pp $k{=}1 \to k{=}3$ gain at
agg=8b:}
\begin{itemize}[itemsep=1pt]
\item Genuine $k \times \text{aggregator}$ interaction: $+2.3$\,pp
\item Override amplification (cleaner $k{=}3$ clusters): $+8.2$\,pp
\item Total: $+10.5$\,pp
\end{itemize}

\paragraph{Best-of-N ablation (strong aggregator).}
On LCB-Hard, SC-MoA's accuracy includes contributions from both
aggregation synthesis (Phase~4) and score-based override (Phase~5).
Table~\ref{tab:best-of-n} isolates each component.

\begin{table}[h]
\centering
\caption{Decomposing SC-MoA accuracy on LCB-Hard ($n{=}171$,
\texttt{gpt-oss-120b}).}
\label{tab:best-of-n}
\small
\begin{tabular}{lcc}
\toprule
Configuration & Accuracy & $\Delta$ from full \\
\midrule
Oracle (any proposal correct) & 78.4\% & --- \\
SC-MoA no-override (agg only) & 65.5\% & $+1.8$\,pp \\
SC-MoA full (with override) & 63.7\% & --- \\
Best-of-N (majority of per-para bests) & 60.8\% & $-2.9$\,pp \\
\bottomrule
\end{tabular}
\end{table}

\noindent At strong aggregators, the override fires on 11.7\% of
problems (20/171) with 60\% accuracy when it fires.  Aggregation alone
(65.5\%) slightly \emph{exceeds} the full pipeline (63.7\%).  The
override provides insurance against catastrophic aggregator failure
(the weak-aggregator regime above, where it contributes $+33$--$42$\,pp)
rather than a systematic accuracy boost at capable aggregators.

\section{Oracle Ceiling Analysis}
\label{app:oracle}

\begin{table}[h]
\centering
\caption{Oracle ceiling: accuracy if the aggregator always selected the
best proposal.  Gap = oracle $-$ actual.}
\small
\begin{tabular}{lccc}
\toprule
Config & Actual (\%) & Oracle (\%) & Gap (pp) \\
\midrule
\multicolumn{4}{l}{\emph{GPQA-Diamond}} \\
120b/120b $k{=}1$ & 73.7 & 85.4 & $+$11.6 \\
120b/120b $k{=}3$ & 73.2 & 92.4 & $+$19.2 \\
120b/8b $k{=}1$   & 71.2 & 85.4 & $+$14.1 \\
120b/8b $k{=}3$   & 70.2 & 90.4 & $+$20.2 \\
8b/8b $k{=}1$     & 30.8 & 69.7 & $+$38.9 \\
8b/8b $k{=}3$     & 31.3 & 93.9 & $+$62.6 \\
\midrule
\multicolumn{4}{l}{\emph{LCB-Hard}} \\
120b/120b $k{=}1$ & 60.8 & 67.8 & $+$7.0 \\
120b/120b $k{=}3$ & 60.2 & 76.6 & $+$16.4 \\
120b/8b $k{=}1$   & 46.8 & 58.5 & $+$11.7 \\
120b/8b $k{=}3$   & 57.3 & 76.0 & $+$18.7 \\
\bottomrule
\end{tabular}
\end{table}

\noindent The aggregator leaves $7$--$63$\,pp of correct answers on the
table.  The gap grows with $k$ (more candidates, same selection
quality) and is largest at weak proposers on QA ($+62.6$\,pp at all-8b
$k{=}3$): the proposer pool contains the correct answer 94\% of the
time, but the system delivers only 31\%.

\section{Cluster Structure at Weak Proposers}
\label{app:clusters}

At weak proposers (\texttt{llama-3.1-8b}), the 5 SPUQ perturbations
fragment into ${\sim}3$ distinct answer clusters (mean majority cluster
size 2.8/5, vs.\ 4.1/5 at strong proposers).

\begin{table}[h]
\centering
\caption{Pre-refinement cluster shapes at all-8b GPQA $k{=}1$.}
\small
\begin{tabular}{lcc}
\toprule
Shape & Count & Meaning \\
\midrule
5 (unanimous) & 14 (7.1\%) & all 5 agree \\
4+1           & ${\sim}$20 & 4 vs.\ 1 outlier \\
3+2           & ${\sim}$16 & split majority \\
3+1+1         & ${\sim}$26 & majority $+$ 2 outliers \\
2+2+1         & ${\sim}$14 & tied $+$ outlier \\
$\leq$2+1+1+1 & ${\sim}$9 & weak/no majority \\
\bottomrule
\end{tabular}
\end{table}

\paragraph{Unanimity reliability.}
At strong proposers, unanimous agreement ($C{=}1.0$) is a strong
signal: ${\sim}10\%$ wrong rate.  At weak proposers, unanimity is a
\emph{negative} signal: ${\sim}43\%$ wrong rate.  Unanimity at weak
proposers reflects shared training-data biases, not convergence to truth.

\begin{table}[h]
\centering
\caption{Unanimity wrong rate across configurations.}
\small
\begin{tabular}{lcc}
\toprule
Config & Unanimous problems & Wrong rate \\
\midrule
GPQA prop=120b $k{=}1$ & ${\sim}50\%$ & ${\sim}10\%$ \\
GPQA all-8b $k{=}1$    & 7.1\% (14/198) & ${\sim}43\%$ \\
LCB prop=120b $k{=}1$  & 46.2\% & 31.6\% \\
LCB agg=8b $k{=}1$     & 47.4\% & 44.4\% \\
\bottomrule
\end{tabular}
\end{table}

\section{Sampling Protocol Comparison}
\label{app:sampling-protocol}

Standard self-consistency~\citep{wang2023selfconsistency} draws
diverse samples via temperature sampling ($T > 0$).  Our baselines
instead use greedy decoding ($T{=}0$) with neutral variant tags
(\S\ref{app:variant-tag}) to create sample diversity.  To confirm
that this protocol does not disadvantage the SC baseline, we run
SC($k{=}10$) at four temperature settings and compare against our
greedy+tag variant.

\begin{table}[h]
\centering
\caption{Sampling protocol comparison: SC($k{=}10$) accuracy (\%)
under different diversity mechanisms.  On QA, greedy+tag and
temperature sampling are statistically indistinguishable (McNemar
$p{=}0.332$ for best temperature vs.\ greedy+tag).  On code,
temperature sampling yields meaningfully higher SC accuracy than
greedy+tag (67.3\% vs.\ 57.3\% at $T{=}0.7$), reflecting the value
of stochastic exploration for code generation.}
\label{tab:sampling-protocol}
\small
\begin{tabular}{lcccc}
\toprule
Condition & Temp & GPQA (\%) & LCB-Hard (\%) \\
\midrule
SC greedy+tag (ours) & 0 & 70.7 & 57.3 \\
SC $T{=}0.3$ & 0.3 & 71.7 & 65.5 \\
SC $T{=}0.5$ & 0.5 & 72.2 & 62.0 \\
SC $T{=}0.7$ & 0.7 & 71.2 & 67.3 \\
SC $T{=}1.0$ & 1.0 & 70.7 & 59.1 \\
\midrule
SC-MoA (ungated) & 0 & \textbf{72.7} & 59.6 \\
\bottomrule
\end{tabular}
\end{table}

\noindent On QA tasks, greedy decoding with a neutral variant tag
produces diversity comparable to temperature sampling: all five
conditions fall within a 2.5\,pp band (70.7--72.2\%), and no
pairwise McNemar test reaches significance.  SC-MoA (72.7\%)
slightly exceeds the strongest temperature SC (72.2\% at $T{=}0.5$)
by 0.5\,pp.

On code tasks, temperature sampling provides substantially better
diversity than greedy+tag: accuracy peaks at $T{=}0.7$ (67.3\%)
and $T{=}0.3$ (65.5\%), compared with 57.3\% for greedy+tag.
Higher temperature ($T{=}1.0$, 59.1\%) degrades accuracy, consistent
with excessive token noise corrupting code syntax.
Token-level stochasticity at moderate temperatures explores
qualitatively different implementation strategies, while variant tags
produce only minor prompt perturbations that often yield similar code.
SC-MoA (persona-greedy) remains competitive with the best temperature
SC (59.6\% vs.\ 67.3\%, McNemar $p{=}0.711$).

\subsection{Input-Side vs.\ Output-Side Diversity}
\label{app:diversity-source}

A natural question is whether input-side diversity (SPUQ
perturbations) fundamentally outperforms output-side diversity
(temperature sampling).  To isolate the diversity source, we hold the
total call budget fixed at $N{=}5$ proposals and compare three
conditions, all using the same aggregation pipeline (Phases~2--5):

\begin{table}[h]
\centering
\caption{Diversity source comparison, holding total proposals fixed
at $N{=}5$.  All conditions use the same aggregation pipeline
(Phases~2--5).  On GPQA, greedy decoding suffices: persona-greedy and
SPUQ-greedy are statistically indistinguishable ($p{=}0.480$), and
temperature hurts.  On LCB-Hard, input-side and output-side diversity
are complementary: SPUQ$+$$T{=}0.7$ yields the highest accuracy
(68.4\%), $+4.7$\,pp over persona-greedy.}
\label{tab:diversity-source}
\small
\begin{tabular}{lccc}
\toprule
Condition & Temp & GPQA (\%) & LCB-Hard (\%) \\
\midrule
SC-MoA persona-greedy & 0 & \textbf{75.3} & 63.7 \\
SC-MoA SPUQ-greedy & 0 & 73.2 & 62.6 \\
SC-MoA SPUQ $+$ $T{=}0.7$ & 0.7 & 72.2 & \textbf{68.4} \\
SC-MoA temp-diversity $T{=}0.7$ & 0.7 & 69.7 & 67.3 \\
SC-MoA temp-diversity $T{=}1.0$ & 1.0 & 74.7 & 62.0 \\
\bottomrule
\end{tabular}
\end{table}

\noindent \textbf{QA.}  SPUQ-greedy (73.2\%) and persona-greedy (75.3\%)
are statistically indistinguishable (McNemar $p{=}0.480$, $n{=}198$):
of the 18 problems where they disagree, personas win 11 and SPUQ wins 7.
Adding temperature to SPUQ does not help on QA: SPUQ$+$$T{=}0.7$
(72.2\%) is 1.0\,pp below SPUQ-greedy (73.2\%) and 3.1\,pp below
persona-greedy (75.3\%).  Temperature-only diversity ($T{=}0.7$,
69.7\%) fares worse still, while $T{=}1.0$ (74.7\%) partially
recovers.  On QA, token-level noise degrades the aggregator's inputs
without contributing useful diversity---greedy decoding suffices.

\textbf{Code.}  On LCB-Hard, temperature-only diversity
($T{=}0.7$, single prompt) yields 67.3\%---identical to SC($T{=}0.7$),
indicating that the aggregation pipeline adds no value when all
diversity comes from token-level stochasticity.  Higher temperature
($T{=}1.0$) degrades further to 62.0\%, below even persona-greedy
(63.7\%).  Stacking SPUQ perturbations with moderate temperature
($T{=}0.7$) recovers a clear gain: 68.4\%, $+4.7$\,pp over
persona-greedy and $+1.1$\,pp over SC($T{=}0.7$).  Input-side
perturbation and output-side stochasticity are \emph{complementary}
on code: SPUQ varies the problem framing while temperature varies the
implementation strategy, together exploring a richer proposal space
than either mechanism alone.

\section{Token-Level Compute Accounting}
\label{app:token-budget}

To verify that SC-MoA's accuracy gains are not explained by
disproportionate token expenditure, we report total tokens
(input~$+$~output) per method per problem.

\begin{table}[h]
\centering
\caption{Mean tokens per problem and accuracy-per-million-tokens
(Acc/MT) on GPQA-Diamond and BBH.  SC-MoA (ungated) achieves the highest
accuracy at a fraction of the token cost of SC($k{=}10$).}
\label{tab:token-budget}
\small
\begin{tabular}{l ccc ccc}
\toprule
& \multicolumn{3}{c}{\textbf{GPQA-Diamond}} & \multicolumn{3}{c}{\textbf{BBH (27 tasks)}} \\
\cmidrule(lr){2-4} \cmidrule(lr){5-7}
Method & Acc (\%) & Tok/prob & Acc/MT & Acc (\%) & Tok/prob & Acc/MT \\
\midrule
Zero-shot CoT & 66.7 & 286K & 2.3 & 70.4 & 1{,}521K & 0.5 \\
SC-MoA (ungated) & \textbf{72.7} & \textbf{631K} & \textbf{1.2} & \textbf{89.7} & \textbf{2{,}921K} & \textbf{0.3} \\
SC-MoA ($k{=}2$, gated) & 72.7 & 1{,}052K & 0.7 & 89.4 & 5{,}820K & 0.2 \\
TextGrad & 69.7 & 1{,}529K & 0.5 & 84.0 & 12{,}698K & 0.1 \\
SC ($k{=}10$) & 70.7 & 1{,}934K & 0.4 & 88.1 & 7{,}988K & 0.1 \\
Self-MoA & 64.1 & 1{,}426K & 0.5 & 86.3 & 5{,}680K & 0.2 \\
MoA$^\ddagger$ & 65.2 & --- & --- & --- & --- & --- \\
\bottomrule
\end{tabular}
\end{table}

\noindent $^\ddagger$Token counts unavailable for the paraphrase-based
MoA variant ($N{=}4$, 9 calls); see Appendix~\ref{app:persona-vs-paraphrase}.

\noindent On GPQA, SC-MoA (ungated) uses $3.1{\times}$ fewer tokens
than SC($k{=}10$) while achieving $+2.0$\,pp higher accuracy.  On
BBH the ratio is $2.7{\times}$ fewer tokens at $+1.6$\,pp.  The
efficiency comes from Phase~1 generating only $N{=}5$ base proposals
plus a single aggregation call, whereas SC generates 10 independent
samples.  The gated variant is more expensive because it also runs
refinement and inner SC ($k{=}2$).  Across both benchmarks, SC-MoA
dominates every baseline on the accuracy-per-token frontier.

\section{Contamination Analysis}
\label{app:contamination}

We analyze potential training-data contamination across all three
benchmark families.

\paragraph{BBH.}
We stratify the 27 BBH tasks by SC-MoA's improvement over
SC($k{=}10$).  Of 27 tasks, 8 show improvement ($+1$ to $+13.5$\,pp),
3 show degradation ($-4$ to $-8$\,pp), and 16 are tied (14 of which
are solved perfectly by both methods).  The improved tasks span
diverse cognitive domains---state tracking
(\emph{penguins\_in\_a\_table}, $+13.5$\,pp), temporal reasoning
(\emph{date\_understanding}, $+9$\,pp), and linguistic parsing
(\emph{disambiguation\_qa}, $+12$\,pp)---with no concentration in
tasks more likely to be contaminated.

\paragraph{GPQA-Diamond.}
GPQA was designed to be ``Google-proof'': domain experts achieve only
${\sim}$65\% accuracy, and the questions were validated against web
search~\citep{rein2024gpqa}.  Stratifying by scientific domain:
Chemistry ($n{=}93$) improves by $+9.7$\,pp over SC, Physics
($n{=}86$) by $+1.2$\,pp, and Biology ($n{=}20$) degrades by
$-8.2$\,pp.  The Biology degradation is consistent with a small
sample ($n{=}20$, not significant) rather than domain-specific
contamination.

\paragraph{LCB-Hard.}
All 171 problems have \texttt{contest\_date} $\geq$ September~2024,
placing them after the training cutoff by construction.  We
additionally verified that SC-MoA does not preferentially solve
older problems within this window: the date distribution of solved
vs.\ unsolved problems is indistinguishable (Kolmogorov--Smirnov
$D{=}0.114$, $p{=}0.633$).  Per-month accuracy ranges from 28.6\%
(April 2025, $n{=}7$) to 77.8\% (December 2024, $n{=}18$), with no
monotonic trend.

\section{Aggregation Value Analysis}
\label{app:aggregation-value}

On QA benchmarks, the score-based override (Phase~5) never fires
because the quality score is always 1.0.  This isolates the
aggregator's contribution: any accuracy difference between
majority-vote-only and SC-MoA is attributable to Phase~4 aggregation.

We decompose the aggregator's effect on GPQA-Diamond by comparing the
pre-aggregation majority answer against the post-aggregation final
answer for each problem:

\begin{table}[h]
\centering
\caption{Aggregation value on GPQA-Diamond ($n{=}198$): $2 \times 2$
contingency table comparing majority-vote-only vs.\ aggregated answer.
McNemar $\chi^2{=}2.88$, $p{=}0.090$.}
\label{tab:aggregation-value}
\small
\begin{tabular}{lcc}
\toprule
& Aggregated correct & Aggregated wrong \\
\midrule
Majority correct & 133 & 5 \\
Majority wrong   & 12 & 48 \\
\bottomrule
\end{tabular}
\end{table}

\begin{table}[h]
\centering
\caption{Aggregation value on BBH-3 ($n{=}296$, pre-refinement
controlled setting): $2 \times 2$ contingency table comparing
majority-vote-only vs.\ synthesized answer.
McNemar $\chi^2{=}5.76$, $p{=}0.016$.}
\label{tab:aggregation-value-bbh}
\small
\begin{tabular}{lcc}
\toprule
& Synthesized correct & Synthesized wrong \\
\midrule
Majority correct & 243 & 5 \\
Majority wrong   & 16 & 32 \\
\bottomrule
\end{tabular}
\end{table}

\noindent \textbf{Note:} these are pre-refinement numbers from the
\S\ref{sec:paradox} controlled setting (synthesis vs.\ voting on the
same proposals).  After anchored refinement in the full pipeline,
the harmful flips are eliminated (Table~\ref{tab:synthesis-decomp}
reports 46 beneficial and 0 harmful on BBH).

\paragraph{Synthesis advantage decomposition across benchmarks.}
Table~\ref{tab:synthesis-decomp} extends the aggregation-value analysis
to four benchmarks using paraphrase perturbations, decomposing the net
synthesis advantage into recovery ($P_w \cdot R_s$) and corruption
($P_c \cdot (1{-}F_s)$) terms per
Proposition~\ref{prop:synthesis} (Figure~\ref{fig:mechanism}b).
GPQA-Diamond and LCB-Hard use the SPUQ paraphrase variant ($N{=}4$,
$k{=}5$); BBH and AIME use $N{=}5$, $k{=}2$.  Recovery dominates
corruption on every benchmark, and all four McNemar tests are
significant ($p{<}0.05$), confirming that beneficial flips reliably
outnumber harmful flips.

\begin{table}[h]
\centering
\caption{Synthesis advantage decomposition across four benchmarks
(paraphrase perturbation variants).  Config column specifies the
operating point; the decisive subset excludes
problems with degenerate answer-only traces (GPQA: 9 excluded,
LCB: 7 excluded due to sandbox timeouts).
Recovery and corruption are expressed as percentage-point
contributions.  All McNemar tests are significant at $p{<}0.05$.
MMLU-ML ($n{=}112$) is excluded because its small decisive set
($n{<}10$ flips) yields unreliable decomposition estimates.}
\label{tab:synthesis-decomp}
\small
\begin{tabular}{llccccc}
\toprule
Benchmark & Config & $n$ & Benef. & Harm. & Recovery (pp) & Net (pp) \\
\midrule
GPQA-Diamond  & $N{=}4$, $k{=}5$ & 189 & 15 & 4  & $+$7.9  & $+$5.8  \\
BBH           & $N{=}5$, $k{=}2$ & 296 & 46 & 0  & $+$15.5 & $+$15.5 \\
AIME 2022--24 & $N{=}5$, $k{=}2$ &  90 & 20 & 0  & $+$22.2 & $+$22.2 \\
LCB-Hard      & $N{=}4$, $k{=}5$ & 164 & 24 & 5  & $+$14.6 & $+$11.6 \\
\bottomrule
\end{tabular}
\end{table}

\subsection{Case Studies}
\label{app:case-studies}

We present a representative beneficial flip from GPQA-Diamond to
illustrate the aggregation mechanism.

\paragraph{Example: GPQA problem 69 (chemistry, point-group symmetry).}
Five perturbations produce proposals:
\begin{itemize}[leftmargin=*,itemsep=1pt]
\item \textbf{Para 1} ($k{=}2$, both wrong): concludes C$_{2v}$
\item \textbf{Para 2} ($k{=}2$, 1 wrong/1 correct): first sample says D$_{\infty h}$, second says D$_{\infty h}$
\item \textbf{Para 3} ($k{=}2$, 1 correct/1 wrong): first says B (C$_{2v}$), second disagrees
\item \textbf{Para 4} ($k{=}2$, both wrong): concludes D$_{\infty h}$ (answer D)
\item \textbf{Para 5} ($k{=}2$, both wrong): concludes C$_{2v}$ (answer B)
\end{itemize}

\noindent SC majority vote: 2/10 samples are correct $\to$ \textbf{wrong}.
Pre-refinement consensus: $C = 0.4$ (split between B and D).
Three minorities adopt the emerging answer during refinement
($C_{\text{post}} = 0.6$).
The aggregator reads all traces---including the correct reasoning
from Paras~2 and 3---and outputs the correct answer \textbf{B}.

The mechanism is not best-of-N selection: the aggregator does not
see correctness labels.  It synthesizes across reasoning traces,
weighting arguments that identify the molecule as bent (C$_{2v}$)
over those that misidentify it as linear (D$_{\infty h}$).

Among the 16 beneficial flips in the SPUQ paraphrase variant
(Table~\ref{tab:aggregation-value}), the majority follow this pattern:
a minority of perturbations contain the correct reasoning, and the
aggregator recovers it by integrating evidence across traces.

\section{Trace Diversity Analysis}
\label{app:trace-diversity}

Figure~\ref{fig:mechanism}(a) shows that trace-level diversity ($D_t$)
predicts beneficial aggregation flips across four benchmarks
(Proposition~\ref{prop:synthesis}).  Here we report the full measurement
methodology and results.

\paragraph{Methodology.}
For each problem, we extract the first reasoning trace from each
perturbation and compute $D_t$ as the mean pairwise cosine distance in
TF-IDF space (5000 features, English stop words removed).  Problems
where all traces are answer-only (degenerate $D_t{=}0$) are excluded
from the distributional analysis (24 problems total: 7 in GPQA, 17 in
BBH).  GPQA-Diamond and LCB-Hard use SPUQ paraphrases ($N{=}4$,
$k{=}5$); BBH (date understanding, disambiguation QA, penguins in a
table) and AIME 2022--24 use $N{=}5$, $k{=}2$.

\paragraph{Results.}
Table~\ref{tab:trace-diversity} reports $D_t$ stratified by aggregation
flip type across all four benchmarks.

\begin{table}[h]
\centering
\caption{Trace diversity ($\bar{D}_t$) by aggregation outcome across
four benchmarks.  Beneficial flips consistently show higher trace
diversity than non-flips.  Degenerate zero-diversity traces excluded.}
\label{tab:trace-diversity}
\small
\begin{tabular}{llccc}
\toprule
Benchmark & Outcome & $n$ & $\bar{D}_t$ & Std \\
\midrule
\multirow{3}{*}{GPQA-Diamond}
  & Beneficial flip & 15  & \textbf{0.585} & 0.171 \\
  & No flip         & 163 & 0.460          & 0.170 \\
  & Harmful flip    & 4   & 0.521          & 0.139 \\
\midrule
\multirow{2}{*}{BBH}
  & Beneficial flip & 44  & \textbf{0.564} & 0.124 \\
  & No flip         & 235 & 0.522          & 0.120 \\
\midrule
\multirow{2}{*}{AIME 2022--24}
  & Beneficial flip & 20  & \textbf{0.521} & 0.145 \\
  & No flip         & 70  & 0.449          & 0.135 \\
\midrule
\multirow{3}{*}{LCB-Hard}
  & Beneficial flip & 24  & \textbf{0.476} & 0.167 \\
  & No flip         & 135 & 0.410          & 0.152 \\
  & Harmful flip    & 5   & 0.485          & 0.060 \\
\bottomrule
\end{tabular}
\end{table}

\noindent Welch $t$-tests (beneficial vs.\ non-flip): GPQA-Diamond
$t{=}2.70$, $p{=}0.015$; BBH $t{=}2.05$, $p{=}0.045$; AIME
$t{=}1.99$, $p{=}0.057$; LCB-Hard $t{=}1.80$, $p{=}0.081$.  The
effect is significant at $p{<}0.05$ on GPQA-Diamond and BBH, and
directionally consistent on AIME and LCB-Hard.  Point-biserial
correlations between $D_t$ and beneficial flip occurrence: $r{=}0.20$
($p{=}0.007$) on GPQA, $r{=}0.13$ ($p{=}0.037$) on BBH, $r{=}0.22$
($p{=}0.042$) on AIME, $r{=}0.15$ ($p{=}0.055$) on LCB-Hard.

\paragraph{Harmful flips.}
Harmful flips occur only on GPQA-Diamond ($n{=}4$) and LCB-Hard
($n{=}5$); BBH and AIME have zero harmful flips.  On both benchmarks
where harmful flips exist, their $D_t$ is comparable to beneficial
flips ($0.521$ and $0.485$, respectively).  This is consistent with
Proposition~\ref{prop:synthesis}: high trace diversity increases
\emph{both} flip types, but the net advantage is positive because
synthesis fidelity $F_s$ is high (Table~\ref{tab:synthesis-decomp}).

\section{Embedding-Based Diversity Validation}
\label{app:embedding-diversity}

The TF-IDF representation used in Appendix~\ref{app:trace-diversity}
captures lexical variation but fits a small per-problem vocabulary
(4~documents per problem at $N{=}4$), raising the question of whether
the $D_t$--flip association reflects genuine structural differences in
reasoning or merely surface rewording.  We validate the finding under
dense pre-trained embeddings that encode semantic structure
independently of per-problem term frequencies.

\paragraph{Methodology.}
We embed all reasoning traces using Amazon Titan Text Embeddings v2
(1024 dimensions, L2-normalized), a pre-trained dense encoder whose
representation is fixed across problems---unlike TF-IDF, which fits
a separate vocabulary per problem on only $N$ documents.  We compute
$D_t$ as $1 - \mathbf{v}_i^\top \mathbf{v}_j$ over all perturbation
pairs, identically to the TF-IDF pipeline but replacing the sparse
representation.  The analysis uses the mechanistic configuration
($N{=}4$, $k{=}5$, GPQA-Diamond, 198~problems, 3{,}390~unique traces).
TF-IDF excludes 8 degenerate problems with answer-only traces
($n{=}190$); embeddings retain all 198.

\paragraph{Results.}
Table~\ref{tab:embedding-diversity} reports $D_t$ under both
representations side by side on the same dataset.

\begin{table}[h]
\centering
\caption{Trace diversity ($\bar{D}_t$) under TF-IDF vs.\ dense
embeddings on GPQA-Diamond ($N{=}4$, $k{=}5$).  Both representations
yield large, significant effects (Cohen's $d \approx 0.92$).}
\label{tab:embedding-diversity}
\small
\begin{tabular}{lcccccc}
\toprule
& \multicolumn{3}{c}{TF-IDF ($n{=}190$)} & \multicolumn{3}{c}{Titan Embed 1024d ($n{=}198$)} \\
\cmidrule(lr){2-4} \cmidrule(lr){5-7}
Outcome & $n$ & $\bar{D}_t$ & $d$ & $n$ & $\bar{D}_t$ & $d$ \\
\midrule
Beneficial flip & 15 & $0.682 \pm 0.149$ & \multirow{2}{*}{\textbf{0.93}} & 16 & $0.494 \pm 0.208$ & \multirow{2}{*}{\textbf{0.92}} \\
No flip         & 170 & $0.514 \pm 0.182$ & & 177 & $0.316 \pm 0.191$ & \\
Harmful flip    & 5  & $0.558 \pm 0.106$ & --- & 5  & $0.310 \pm 0.160$ & --- \\
\bottomrule
\end{tabular}
\end{table}

\noindent Welch $t$-tests: TF-IDF $t{=}3.44$, $p{=}0.0007$; embedding
$t{=}3.53$, $p{=}0.0005$.  Point-biserial correlations are nearly
identical: $r{=}0.246$ under both ($p < 0.001$).

\paragraph{Interpretation.}
Dense embeddings compress the absolute $D_t$ scale (the overall mean
drops from $0.529$ to $0.330$) because semantically similar traces
cluster more tightly in embedding space.  Despite this compression,
the beneficial-flip separation is preserved with essentially identical
effect size, indicating that the signal is carried by genuine
reasoning-strategy variation rather than lexical noise.

The most informative discrepancy is in harmful flips:
under TF-IDF, harmful flips show moderate diversity
($\bar{D}_t{=}0.558$, above the no-flip baseline of $0.514$),
but under embeddings they collapse to baseline
($\bar{D}_t{=}0.310$ vs.\ no-flip $0.316$).
This indicates that the TF-IDF diversity in harmful cases is
driven by surface rewording---different words expressing the same
flawed reasoning---rather than the structural reasoning divergence
that enables recovery.  The embedding representation strips away
this superficial variation, cleanly separating the beneficial
signal from noise.

\section{Controlled Perturbation vs.\ Free-Form Paraphrasing}
\label{app:controlled-perturbation}

SC-MoA's default perturbation strategy (Phase~0) uses free-form
paraphrasing: a single LLM call generates $N$ unconstrained
meaning-preserving rephrasings (instantiated via SPUQ~\citep{gao2024spuq}).  An alternative
is \emph{controlled perturbation}, where each paraphrase applies
exactly one named linguistic operation drawn from a fixed taxonomy.
This section reports a head-to-head comparison on GPQA-Diamond.

We adapt the controlled-perturbation taxonomy from
Polyjuice~\citep{wu2021polyjuice} and Tailor~\citep{ross2022tailor}
to an LLM-based generator, defining five meaning-preserving operation
types (lexical, syntactic, discourse, register, compression).  Each
of $N{=}5$ agents receives a paraphrase generated under a different
type; the same protected-token validation pipeline applies.

\paragraph{Experimental setup.}
We evaluate on a 50-problem subset of GPQA-Diamond using
\texttt{gpt-oss-120b} for both paraphrase generation and SC-MoA
proposing/aggregation ($N{=}5$, $k{=}2$, always-aggregate).  Both
conditions use identical model, temperature, and caching
infrastructure; the only difference is the perturbation generator.
Paraphrases are pre-generated and cached before the SC-MoA run to
avoid rate-limit interference.

\paragraph{Results.}
Table~\ref{tab:controlled-vs-spuq} summarizes the comparison.

\begin{table}[h]
\centering
\caption{Head-to-head on GPQA-Diamond ($n{=}50$, \texttt{gpt-oss-120b},
$N{=}5$, $k{=}2$).  Both methods use identical SC-MoA
infrastructure; the only variable is the perturbation generator.}
\label{tab:controlled-vs-spuq}
\small
\begin{tabular}{lccccc}
\toprule
Method & Accuracy & $\bar{C}_{\text{pre}}$ & $\bar{C}_{\text{post}}$
& Avg calls & Validation \\
\midrule
Controlled perturbation & \textbf{76.0\%} & 0.804 & 0.960 & 12.0 & 97.2\% \\
Free-form SPUQ          & 74.0\%          & 0.800 & 0.960 & 12.0 & 98.0\% \\
\bottomrule
\end{tabular}
\end{table}

\noindent The $+2.0$\,pp difference is within the margin expected
from $n{=}50$ (McNemar's test: 4~vs.~3 discordant pairs,
$p{=}1.0$).  Both methods produce nearly identical consensus
profiles ($\bar{C}_{\text{pre}} \approx 0.80$,
$\bar{C}_{\text{post}} \approx 0.96$) and use the same number of
LLM calls.

\paragraph{Power analysis.}
The $n{=}50$ sample yields 80\% power to detect a ${\geq}12$\,pp
effect (two-sided McNemar, $\alpha{=}0.05$).  The observed 2\,pp
difference is well within the noise floor.  We interpret this as
absence of evidence for a large perturbation-content effect, not as
evidence of absence.  A properly powered comparison
($n \geq 400$) would require additional compute budget.

\paragraph{Per-type validation.}
Controlled perturbations achieve 97.2\% overall validation (243/250),
compared to 98.0\% for free-form SPUQ (245/250).  Lexical and
syntactic perturbations preserve all protected tokens (100\%);
compression is hardest (92\%), as tightening prose occasionally drops
a unit or number.

We retain SPUQ as the default because it requires one LLM call
vs.\ five, reducing Phase~0 cost by ${\sim}5{\times}$.

\subsection{Lightweight $k{=}1$ variant}
\label{app:k1-variant}

\begin{table}[h]
\centering
\caption{\textbf{$k{=}1$ variant} ($N{=}4$, SPUQ paraphrases,
\texttt{gpt-oss-120b}).  SC-MoA at $k{=}1$ uses
${\sim}5.2$ LLM calls and matches or exceeds SC($k{=}10$) on 3 of 4
benchmarks.  \textbf{Note:} this table uses $N{=}4$; Table~\ref{tab:main}
uses $N{=}5$, $k{=}2$---numbers are not directly comparable.}
\label{tab:k1}
\small
\begin{tabular}{lccccc}
\toprule
Method & BBH (2{,}561) & GPQA & LCB-Hard & Calls \\
\midrule
SC ($k{=}10$) & 88.1 & 69.7 & 49.1 & 10 \\
\rowcolor{scmoarow!50}
SC-MoA ($N{=}4$, $k{=}1$) & \textbf{89.4} & \textbf{74.7} & \textbf{55.6} & ${\sim}5.2$ \\
\rowcolor{scmoarow}
SC-MoA ($N{=}4$, $k{=}2$) & \textbf{89.7} & \textbf{75.4} & \textbf{63.7} & ${\sim}11.5$ \\
\bottomrule
\end{tabular}
\end{table}

At $k{=}1$, SC-MoA uses ${\sim}5.2$ calls (no inner SC) and
exceeds SC($k{=}10$) on GPQA ($+5.0$\,pp) at roughly half the
compute (Figure~\ref{fig:pareto}; discussion in
\S\ref{sec:results}).

\section{Additional Ablations and Stability}
\label{app:additional-ablations}

\subsection{Trace Ablation: Full Results}
\label{app:trace-ablation}

We ablate the trace content visible to the aggregator, holding
Phases~1--3 constant (same SPUQ-paraphrased proposals, same
refinement, same post-refinement clusters).  Only the Phase~4
aggregation call differs.  Eight conditions test distinct hypotheses
about what information drives the aggregator's beneficial corrections
(Table~\ref{tab:trace-ablation}).  ``Ben.''\ and ``Harm.''\ count
beneficial and harmful flips vs.\ majority vote.

\begin{table}[h]
\centering
\caption{\textbf{Trace ablation on GPQA-Diamond} ($n{=}198$, SPUQ
paraphrases, $N{=}5$, $k{=}2$).  Beneficial and harmful flips are
measured against majority vote.  $p^*$: Holm--Bonferroni-corrected
McNemar $p$-value (7 comparisons against majority vote).}
\label{tab:trace-ablation}
\small
\begin{tabular}{lrrrrl}
\toprule
Condition & Acc.\ (\%) & Ben.\ & Harm.\ & Ratio & $p^*$ \\
\midrule
Majority vote       & 69.7 & --  & --  & --   & --          \\
Full trace          & 72.2 & 8   & 3   & 2.67 & --          \\
(a) Answer-only     & 69.7 & 6   & 6   & 1.00 & 1.00        \\
(b) Truncated-400   & 73.2 & 10  & 3   & 3.33 & 0.65        \\
(c) Majority-only   & 69.7 & 3   & 3   & 1.00 & 1.00        \\
(d) Minority-only   & 73.2 & 11  & 4   & 2.75 & 0.71        \\
(e) Permuted        & 73.2 & 11  & 4   & 2.75 & 0.71        \\
(f) Random 3/5      & 70.2 & 6   & 5   & 1.20 & 1.00        \\
(g) No metadata     & 72.7 & 9   & 3   & 3.00 & 0.71        \\
\bottomrule
\end{tabular}
\end{table}

\noindent
Table~\ref{tab:trace-ablation} reports all conditions; we highlight
the consensus-stratified findings below.

\paragraph{Consensus-stratified analysis.}
The effect concentrates in the strong-consensus stratum (23
problems, 12\% of total): full-trace accuracy jumps from 21.7\%
(majority vote) to 47.8\%, and minority-only traces reach
47.8\% as well---more than doubling majority-vote accuracy on the
hardest disputed problems.  On unanimous problems (175/198, 88\%), all
conditions perform within ${\pm}1.7$\,pp of majority vote, as
expected: when no minority traces exist, the aggregator has no
complementary reasoning to integrate.

\paragraph{Degenerate cases.}
In condition (d), 175/198 problems (88.4\%) have unanimous
post-refinement consensus, leaving no genuine minority traces.
For these, the aggregator receives a synthetic message noting
unanimity; it flips away from the unanimous answer on only 2/175
problems (1.1\%), confirming that the aggregator respects strong
consensus even without the evolution narrative.

\subsection{Information Ladder: Compute vs.\ Evidence}
\label{app:info-ladder}

The trace ablation (\S\ref{app:trace-ablation}) establishes that
answer-only and majority-only conditions collapse to MV, but cannot
distinguish whether the aggregator's gain comes from
\emph{re-solving with additional compute} or from
\emph{synthesizing complementary trace evidence}.  We construct an
information ladder that progressively reveals proposer information
while holding Phases~1--3 constant
(Table~\ref{tab:info-ladder}).

\paragraph{Conditions.}
Eight conditions span zero to full information.
\textbf{Q-only (agg):} the aggregator receives only the problem
under its standard system prompt (mentioning ``specialist
proposals''), with zero proposals---a pure re-solve baseline.
\textbf{Q-only (CoT):} identical, but using the proposer's
chain-of-thought prompt, controlling for whether the aggregator
prompt biases re-solving.
\textbf{Answers bare:} an unnumbered list of extracted answer labels
(e.g., ``A, A, A, B, C'') with no metadata---testing whether the
raw vote distribution helps beyond re-solving.
\textbf{Answers + metadata:} the \texttt{answer\_only} condition
from \S\ref{app:trace-ablation}, including persona names, quality,
intra-agreement, cluster labels, and evolution narrative.
\textbf{Traces bare:} full proposal text (1{,}800 char cap) with
generic headers, no persona names, quality, clusters, or evolution.
\textbf{Traces + partial meta:} the \texttt{no\_metadata} condition.
\textbf{Traces $-$ evolution:} the \texttt{no\_evolution} condition.
\textbf{Full:} the unmodified Phase~4 call.

\begin{table}[h]
\centering
\caption{\textbf{Information ladder on GPQA-Diamond}
($n{=}198$, SPUQ paraphrases, $N{=}5$, $k{=}2$).
Conditions ordered by increasing information available to the
aggregator.
Ben./Harm.\ = beneficial/harmful flips vs.\ majority vote.
$p$: McNemar, Holm--Bonferroni corrected (7 comparisons vs.\ MV).}
\label{tab:info-ladder}
\small
\begin{tabular}{llrrrrl}
\toprule
& Condition & Acc.\ (\%) & Ben.\ & Harm.\ & Ratio & $p$ \\
\midrule
& Majority vote       & 69.7 & --  & --  & --   & --  \\
\cmidrule{1-7}
\multirow{2}{*}{\rotatebox{90}{\scriptsize re-solve}}
& Q-only (agg prompt)  & 55.6 & 10  & 38  & 0.26 & $<$0.001 \\
& Q-only (CoT prompt)  & 63.1 & 10  & 23  & 0.43 & 0.035 \\
\cmidrule{1-7}
\multirow{2}{*}{\rotatebox{90}{\scriptsize votes}}
& Answers bare          & 72.2 &  9  &  4  & 2.25 & 0.267 \\
& Answers + metadata    & 69.7 &  6  &  6  & 1.00 & 1.000 \\
\cmidrule{1-7}
\multirow{4}{*}{\rotatebox{90}{\scriptsize traces}}
& Traces bare           & 72.7 & 11  &  5  & 2.20 & 0.210 \\
& Traces + partial meta & 72.7 &  9  &  3  & 3.00 & 0.146 \\
& Traces $-$ evolution  & 74.2 & 11  &  2  & 5.50 & 0.023 \\
& Full (all info)       & 72.2 &  8  &  3  & 2.67 & 0.227 \\
\bottomrule
\end{tabular}
\end{table}

\paragraph{Value decomposition.}
\begin{align*}
\underbrace{+2.5\text{\,pp}}_{\text{full} - \text{MV}} \;=\;
\underbrace{-14.1\text{\,pp}}_{\text{re-solve}} \;+\;
\underbrace{+16.7\text{\,pp}}_{\text{trace evidence}}
\end{align*}
The re-solve component is strongly negative: the aggregator
destroys accuracy when forced to reason from scratch.  The entire
$+16.7$\,pp trace-evidence component must compensate for the
re-solve penalty to produce the net $+2.5$\,pp gain.

\paragraph{The CoT prompt control.}
Q-only (CoT) scores 63.1\%---$+7.5$\,pp above Q-only (agg) but
still $-6.6$\,pp below MV.  The aggregator system prompt, which
references ``specialist proposals'' and ``minority positions,''
confuses the model when no proposals are present.  Neither
re-solve variant approaches MV.

\paragraph{Vote distribution provides weak signal.}
Answers bare (72.2\%) matches full-trace accuracy and exceeds MV,
suggesting the raw vote distribution provides some signal.  Adding
structured metadata to answer labels \emph{degrades} accuracy back
to MV (69.7\%, ratio 1.00)---metadata without supporting reasoning
introduces noise that overrides the implicit vote signal.

\paragraph{Traces are the active ingredient.}
Traces bare (72.7\%) and traces with partial metadata (72.7\%) both
exceed MV, with beneficial-to-harmful ratios of 2.20 and 3.00.
Removing the evolution narrative while keeping full metadata yields
the highest accuracy (74.2\%, ratio 5.50, $p{=}0.023$)---consistent
with the finding that evolution narrative introduces mild noise.

\paragraph{Consensus-stratified analysis.}
On unanimous problems (175/198, 88\%), all trace-level conditions
cluster within ${\pm}0.6$\,pp of MV (76.0\%).  The re-solve
conditions are the exception: Q-only (agg) drops to 58.9\%
($-17.1$\,pp), showing that re-solving corrupts even easy problems.

On contested problems (23/198, 12\%), MV scores 21.7\%; Q-only
(agg) reaches 30.4\%; but traces bare jumps to 47.8\% and traces
$-$ evolution reaches 60.9\%.  The aggregator nearly triples
contested-problem accuracy when given reasoning traces, confirming
complementary evidence from diverse proposers is the mechanism.

\begin{table}[h]
\centering
\caption{\textbf{Consensus-stratified information ladder} on
GPQA-Diamond.  Unanimous: post-refinement consensus $\geq 80\%$;
contested: $< 80\%$.}
\label{tab:info-ladder-strat}
\small
\begin{tabular}{lrrrrrr}
\toprule
& \multicolumn{3}{c}{Unanimous ($n{=}175$)} &
  \multicolumn{3}{c}{Contested ($n{=}23$)} \\
\cmidrule(lr){2-4} \cmidrule(lr){5-7}
Condition & Acc & Ben & Harm & Acc & Ben & Harm \\
\midrule
Majority vote      & 76.0 & -- & -- & 21.7 & -- & -- \\
Q-only (agg)       & 58.9 &  5 & 35 & 30.4 &  5 &  3 \\
Q-only (CoT)       & 67.4 &  5 & 20 & 30.4 &  5 &  3 \\
Answers bare       & 76.6 &  2 &  1 & 39.1 &  7 &  3 \\
Answers + meta     & 75.4 &  1 &  2 & 26.1 &  5 &  4 \\
Traces bare        & 76.0 &  1 &  1 & 47.8 & 10 &  4 \\
Traces + meta      & 75.4 &  0 &  1 & 52.2 &  9 &  2 \\
Traces $-$ evo     & 76.0 &  0 &  0 & 60.9 & 11 &  2 \\
Full               & 75.4 &  0 &  1 & 47.8 &  8 &  2 \\
\bottomrule
\end{tabular}
\end{table}

\paragraph{Adjacent-pair McNemar tests.}
Four pre-registered adjacent contrasts along the ladder
(Table~\ref{tab:info-ladder-mcnemar}), corrected with
Holm--Bonferroni ($m{=}4$).

\begin{table}[h]
\centering
\caption{\textbf{Adjacent-pair McNemar tests} along the information
ladder.}
\label{tab:info-ladder-mcnemar}
\small
\begin{tabular}{lrrrr}
\toprule
Contrast & Ben & Harm & $p_{\text{raw}}$ & $p_{\text{corr}}$ \\
\midrule
Q-only $\to$ Answers bare   & 40 &  7 & $<$0.001 & $<$0.001 \\
Answers bare $\to$ Ans+meta &  2 &  7 & 0.180    & 0.539    \\
Ans+meta $\to$ Traces bare  & 11 &  5 & 0.210    & 0.539    \\
Traces bare $\to$ Full      &  5 &  6 & 1.000    & 1.000    \\
\bottomrule
\end{tabular}
\end{table}

\noindent
Only the first contrast is significant: the jump from zero
proposer information to vote distribution accounts for essentially
all recoverable accuracy.  Subsequent steps do not produce
individually significant gains against their immediate predecessor,
though trace-level conditions collectively dominate answer-level
conditions in flip ratio (2.20--5.50 vs.\ 1.00--2.25).

\begin{figure}[h]
\centering
\includegraphics[width=0.95\columnwidth]{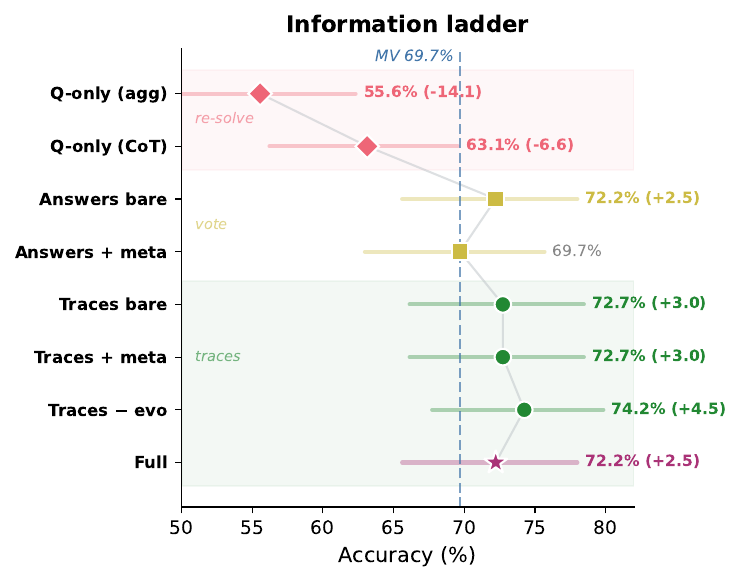}
\caption{\textbf{Information ladder} on GPQA-Diamond ($n{=}198$).
Accuracy with Wilson 95\% CIs, ordered by increasing information.
The dashed line marks majority vote (69.7\%).  Re-solve conditions
(red) fall well below MV; trace-level conditions (green) cluster
above it.}
\label{fig:info-ladder-supp}
\end{figure}

\begin{figure}[h]
\centering
\includegraphics[width=0.95\columnwidth]{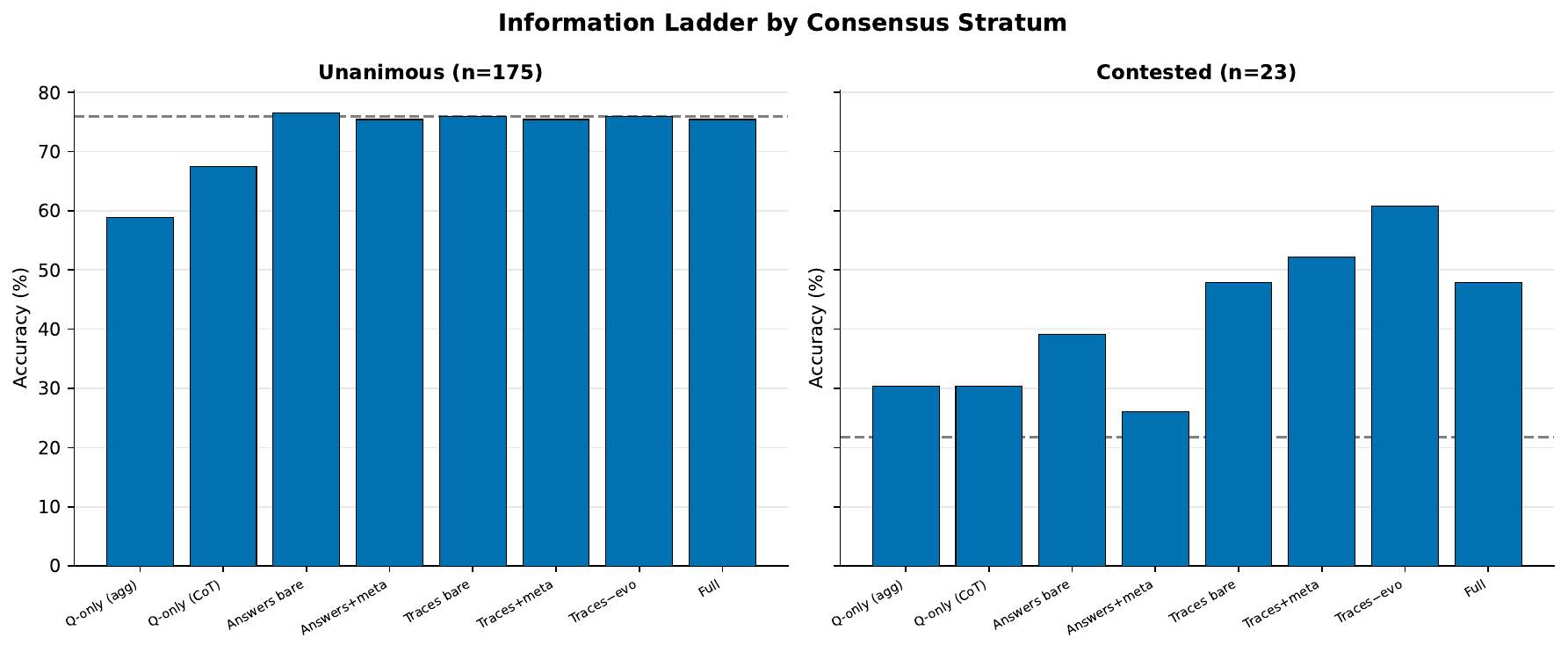}
\caption{\textbf{Consensus-stratified information ladder.}
On unanimous problems (left, $n{=}175$), all conditions cluster near
MV except re-solve, which collapses.  On contested problems (right,
$n{=}23$), trace conditions nearly triple MV accuracy.}
\label{fig:info-ladder-strat-fig}
\end{figure}

\begin{figure}[h]
\centering
\includegraphics[width=0.95\columnwidth]{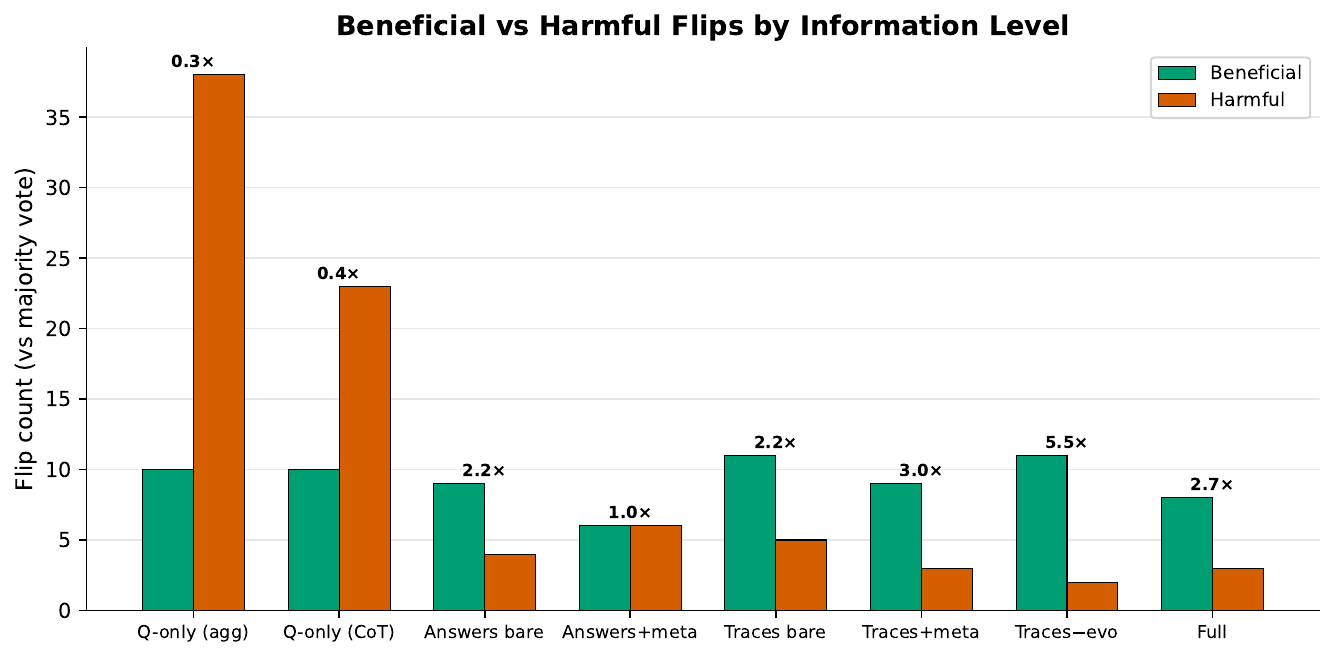}
\caption{\textbf{Beneficial-to-harmful flip ratio} across the
information ladder.  Re-solve conditions ($<1{\times}$) are
corruption-dominated; trace conditions ($2$--$5.5{\times}$) are
recovery-dominated.}
\label{fig:flip-ratio-supp}
\end{figure}

\subsection{Selection Baselines: Full Results}
\label{app:selection-baselines}

We compare SC-MoA's synthesis aggregation against two selection
baselines that choose the best \emph{existing} proposal rather than
generating a new one, using the same Phase~1 proposals ($N{=}5$,
$k{=}2$).

\paragraph{Dawid--Skene EM~\citep{dawid1979}.}
Multinomial Dawid--Skene estimates a per-agent confusion matrix
$\pi_j[c_{\text{true}}, c_{\text{obs}}]$ via EM across all 198 GPQA
problems, then selects the answer with highest reliability-weighted
posterior per problem.  This adds 0 extra LLM calls---it is a pure
post-hoc reweighting of the existing agent votes.

\paragraph{Pairwise LLM judge~\citep{zheng2023judging}.}
For each problem, all $\binom{5}{2}{=}10$ proposal pairs are
presented to the same model (gpt-oss-120b) with a judge prompt
requesting a binary ``A'' or ``B'' verdict.  Borda count aggregates
pairwise wins; the proposal with the highest Borda score is selected.
This adds 10 LLM calls per problem.

\begin{table}[h]
\centering
\caption{\textbf{Selection vs.\ synthesis on GPQA-Diamond}
($n{=}198$).  McNemar tests are against SC-MoA synthesis.}
\label{tab:selection-baselines}
\small
\begin{tabular}{lccccl}
\toprule
Method & Acc.\ (\%) & 95\% CI & Calls & McNemar $p$ \\
\midrule
Majority vote ($N{=}5$) & 69.2 & $[62.6, 75.8]$ & 10 & 0.012 \\
Dawid--Skene EM         & 68.7 & $[62.1, 75.3]$ & 10 & 0.007 \\
SC ($k{=}10$)           & 69.7 & $[63.1, 76.3]$ & 10 & 0.027 \\
Pairwise judge          & 75.3 & $[69.2, 81.3]$ & 21 & 1.000 \\
\midrule
\rowcolor{scmoarow}
SC-MoA synthesis        & 75.3 & $[69.2, 81.3]$ & 11.5 & --- \\
\bottomrule
\end{tabular}
\end{table}

\noindent
SC-MoA synthesis significantly outperforms all voting and selection
baselines at $p < 0.03$.  The pairwise judge matches accuracy but
at nearly $2{\times}$ compute (21~vs.\ 11.5~calls per problem),
making SC-MoA Pareto-dominant on the accuracy--compute frontier.

Dawid--Skene slightly \emph{underperforms} naive majority vote
($68.7\%$ vs.\ $69.2\%$), consistent with the known limitation that
EM-based annotator models require substantially more annotators than
$N{=}5$ to estimate reliable confusion
matrices~\citep{dawid1979}.  The learned weights add noise rather
than signal at this scale.

The pairwise judge result is informative: it shows that the correct
answer \emph{is present} among the proposals ${\sim}75\%$ of the
time---the challenge is extracting it.  Selection via $\binom{5}{2}$
comparisons and synthesis via a single aggregation call reach the
same accuracy, but synthesis does so at half the cost.  This confirms
that the aggregator's advantage is efficiency, not access to
information that selection methods lack.

\subsection{Stability Under Input Perturbation}
\label{app:stability}

All experiments use greedy decoding with content-hash caching, so
re-running produces bitwise-identical results.  We quantify stability
to the \emph{choice} of input perturbation---a stronger test than seed
variation, since it varies the actual problem text.

\paragraph{Bootstrap accuracy CIs.}
10{,}000 bootstrap resamples over the 198 GPQA-Diamond problems yield
95\% CIs for each method (Table~\ref{tab:selection-baselines}).
SC-MoA's CI $[69.2\%, 81.3\%]$ excludes all voting baselines' point
estimates, consistent with the McNemar tests.

\paragraph{Cross-perturbation agreement.}
The SPUQ paraphrase-mode run generates 5 independent perturbations per
problem.  Each perturbation produces a best-of-$k$ proposal whose
extracted answer may differ.  Across 198 problems,
unanimous agreement (all 5 paraphrases yield the same answer) occurs
on 56.4\% of problems, with 50.9\% all-correct.
The 56.4\% unanimity rate is a lower bound on stability: after aggregation, the pipeline
absorbs remaining disagreement into a single synthesized answer.


\section{Architectural Comparison of Aggregation Strategies}
\label{app:architecture-comparison}

Figure~\ref{fig:arch-comparison} compares the architectural structure of
six multi-agent aggregation strategies.  Self-consistency
\citep{wang2023selfconsistency} samples $k$ responses from a single
prompt and takes a majority vote over extracted answers.  Mixture of
Agents \citep{wang2025moa} uses $N$ heterogeneous proposers followed by
a layered debate and aggregation stage.  Self-MoA
\citep{yuksekgonul2025textgrad} applies MoA's layered aggregation to
$k$ samples from a single prompt, combining Self-consistency's
single-model setting with MoA's synthesis mechanism.  TextGrad
\citep{yuksekgonul2025textgrad} iteratively refines a single solution
through evaluate--gradient--update cycles.  GoA
\citep{goa2026} structures agent interactions as a graph with
meta-LLM-controlled node/edge sampling, message passing, and graph
pooling.

SC-MoA (rightmost column, highlighted) implements the three design
constraints from \S\ref{sec:method}: perturbation diversity, anchored
refinement, and universal synthesis.
The detailed pipeline is shown in Figure~\ref{fig:pipeline}.

\begin{figure*}[h]
\centering
\resizebox{0.95\textwidth}{!}{\input{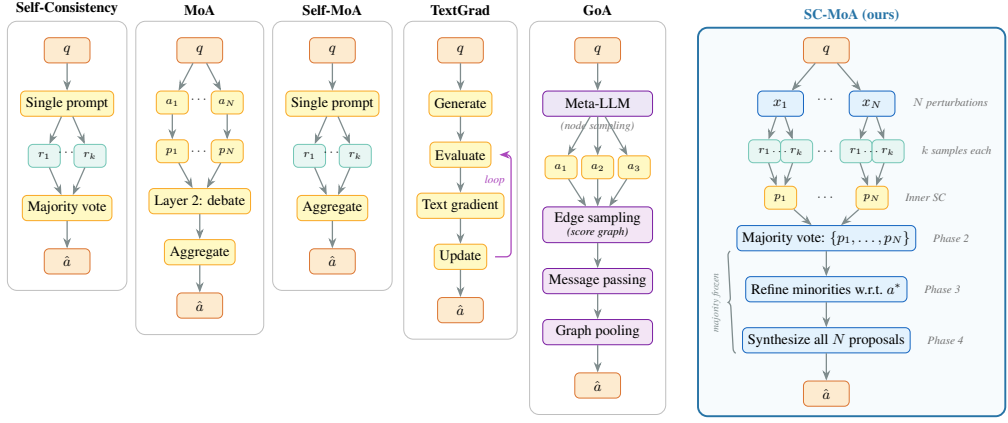}}
\vspace{-2mm}
\caption{\textbf{Architectural comparison of six multi-agent aggregation strategies.}
Self-Consistency votes over samples from a single prompt.
MoA layers heterogeneous proposers with debate.
Self-MoA applies MoA-style aggregation to single-model samples.
TextGrad iteratively refines via text gradients.
GoA structures agent interaction as a scored graph.
SC-MoA (highlighted, rightmost) uniquely combines perturbation-based diversity,
anchored refinement with the majority answer frozen, and universal
trace-level synthesis with verification---never gating on consensus.}
\label{fig:arch-comparison}
\end{figure*}


\end{document}